\documentclass[letter]{article}

     \PassOptionsToPackage{numbers, compress}{natbib}
\usepackage[preprint]{neurips_2024}
\makeatletter
\renewcommand{\@noticestring}{%
}
\makeatother

\usepackage[utf8]{inputenc} %
\usepackage[T1]{fontenc}    %
\usepackage{url}            %
\usepackage{booktabs}       %
\usepackage{amsfonts}       %
\usepackage{nicefrac}       %
\usepackage{microtype}      %
\usepackage{xcolor}         %
\usepackage[hidelinks]{hyperref}       %
\hypersetup{
    colorlinks,
    linkcolor={red!50!black},
    citecolor={blue!50!black},
    urlcolor={blue!80!black}
}

\usepackage{enumitem}
\usepackage{multicol}

\usepackage{graphicx}
\usepackage[export]{adjustbox}
\usepackage{booktabs, caption}
\usepackage[subrefformat=parens]{subcaption} %

\usepackage{algorithm,algcompatible}

\usepackage{amsmath}
\usepackage{amsthm}
\newtheorem{thm}{Theorem}[section]
\newtheorem{prop}[thm]{Proposition}%
\newtheorem{corollary}[thm]{Corollary}%
\theoremstyle{remark}
\newtheorem{rem}[thm]{Remark}%
\newtheorem*{rem*}{Remark}

\numberwithin{equation}{section}

\makeatletter
\newcommand{\leqnomode}{\tagsleft@true}
\newcommand{\reqnomode}{\tagsleft@false}
\makeatother

\newcommand{\RR}{\mathbb{R}}%
\newcommand{\NN}{\mathbb{N}}%
\newcommand{\EE}{\mathbb{E}}%
\renewcommand{\epsilon}{\varepsilon}

\newcommand{\setint}[1]{
\mathchoice{\ring{#1}}
           {{#1}^\circ}
           {{#1}^\circ}
           {{#1}^\circ}%
}

\usepackage{mathabx}

\usepackage[
	disable,
	textwidth=3.5cm,
	backgroundcolor=lightgray,
	bordercolor=lightgray,
	textsize=footnotesize
	]{todonotes}
	
\usepackage{soul, environ}
\setulcolor{lightgray} 

\NewEnviron{myul}{\expandafter\ul\expandafter{\BODY}}

\makeatletter
\let\mytodo\todo
\if@todonotes@disabled
\newcommand{\hltodo}[2]{#2}
\else
\newcommand{\hltodo}[2]{\mytodo[noline]{\begin{Spacing}{1}#1\end{Spacing}\vspace{0.5\baselineskip}}\begin{myul}#2\end{myul}}
\fi
\let\todo\hltodo
\makeatother

\title{Mutation-Bias Learning in Games}

\author{%
  Johann Bauer\thanks{\texttt{\{first name\}.\{last name\}@city.ac.uk}}\\
  Dept. of Mathematics\\
  City, University of London, UK\\
  \And
  Sheldon West\\%
  Dept. of Computer Science\\
  City, University of London, UK\\
  \And
  Eduardo Alonso\\%
  Dept. of Computer Science\\
  City, University of London, UK\\
  \And
  Mark Broom\\%
  Dept. of Mathematics\\
  City, University of London, UK\\
}

\begin{document}
\widowpenalty0
\clubpenalty0
\setlength{\floatsep}{9pt}
\setlength{\textfloatsep}{9.0pt plus 2.0pt minus 4.0pt}
\setlength{\abovecaptionskip}{1pt}
\setlength{\belowcaptionskip}{0pt}

\maketitle

\begin{abstract}
We present two variants of a multi-agent reinforcement learning algorithm based on evolutionary game theoretic considerations. The intentional simplicity of one variant enables us to prove results on its relationship to a system of ordinary differential equations of replicator-mutator dynamics type, allowing us to present proofs on the algorithm’s convergence conditions in various settings via its ODE counterpart. The more complicated variant enables comparisons to Q-learning based algorithms. We compare both variants experimentally to WoLF-PHC and frequency-adjusted Q-learning on a range of settings, illustrating cases of increasing dimensionality where our variants preserve convergence in contrast to more complicated algorithms. The availability of analytic results provides a degree of transferability of results as compared to purely empirical case studies, illustrating the general utility of a dynamical systems perspective on multi-agent reinforcement learning when addressing questions of convergence and reliable generalisation.
\end{abstract}

\section{Introduction}%

Reinforcement learning algorithms have been employed in a wide range of problem settings with great success, e.g., \cite{silver_mastering_2017}, and for the single-agent case the conditions for convergence of, e.g., Q-learning have been clarified, \cite{watkins_q-learning_1992}.
However, for multi-agent reinforcement learning (MARL), questions of convergence are still very much open.
Even simple two-player settings, e.g. the Rock-Paper-Scissors (RPS) game, can exhibit chaotic behaviour under simple dynamics, \cite{sato_chaos_2002}, and make a rigorous \emph{a priori} analysis challenging.
For more complicated algorithms, an analysis beyond experimental evaluation is often hardly possible.
However, more general analyses are highly informative of why algorithms behave in a certain way and theoretical guarantees for at least the simplest of settings are highly desirable in order to assess how reliably MARL algorithms will generalise to similar settings.

In particular, as MARL algorithms often lead to stochastic discrete-time dynamic systems, insights from the fields of learning dynamics in games and of evolutionary game theory (EGT) have been particularly relevant. %
EGT approaches and specifically the established replicator dynamics (RD) have informed a number of constructions or analyses of learning algorithms in multi-agent settings, e.g.,
\cite{mertikopoulos_learning_2016, omidshafiei_-rank_2019}.
The potential of EGT to inform learning algorithms is illustrated, as a particularly prominent example, by the fact that the WoLF-PHC learning algorithm, \cite{bowling_multiagent_2002}, keeps track of the past average policy. In light of RD, this is particularly useful, as the time-average policy in RD converges to a Nash equilibrium under self-play in zero-sum games, e.g, \cite[prop. 3.6, p. 92]{weibull_evolutionary_1995}, 
providing an intuition for how WoLF-PHC can learn Nash equilibria in self-play in a number of settings.

\subsubsection*{Contribution}

Building on the relation between RD and a simple form of reinforcement learning, called Cross learning \cite{borgers_learning_1997, cross_stochastic_1973}, we formulate two variants of a new reinforcement learning algorithm: Mutation-bias learning with direct policy updates (MBL-DPU)--a least complexity modification of Cross learning-- and mutation-bias learning with logistic choice (MBL-LC).
Explicitly taking into account the stochasticity of the problem, we prove that MBL-DPU can be approximated by a mutation-perturbed replicator dynamics (\ref{eq:RMD}), specified in \cite{bauer_stabilization_2019}, a non-linear dynamics whose stability properties can still be studied analytically to a certain degree.
Although the Lyapunov stability and other properties of the continuous-time case do not always transfer to the discrete-time learning dynamics---a prominent example is the RPS game, \cite{weibull_evolutionary_1995}---we show that asymptotic stability in the continuous case does imply the convergence of the MARL algorithm.
Simple RD cannot have asymptotically stable interior euqilibria, e.g. \cite[lemma 1]{ritzberger_evolutionary_1995}. Hence, Cross learning is unable to learn interior equilibria and will quickly deviate from RD in cases of merely neutral stability, such as in RPS games. 
In contrast to RD and Cross learning, \ref{eq:RMD} allows interior equilibria to be asymptotically stable, \cite{bauer_stabilization_2019}, enabling the proposed MBL algorithm to overcome this fundamental limitation of Cross learning and approach interior Nash equilibria arbitrarily closely.
Hence, we can show that in the case of globally asymptotically stable equilibria, MBL processes revisit arbitrary neighbourhoods of such equilibria infinitely often almost surely, particularly in zero-sum games.
In contrast to more complicated algorithms, the simplicity of MBL allows an analytic approach to the question of convergence of MBL to an $\varepsilon$-equilibrium in a given game \emph{a priori}---be it zero-sum or not---and further understanding when convergence should not be expected, irrespective of parameter choices. To our knowledge, MBL is among the simplest uncoupled, in the sense of \cite{bowling_multiagent_2002, hart_uncoupled_2003}, algorithms that can learn interior equilibria and among the few such for which a more general rigorous dynamic system analysis is available.

The rest of this paper proceeds as follows: After relating our results to the literature, we state the necessary evolutionary game theoretic preliminaries. We then introduce the two MBL variants, MBL-DPU and MBL-LC, which demonstrates an alternative approach to include the mutation perturbation term closer to Q-learning inspired approaches, and state the propositions on the relation of MBL-DPU to \ref{eq:RMD} and the convergence properties of MBL-DPU.
We then illustrate the theoretical results with numerical experiments in a range of two-player games, as well as a three-player game, and compare the behaviours of the two MBL variants to those of frequency-adjusted Q-learning (FAQ), \cite{kaisers_frequency_2010}, and Win-or-Learn-Fast Policy-Hill-Climbing (WoLF-PHC), \cite{bowling_multiagent_2002}, demonstrating the utility of a rigorous dynamic system analysis in the study of MARL algorithms.

\subsubsection*{Related work}
A larger class of stochastic reinforcement learning rules is related to deterministic continuous-time systems of RD type in \cite{rustichini_optimal_1999}.
Systems of RD type with additional perturbations have been related to various learning rules, including such with entropy related perturbation terms, \cite{sato_coupled_2003}, and exponential learning based on a logit model, \cite{marsili_exact_2000}.
Some analyses focus specifically on Q-learning based learning algorithms.
For instance, \cite{kianercy_dynamics_2012} considers the stability and convergence properties of Q-learning in the two-player setting; however, the Q-values enter as expectations, not as random variables, and therefore the effects of stochasticity are not considered---a crucial factor in a rigorous analysis.
A similar approach is pursued by the frequency-adjusted Q-learning algorithm (FAQ) in \cite{tuyls_evolutionary_2006} with a corrected derivation given in \cite{kaisers_frequency_2010}. However, both strands start from assumptions which have not been proved, and therefore no theoretical guarantees can be inferred. 
Nonetheless, we choose FAQ-learning as a comparison, as \cite{kaisers_frequency_2010} claims it to be linked to an ODE system similar to \ref{eq:RMD} and as it is a sufficiently simple uncoupled algorithm very close to Q-learning, making it a natural candidate for comparison. 
As a second candidate for comparison, we choose WoLF-PHC, \cite{bowling_multiagent_2002}, since its variant WoLF-IGA is strongly linked to a dynamic systems perspective and WoLF-PHC, too, is an uncoupled and relatively simple algorithm, close to Q-learning. Although its theoretical analysis is more thorough than for FAQ, only the two-player two-action analysis of WoLF-IGA is available. 
Both algorithms have demonstrated that they are able to learn Nash equilibria in simple settings under self-play, where simpler algorithms such as Policy-Hill-Climbing would fail.

A separate approach to MARL convergence analysis is pursued via multiple timescales algorithms, where Q-value estimates are learned quicker than policy changes occur, e.g., \cite{collins_convergent_2003}.
Here, the convergence analysis relates to smoothed best-response dynamics. However, the timescale separation results in a fundamentally more complicated approach and more complicated algorithms.
For the case of $\epsilon$-greedy multi-agent Q-learning under stochastic payoffs, convergence conditions are given in \cite{chapman_convergent_2013}. However, this algorithm operates on joint actions, which requires agents to be able to observe the actions chosen by all agents, and is therefore not uncoupled in the sense of \cite{bowling_multiagent_2002}.

We do not take into account proximal policy optimization (PPO) algorithms, \cite{schulman_proximal_2017}, for our comparison, since they require an agent to construct an approximation of the actual target function and solve a constrained optimisation problem at each learning step with a suitable sampling strategy in-between learning and to keep track of a potentially large number of estimates. This results in a much more complicated algorithm than analysed here and convergence analysis even in the single-agent setting is challenging, e.g., \cite{liu_optimistic_2023}. We are not aware of a rigorous MARL convergence analysis in non-cooperative games, although experimental results in this direction exist, e.g., \cite{mali_policy-based_2023} for n-player RPS games with convergence only in very limited cases, or \cite{ratcliffe_win_2019} extending PPO to WoLF-PPO in experimental studies of Matching Pennies and two-player RPS.

\section{Preliminaries}%

As our analysis of multi-agent learning is formulated in the setting of (evolutionary) game theory, we give short definitions of the main concepts employed and refer the reader to the standard literature for details \cite[e.g.,][]{hofbauer_evolutionary_1998, weibull_evolutionary_1995}.%

\paragraph{Finite normal-form games.}
A normal-form game is a tuple $(P, A, r)$, where $P = \{1, \ldots, N\}$ represents the set of players, $A = \times_{i \in P} A_{i}$ where $A_{i} = \{1, \ldots, n_i\}$ is the set of pure strategies of each player $i$,\footnote{%
    $A$ is usually denoted $S$ in the game theory literature, and players are conceived as populations of pure strategies in the EGT literature. In the simplest case, pure strategies correspond to actions in the reinforcement learning literature. We use the terms `player' and `agent' synonymously.%
} and $r = (r_{i})_{i \in P}$ is a family of functions with $r_i: A \rightarrow \RR$ mapping the pure strategy profiles in $A$ to the payoffs of player $i$. 
For each player $i \in P$, we assume that the player chooses a pure strategy from $A_i$ according to some probability distribution $x_i$ over $A_i$, i.e., according to some tuple $(x_{ih})_{h \in A_i} \in \mathcal{D}_i := \{ \xi \in \RR^{A_i}_{\geq 0} :\: \sum_h \xi_{h} = 1 \}$.
We call such an $x_i$ the mixed strategy of player $i$.\footnote{This would be referred to as a policy in the reinforcement learning literature.}
We will call mixed strategies simply \emph{strategies}, where there is no danger of confusion.

\paragraph{Nash equilibrium.}
We call a strategy profile $x^* := (x^*_i)_{i \in P} \in \mathcal{D} := \times_{i \in P} \mathcal{D}_i$ a \emph{Nash equilibrium} if for all players $i \in P$ and all mixed strategies $x_i \in \mathcal{D}_i\setminus\{x^*_i\}$, we have
\begingroup%
\abovedisplayskip=3pt%
\belowdisplayskip=3pt%
\begin{align}\textstyle
    \EE[r_i(a)|x^*] \geq \EE[r_i(a)|(x_i, x^*_{-i})]
\end{align}%
\endgroup%
where $(x_i, x^*_{-i}) \in \mathcal{D}$ denotes the mixed strategy profile for which $(x_i, x^*_{-i})_{ih} = x_{ih}$ ($\forall h \in A_i$) and $(x_i, x^*_{-i})_{jh} = x^*_{jh}$ ($\forall j \in P\setminus\{i\}, h \in A_j$). 
The equilibrium is called a \emph{strict Nash equilibrium} if the inequality is strict for all $i \in P$. 
The well-known intuition of this concept is that no player has an incentive to deviate from the Nash equilibrium strategy given that all other players play the Nash equilibrium strategy profile, since for each player $i \in P$, $x^*_i$ is a \emph{best-response} to $x^*$. 
Equivalently, no pure strategy has a higher payoff than the Nash equilibrium strategy:
\begingroup%
\abovedisplayskip=3pt%
\belowdisplayskip=3pt%
\begin{align}\textstyle
    \forall i \in P, h \in A_i:\: \EE[r_i(a)|x^*] \geq \EE[r_i(a)|x^*, a_i = h] \:.
\end{align}%
\endgroup%
As a useful relaxation of this concept, we call a strategy profile $(\tilde{x}_i)_{i \in P} \in \mathcal{D}$ an \mbox{\emph{$\varepsilon$-equilibrium}} if%
\begingroup%
\abovedisplayskip=3pt%
\belowdisplayskip=3pt%
\begin{align}\textstyle
    \exists \varepsilon > 0 \: \forall i \in P, h \in A_i:\: \EE[r_i(a)|\tilde{x}] \geq \EE[r_i(a)|\tilde{x}, a_i = h] - \varepsilon \:,
\end{align}%
\endgroup%
i.e. every pure strategy is by at most $\epsilon$ better than $(\tilde{x}_i)_{i \in P}$, and for all players $i \in P$, $(\tilde{x}_i)_{i \in P}$ is an \emph{$\varepsilon$-best-response} to $\tilde{x}$.

\paragraph{Repeated games, learning and rationality.}
Given a finite normal-form game, we consider an infinitely repeated game to be a repetition of the normal-form game for each round $t \in \NN$. In particular, assuming that in each round $t$ the players choose a pure strategy profile $a(t)$ according to the mixed strategy profile $x(t) = (x_i(t))_{i \in P}$, these pure strategy profiles define a stochastic process $\{a(t)\}_{t \in \NN}$.
In turn, an algorithm which adapts the mixed strategy profile in each round $t$, defines a potentially stochastic process $\{ x(t) \}_{t \in \NN}$.
It is this resulting process and its properties which are the focus of our convergence analysis. 
Following the definition given by \cite{bowling_multiagent_2002}, we call such a process \emph{rational}, if a player $i$'s mixed strategy $\{ x_i(t) \}_{t \in \NN}$ converges to a best-response whenever all other players' strategies converge to a stationary policy. We call a process \emph{$\varepsilon$-rational} if it converges to an $\varepsilon$-best-response. It is clear that in the case of stationary policies for all other players, the focal player faces a Markov decision process and the best-response strategy maximises the player's average expected payoff. In the simplest case, where players cannot observe other players' actions and have no memory, as considered here, the usual state space and the state-dependency of policies disappear.

\paragraph{Replicator-mutator dynamics.}
We consider the multi-population replicator-mutator dynamics formulated in \cite{bauer_stabilization_2019}, which is a special case of general replicator-mutator dynamics \cite[e.g.,][]{page_unifying_2002}: 
For all $i \in P$, let $M_i > 0$ be a mutation parameter, $c_i \in \setint{\mathcal{D}_i}$ (denoting the interior of $\mathcal{D}_i$) some fixed parameter and $f_i: \mathcal{D} \rightarrow \RR^{A_i}$ a continuously differentiable fitness function.
Then the replicator-mutator dynamics is given for $i \in P$, $h \in A_i$ by
\begingroup%
\abovedisplayskip=3pt%
\belowdisplayskip=3pt%
\begin{align*}\tag{RMD}\label{eq:RMD}\textstyle
    \dot{x}_{ih}(t) = x_{ih}(t) \big( f_{ih}(x(t)) - \sum_k x_{ik}(t) f_{ik}(x(t)) \big) + M_i (c_{ih} - x_{ih}(t)) \:.
\end{align*}%
\endgroup%
In case that $M_i = 0$ for all $i \in P$, \ref{eq:RMD} reduces to the standard multi-population replicator dynamics (RD).
One possible (and usual) conceptualisation of the fitness of a pure strategy $h \in A_i$ is to assume that it is the expected payoff of playing $h$, given all other players' strategies, or more concretely, given a strategy profile $x \in \mathcal{D}$ let the fitness $f_{ih}$ satisfy
$%
    f_{ih}(x) = \EE[ r_i(a) | x, a_i = h ]
$. %
It is clear that all fitness functions are continuously differentiable in this case.

\begin{rem*}
The equilibria of \ref{eq:RMD}, also called \emph{mutation equilibria}, in general are not Nash equilibria of the underlying game. Instead, they are $\epsilon$-equilibria, where $\epsilon$ depends on $(M_i)_{i \in P}$ as shown in \cite{bauer_stabilization_2019}.
\end{rem*}

\section{Mutation-bias learning}%

We can now introduce the stochastic learning rules and specify their relation to \ref{eq:RMD}. 
We provide two variants of MBL: 
one, based on direct policy updates (MBL-DPU, alg. \ref{alg:MBL-DPU})--where the policy update corresponds to Cross learning, \cite{cross_stochastic_1973}, with a mutation bias as a perturbation term;
the other, based on logistic choice (MBL-LC, alg. \ref{alg:MBL-LC})--where the policy corresponds to logistic choice %
based on action-value estimates which are updated with a mutation bias perturbation.

\begin{algorithm}[h]
\caption{\textbf{(MBL-DPU)} MBL with direct policy update for generic player $i \in P$}
\label{alg:MBL-DPU}
\begin{algorithmic}[1]
    \STATE \textbf{Initialise:} Choose learning rate $\theta$, mutation parameters $M_i > 0$ and $c_i \in \setint{\mathcal{D}_i}$, initial $x_i \in \mathcal{D}_i$.
    
    \FORALL{times $t$}
      \STATE Select strategy $a_i \in A_i$ with probabilities $\Pr(a_i = h) = x_{ih}$ ($\forall h \in A_i$).%
     
      \STATE Observe payoff $r_i$ resulting from strategy profile $(a_j)_{j \in P}$.
     
      \STATE For all $h \in A_i$, set:
$\displaystyle\;
x_{ih} \leftarrow \begin{cases}
x_{ih} + \theta (1-x_{ih}) r_{i} + \theta M_i \left( c_{ih} - x_{ih} \right) & \text{ if } h = a_i,
\\
x_{ih} - \theta x_{ih} r_{i} + \theta M_i \left( c_{ih} - x_{ih} \right) & \text{ otherwise. }
\end{cases}
$
  \ENDFOR
\end{algorithmic}
\end{algorithm}

\paragraph{MBL with direct policy update (MBL-DPU).}
MBL-DPU, alg. \ref{alg:MBL-DPU}, is the simpler of the two variants with a direct policy update and no estimation of $Q$-values.
It is an additive linear perturbation of Cross learning with perturbation term $\theta M_i \left( c_{ih} - x_{ih} \right)$, line 5, and 
becomes identical to Cross learning, \cite{borgers_learning_1997, cross_stochastic_1973}, for $M_i = 0$ ($\forall i \in P$). 
In this sense it can be said to be a least complexity modification of Cross learning, since only few elementary computations are required in addition to simple Cross learning. 
We note that the assumption in Cross learning, that the payoffs $r_i$ be restricted to $[0,1]$ is not necessary. It suffices that payoffs are non-negative and bounded. In this case, $\theta$ has to be chosen small enough to ensure well-definition of MBL-DPU.
Note that this assumption is not restrictive for finite games, as boundedness is trivially satisfied for finite games and non-negativity can be ensured by adding a constant $C_i$ to all payoffs $r_i$, affecting neither the Nash equilibria nor the dynamics in the deterministic limit---a straightforward property of RD and \ref{eq:RMD}.

\begin{algorithm}[h]
\caption{\textbf{(MBL-LC)} MBL with logistic choice for generic player $i \in P$}
\label{alg:MBL-LC}
\begin{algorithmic}[1]
    \STATE \textbf{Initialise:} Choose learning rate $\theta$, $M_i > 0$ and $c_i \in \setint{\mathcal{D}_i}$, $Q_i \in \RR^{A_i}$. Choose $\beta > 0$, $\tau > 0$.

    \FORALL{times $t$}

        \STATE For all $h \in {A}_i$, set: %
      $\enspace %
      x_{ih} \leftarrow \frac{ e^{\tau Q_{ih}} }{ \sum_{k \in {A}_i} e^{\tau Q_{ik}}}.
      $%

        \STATE Select strategy $a_i \in A_i$ with probabilities $\Pr(a_i = h) = x_{ih}$ ($\forall h \in A_i$).%
     
        \STATE Observe payoff $r_i$ resulting from strategy profile $(a_j)_{j \in P}$.

        \STATE For $h = a_i$, set: %
        $\enspace %
            Q_{ih} \leftarrow Q_{ih} + \min\left\{ \frac{\beta}{x_{ih}}, 1 \right\} \theta \left( r_{i} + M_i \frac{c_{ih}}{x_{ih}} \right).
        $%
  \ENDFOR
\end{algorithmic}
\end{algorithm}

\paragraph{MBL with logistic choice (MBL-LC).}
Clearly, the simple perturbation in MBL-DPU can be combined with a wide class of transformations on the payoffs without affecting the additive character of the perturbation.
A somewhat more involved possibility to combine the mutation-like perturbation with a policy update is based on a Boltzmann distribution or multinomial logistic choice, %
as frequently encountered in Q-learning. %
In MBL-LC, alg. \ref{alg:MBL-LC}, the perturbation affects the action-value updates instead of the policy.
Hence, this version more closely resembles the algorithms analysed in \cite{kaisers_frequency_2010, kianercy_dynamics_2012}, and allows a closer comparison to FAQ.
In particular, restricting the adjustment in line 6 by applying a minimum is parallel FAQ.
One can see that the logistic choice policy can still be expressed as a policy update with modified payoffs:
\begingroup%
\abovedisplayskip=2pt%
\belowdisplayskip=3pt%
\begin{align}\textstyle
x_{ih} \leftarrow
\begin{cases}
x_{ih} + (1 - x_{ih}) \tilde{r}_i & \text{if } h=a_i ,
\\
x_{ih} - x_{ih} \tilde{r}_i & \text{otherwise,}
\end{cases}
\qquad \text{with} \enspace
\tilde{r}_i =
\frac{ x_{i a_i} ( e^{\tau \Delta Q_{i a_i} } - 1) }%
{ x_{i a_i} ( e^{\tau \Delta Q_{i a_i} } - 1) + 1 } ,
\end{align}%
\endgroup%
where $Q$ denotes an action-value function and $\Delta Q_{i a_i}$ denotes the update of the action-value of the chosen action $a_i$.
From this it is clear that an intermediate approach could be using the simpler MBL-DPU combined with Q-learning, which is equivalent to transforming payoffs accordingly.

\subsubsection*{Convergence of MBL-DPU}%

We address the question of convergence in two steps. First, we determine whether the stochastic process induced by the learning algorithm can be approximated by a deterministic dynamics. Second, we transfer the convergence properties of the deterministic dynamics to the stochastic process.
For MBL-DPU we have the following convergence result (proved in appendix \ref{app:proofs}):
\begin{prop}\label{prop:MBL-RMD}
For every time $T < \infty$, the family of stochastic processes $\{(X^\theta_{ih}(t))_{i,h}\}_{t \geq 0}$ induced by MBL-DPU converges to \ref{eq:RMD} in the sense that for all $\epsilon > 0$:
\begingroup%
\abovedisplayskip=\belowdisplayshortskip%
\belowdisplayskip=\belowdisplayshortskip%
\begin{align}\textstyle
    \sup_{x(0)}\, \Pr(\|X^\theta(n_\theta) - \Phi(x(0), T)\| > \epsilon ) \rightarrow 0 \quad \text{as} \quad \theta \rightarrow 0,
\end{align}%
\endgroup%
where $n_\theta \theta \rightarrow T$ for $\theta \rightarrow 0$, $x(0)$ is a.s. the initial state of the stochastic processes and $\Phi(x(0), \cdot)$ is the unique solution of \ref{eq:RMD} with $\Phi(x(0), 0) = x(0)$.
\end{prop}

\begin{rem*}
As discussed in \cite{borgers_learning_1997, norman_markov_1972}, proposition \ref{prop:MBL-RMD} on its own does not yield an analysis of the asymptotic behaviour of the stochastic process.
However, if a mutation equilibrium $x^M$ of \ref{eq:RMD} is asymptotically stable and $x(0)$ lies in the basin of attraction of $x^M$, then we have $\Phi(x(0), T) \rightarrow x^M$ as $T \rightarrow \infty$.
Hence, with the asymptotic stability of $x^M$, we have that for $T$ large enough, $\Phi(x(0),T)$ is arbitrarily close to $x^M$ and together with proposition \ref{prop:MBL-RMD}, any neighbourhood of $x^M$ will be reached by the learning process $\{X^\theta(t)\}_{t \geq 0}$ with an arbitrary degree of certainty after finitely many steps for suitable choice of $\theta$.
Although this does not imply that the process must remain in this neighbourhood afterwards, it will revisit the neighbourhood with arbitrary probability depending on $\theta$.
\end{rem*}

\paragraph{Attracting mutation limits.}
In \cite{bauer_stabilization_2019} it was shown that every game has at least one connected Nash equilibrium component that is approximated by mutation equilibria irrespective of the choice of the mutation parameter $c$, as $M \rightarrow 0$, called a \emph{mutation limit}.
Furthermore, it was shown that for the game of Matching Pennies the Nash equilibrium is approximated by asymptotically stable mutation equilibria, warranting the name \emph{attracting mutation limit} for such Nash equilibria.
This implies the following consequence (proved in appendix \ref{app:proofs}):

\begin{prop}\label{prop:MBL_attracting}\label{prop:MBL-DPU_attr_mut_lim}
If a unique Nash equilibrium $x^* \in \setint{\mathcal{D}}$ is an attracting mutation limit and $U$ a neighbourhood of $x^*$, then for every mutation parameter $c \in \setint{\mathcal{D}}$ there are $M > 0$, $\theta > 0$ such that the stochastic process $\{(X^\theta(t))\}_{t \in \NN_0}$ induced by MBL-DPU visits $U$ at a finite time a.s., i.e., with probability $1$ there is $S \in \NN_0$ with $X^\theta(S) \in U$.
In fact, $\{(X^\theta(t))\}_{t \in \NN_0}$ a.s. visits $U$ infinitely often.
\end{prop}
In contrast to MBL-DPU, we do not have a proof of an analogous result for MBL-LC, yet.
In \cite{kaisers_frequency_2010, kianercy_dynamics_2012} it is assumed that FAQ, a similar logistic choice learning rule based on Q-learning, converges to a perturbation of the replicator dynamics, albeit no proof is given.
Although it seems plausible for MBL-LC to behave similarly to MBL-DPU, the experimental results indicate that MBL-LC is likely more sensitive to the choice of learning rate than MBL-DPU, since the logistic choice can cause a stronger variance of the strategy at each learning step, as indicated in the more detailed results for MBL-LC in appendix \ref{app:experiments}. The larger variance in the learning step is also the reason why our proof strategy is considerably more challenging for MBL-LC. %

\paragraph{Perturbation creates a trade-off between accuracy and speed.}
We note that neither MBL-DPU nor MBL-LC converge to a Nash equilibrium but only to an {$\epsilon$\nobreakdash-equilibrium} and in particular, that both stay away from the boundary of $\mathcal{D}$.
For MBL-DPU this is clear from the fact that the equilibria of \ref{eq:RMD} are not Nash equilibria and that the boundary of $\mathcal{D}$ is repelling.
For MBL-LC this is also due to the exploration parameter $\tau$.
For the latter, it is further the case that $\tau$ cannot be let to approach $\infty$ as this collides with the $\theta \rightarrow 0$ limit and makes the time derivative of the policy unbounded. This results in a highly increased variance in the stochastic process, preventing effective learning of equilibria.
This particular aspect applies also to other logistic choice based algorithms, particularly FAQ. 
However, if MBL-LC and FAQ indeed converge to the corresponding ODE systems, then these include $\tau$ as a simple scaling parameter.
Since constant positive rescalings do not change the trajectories, the systems can be rescaled by $1/\tau$ in such a way that $\tau$ effectively regulates the perturbation's strength relative to the replicator dynamics.
In the case of \ref{eq:RMD}, $1/\tau$ can be absorbed by the mutation strength $M$.
Thus an increase of $\tau$ has the same effect as a decrease of $M$ which results in all mutation equilibria moving closer to a Nash equilibrium, as desired. 
A reduction in the perturbation strength also results in a longer time to approach equilibria and this creates a trade-off between accuracy and speed for both MBL-LC and MBL-DPU.

\section{Experimental results}\label{sec:exp_res}%

We illustrate the theoretical results in a number of experimental settings: the Prisoner's Dilemma (PD), Matching Pennies (MP), Rock-Paper-Scissors (RPS) with 3, 5 and 9 available strategies, and the three-player Matching Pennies (3MP) games. We compare MBL-DPU and MBL-LC to FAQ, \cite{kaisers_frequency_2010}, and WoLF-PHC, \cite{bowling_multiagent_2002}. 
For details on the games' payoffs and further experiments, cf. appendix \ref{app:experiments}.

\paragraph{Prisoner's Dilemma (PD).}
PD is an example of a game with a strict Nash equilibrium at a vertex of the joint strategy space $\mathcal{D}$. It is known that strict Nash equilibria are asymptotically stable under RD, e.g., \cite{weibull_evolutionary_1995}. In this case, plain Cross learning would also converge to the Nash equilibrium. It was shown that \ref{eq:RMD} does not destabilise asymptotically stable equilibria of RD \cite[lemma 4.8]{bauer_stabilization_2019}. Hence, the mutation equilibrium resulting from the mutation perturbation remains asymptotically stable and, with our result, MBL-DPU also learns an approximation of the Nash equilibrium. 
In this sense, PD is the least challenging setting in terms of the ease with which the Nash equilibrium can be learned.
The setting serves mainly to illustrate the fact that the learned equilibria of MBL-DPU and MBL-LC in fact lie away from the boundary Nash equilibrium, in particular since mutation pushes the trajectories away from the boundary of $\mathcal{D}$, in contrast to the other two algorithms.
With decreasing mutation strength $M$, both algorithms are able to better approach the Nash equilibrium, as would be expected from \ref{eq:RMD}.
This case also illustrates that the more elementary MBL-DPU converges more slowly than either of MBL-LC, FAQ, or WoLF-PHC. 
For more details and figures on this benign case, we refer the reader to appendix \ref{app:exp_spec_PD}.
\begin{figure}[h] %
    \begin{minipage}[t][][t]{0.67\textwidth}
    \begin{center}
    	\begin{subfigure}[t]{130pt}
    	\centering
        \hspace{-8bp}
        \includegraphics[width=125pt, trim={0 1030bp 0 0}, clip ]{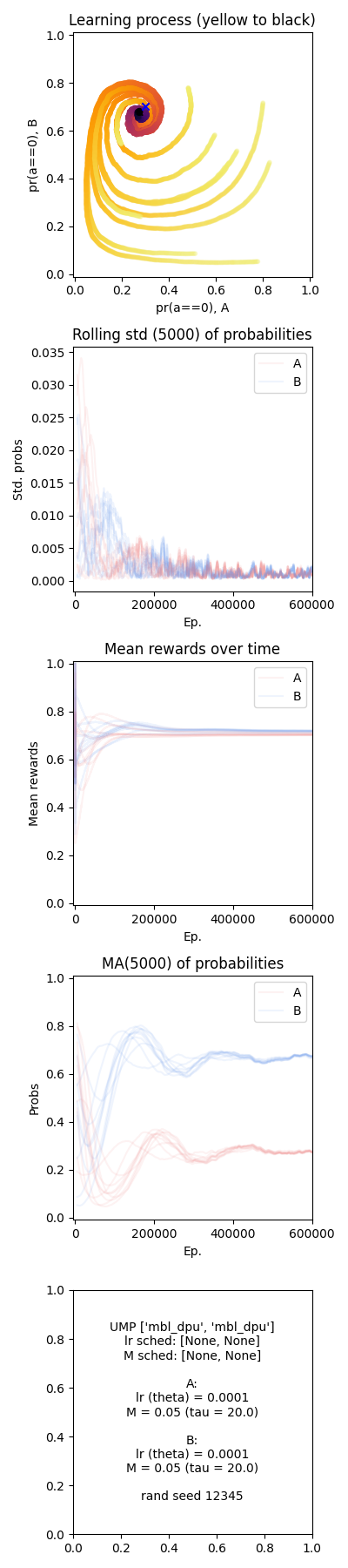}
        \captionsetup{justification=justified,singlelinecheck=true}
    	\subcaption{MBL-DPU with $M^{-1} = 20$.}
        \label{fig:MP_MBL-DPU_20}
    	\end{subfigure}
    	\begin{subfigure}[t]{130pt}
    	\centering
        \hspace{-8bp}
        \includegraphics[width=125pt, trim={0 1030bp 0 0}, clip ]{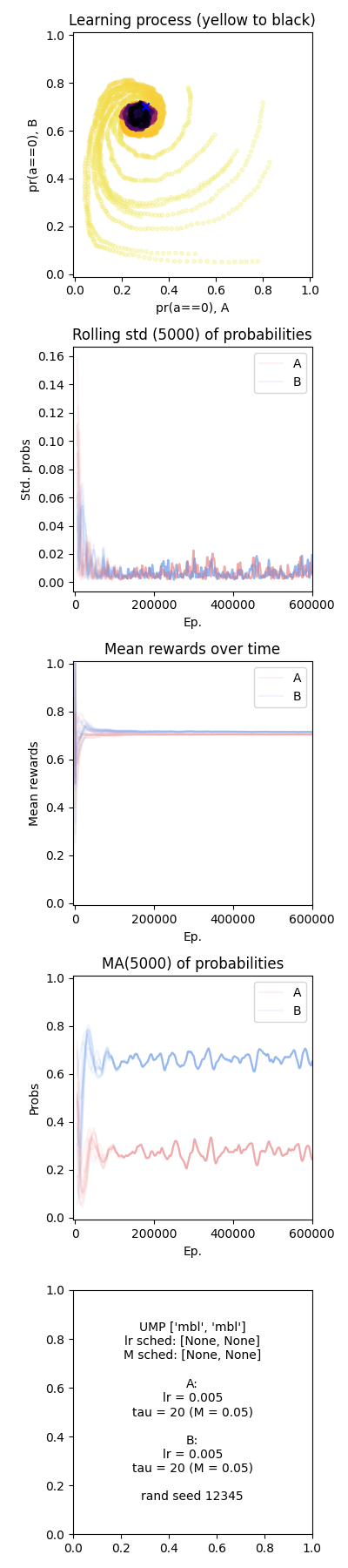}
        \captionsetup{justification=justified,singlelinecheck=true}
    	\caption{MBL-LC with $M^{-1} = \tau = 20$.}
        \label{fig:MP_MBL-LC_20}
        \end{subfigure}
        \\[5bp]
    	\begin{subfigure}[t]{130pt}
    	\centering
        \hspace{-8bp}
        \includegraphics[width=125pt, trim={0 1030bp 0 0}, clip ]{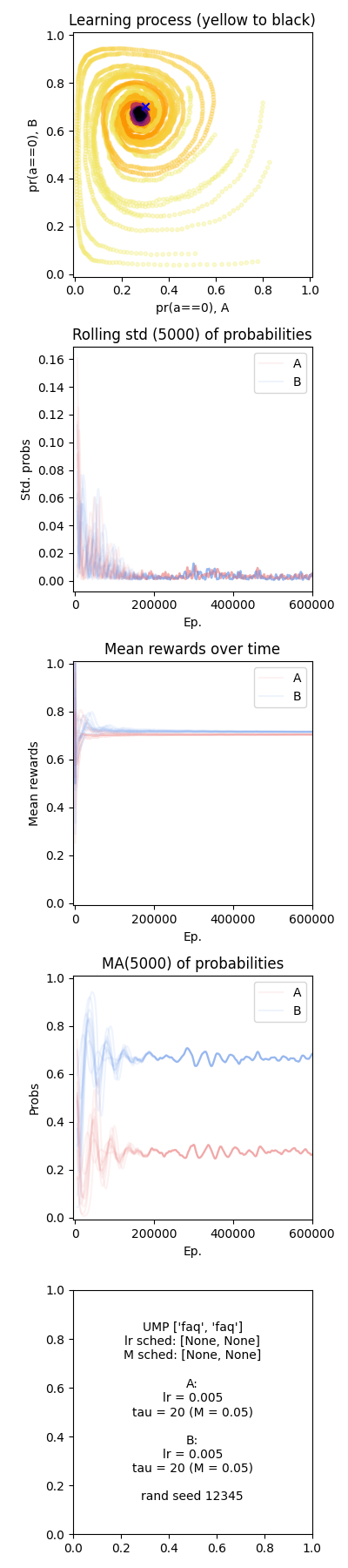}
        \captionsetup{justification=justified,singlelinecheck=true}
    	\caption{FAQ with $\tau = 20$.}
        \label{fig:MP_FAQ_20}
    	\end{subfigure}
        \begin{subfigure}[t]{130pt}
    	\centering
        \hspace{-8bp}
    	\includegraphics[width=125pt, trim={0 1030bp 0 0}, clip ]{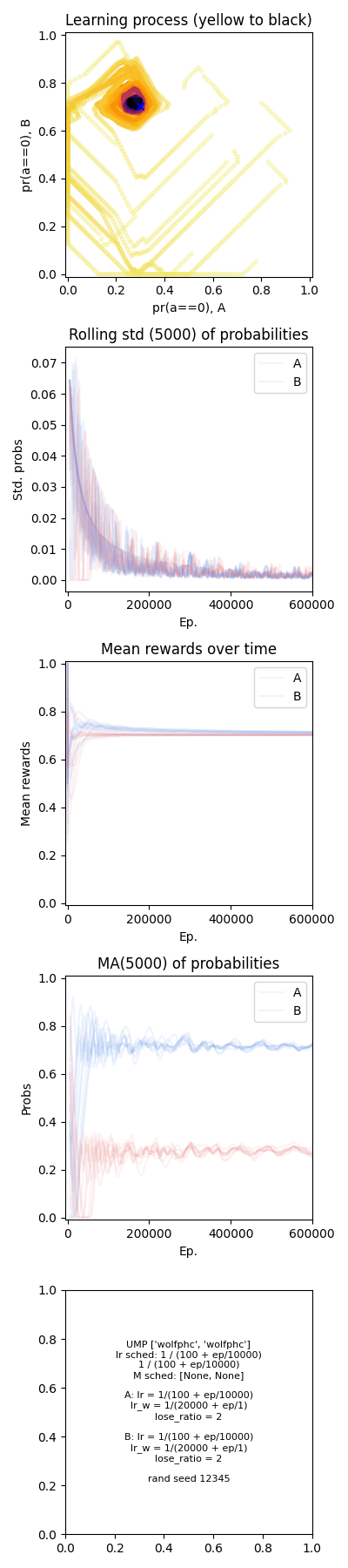}
        \captionsetup{justification=justified,singlelinecheck=true}
        \caption{WoLF-PHC with initial learning rate $10^{-1}$ for $Q$, win learning rate $1/2 \cdot 10^{-4}$.}
        \label{fig:MP_WoLF-PHC}
    	\end{subfigure}
    	\hfill
    \end{center}
    \end{minipage}
    \hfill
    \begin{minipage}[][][t]{0.3\textwidth}
        \captionsetup{justification=justified,singlelinecheck=true}
        \caption[Self-play on the MP game.]{Self-play on the MP game; for 10 different initial conditions.
        Each subfigure shows the ten trajectories in the projection onto the first components of the players' strategies, in this case the `defect' strategy, with the first player on the horizontal axis and the second on the vertical axis.
        Points coloured yellow correspond to earlier points in time, changing over orange and violet to black for later points in time. 
        The position of the game's Nash equilibrium is marked with a blue cross in the projection plane.
        }
        \label{fig:MP_all4}
    \end{minipage}
    \hfill
\end{figure}
\paragraph{Zero-sum games---Matching Pennies (MP).}
As a second, structurally different case, we consider zero-sum games %
which have interior Nash equilibria.
For the games considered here it is straightforward to check that the eigenvalues of the Jacobian of \ref{eq:RMD} in the neighbourhood of the Nash equilibrium only have negative real parts. Equivalently, one can check that the eigenvalues of the Jacobian of RD are purely imaginary in the neighbourhood of the Nash equilibrium and consider that \ref{eq:RMD} shifts the eigenvalues towards the negative half-plane, rendering the Nash equilibrium an attracting mutation limit.
With propositions \ref{prop:MBL-RMD} and \ref{prop:MBL_attracting}, respectively, MBL-DPU is guaranteed to converge in these specific cases,\footnote{In more complex cases with multiple equilibria, convergence depends on the initial state lying in the basin of attraction of an equilibrium.}
with a general result on convergence and stability of \ref{eq:RMD} in zero-sum settings in preparation. 
In fact, we observe convergence in the MP setting for MBL-DPU, MBL-LC, as well as our comparisons, FAQ learning and WoLF-PHC, fig. \ref{fig:MP_all4}. 
This setting illustrates that MBL-DPU overcomes the limitations of Cross learning at a minimal cost in increased complexity. %
Similar to the PD setting, MBL-DPU converges more slowly than the more complicated algorithms, MBL-LC, FAQ, or WoLF-PHC.
With MP being a planar system and the Poincaré-Bendixson theorem, the complexity of the system is still relatively small.

\paragraph{Zero-sum games---Rock-Paper-Scissors (RPS).}
For the higher dimensional settings, i.e., RPS with 3, 5 and 9 strategies, we still observe convergence for MBL-DPU, fig. \ref{fig:RPS_MBL-DPU_20}, as guaranteed by the Nash equilibrium being an attracting mutation limit. Naturally, the trajectories of the resulting 4, 8 and 16 dimensional systems appear less intuitive in the 2D-projection. 
For MBL-LC, fig. \ref{fig:RPS_MBL-LC_20}, and FAQ, fig. \ref{fig:RPS_FAQ_20}, we observe convergence in the RPS-3 case, but both algorithms deteriorate in higher dimensions, MBL-LC for RPS-9, fig. \ref{fig:RPS9_MBL-LC_20}, and FAQ for RPS-5 and RPS-9, figs. \ref{fig:RPS5_FAQ_20} and \ref{fig:RPS9_FAQ_20}, with both showing the convergence region splitting up such that some trajectories stop approximating the Nash equilibrium. 
Similarly, while WoLF-PHC seems to approach the Nash equilibrium in RPS-3 and RPS-5, fig. \ref{fig:RPS_WoLF-PHC}, it loses the ability to learn the Nash equilibrium for RPS-9, fig. \ref{fig:RPS9_WoLF-PHC}, with trajectories seemingly getting stuck near the boundary of $\mathcal{D}$.
\begin{figure}[t] %
    \begin{center}
    	\begin{subfigure}[t]{130pt}
    	\centering
        \hspace{-8bp}
        \includegraphics[width=125pt, trim={0 1030bp 0 0}, clip ]{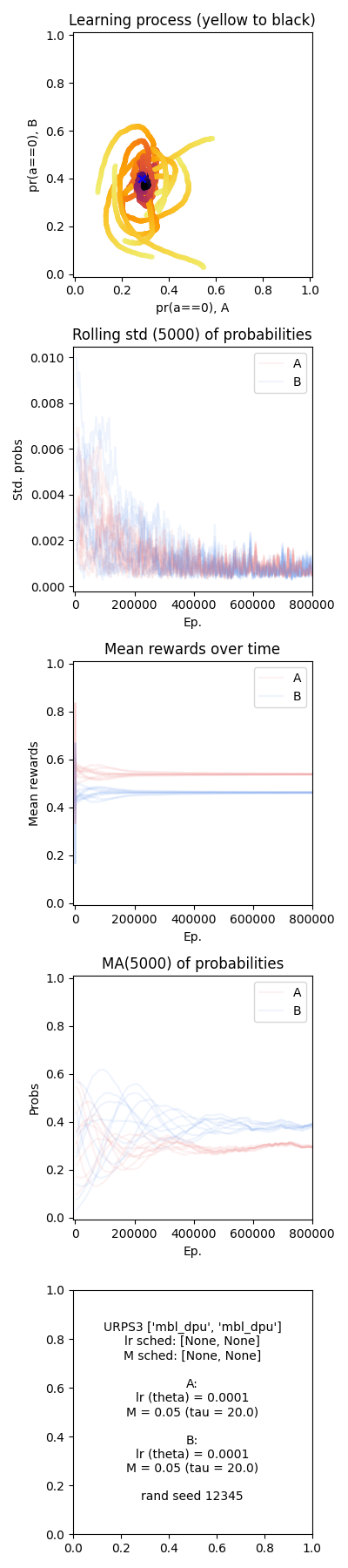}
        \captionsetup{justification=justified,singlelinecheck=true}
    	\subcaption{MBL-DPU on RPS-3.}
        \label{fig:RPS3_MBL-DPU_20}
    	\end{subfigure}
    	\begin{subfigure}[t]{130pt}
    	\centering
        \hspace{-8bp}
        \includegraphics[width=125pt, trim={0 1030bp 0 0}, clip ]{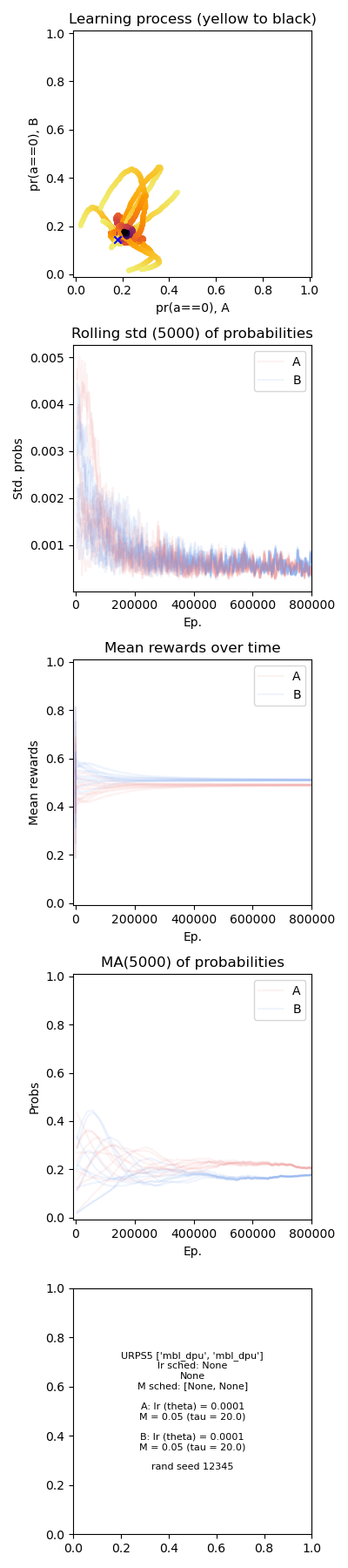}
        \captionsetup{justification=justified,singlelinecheck=true}
    	\subcaption{MBL-DPU on RPS-5.}
        \label{fig:RPS5_MBL-DPU_20}
        \end{subfigure}
    	\begin{subfigure}[t]{130pt}
    	\centering
        \hspace{-8bp}
        \includegraphics[width=125pt, trim={0 1030bp 0 0}, clip ]{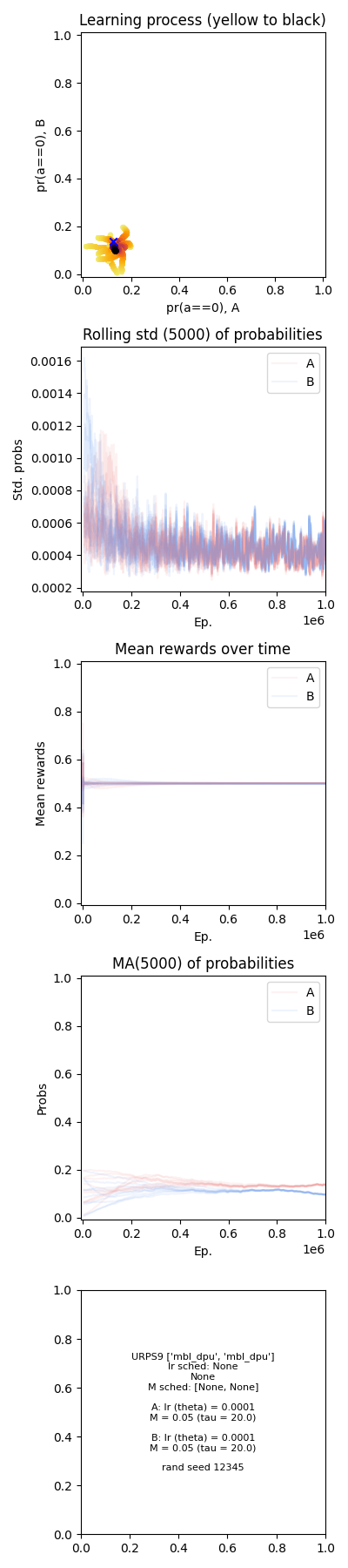}
        \captionsetup{justification=justified,singlelinecheck=true}
    	\subcaption{MBL-DPU on RPS-9.}
        \label{fig:RPS9_MBL-DPU_20}
    	\end{subfigure}
    \end{center}
        \captionsetup{justification=justified,singlelinecheck=true}
        \caption[MBL-DPU in self-play on the RPS games.]{Self-play of MBL-DPU on RPS-3, RPS-5 and RPS-9 games, with $M^{-1} = 20$.
        }
        \label{fig:RPS_MBL-DPU_20}
\end{figure}
\begin{figure}[h] %
    \begin{center}
    	\begin{subfigure}[t]{130pt}
    	\centering
        \hspace{-8bp}
        \includegraphics[width=125pt, trim={0 1030bp 0 0}, clip ]{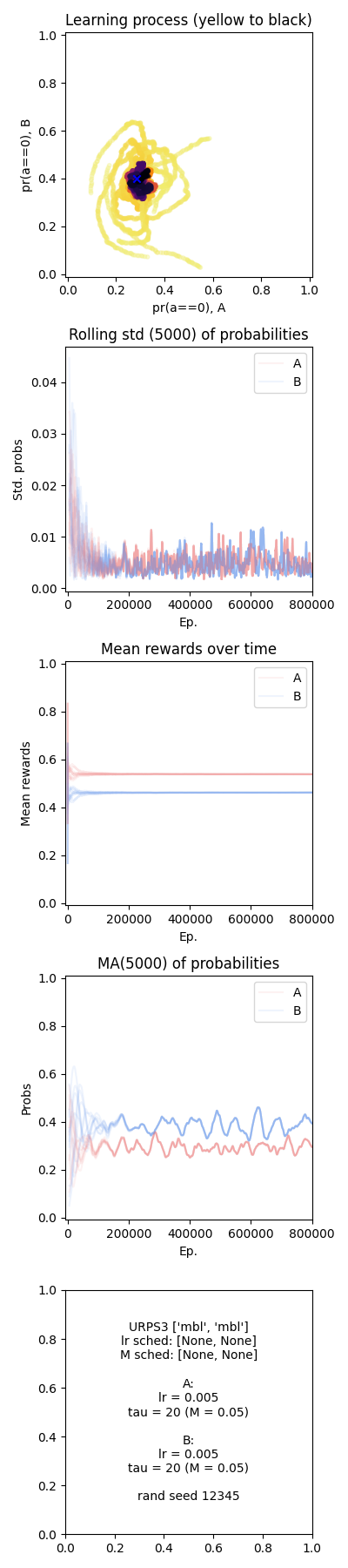}
        \captionsetup{margin=8pt}
    	\subcaption{MBL-LC on RPS-3.}
        \label{fig:RPS3_MBL-LC_20}
    	\end{subfigure}
    	\begin{subfigure}[t]{130pt}
    	\centering
        \hspace{-8bp}
        \includegraphics[width=125pt, trim={0 1030bp 0 0}, clip ]{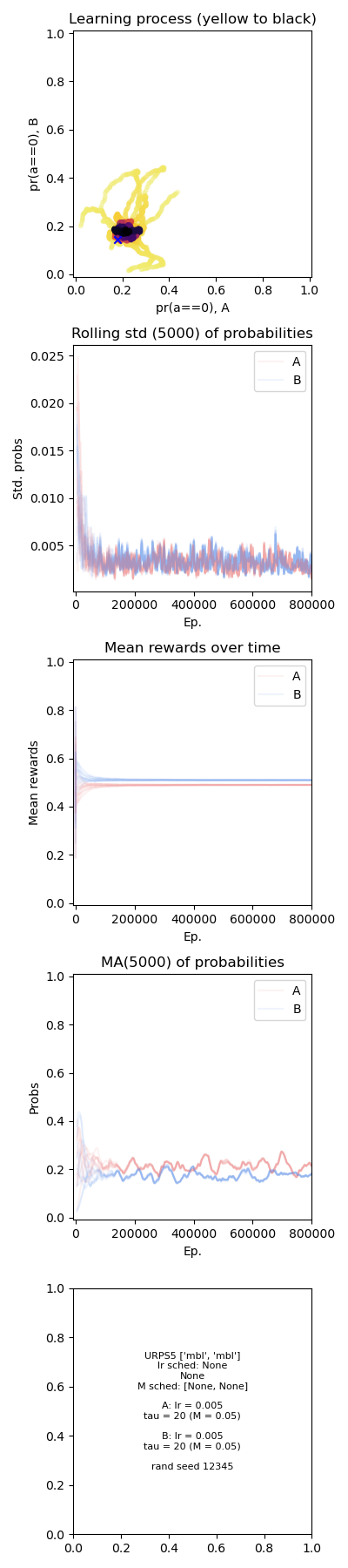}
        \captionsetup{margin=8pt}
    	\subcaption{MBL-LC on RPS-5.}
        \label{fig:RPS5_MBL-LC_20}
        \end{subfigure}
    	\begin{subfigure}[t]{130pt}
    	\centering
        \hspace{-8bp}
        \includegraphics[width=125pt, trim={0 1030bp 0 0}, clip ]{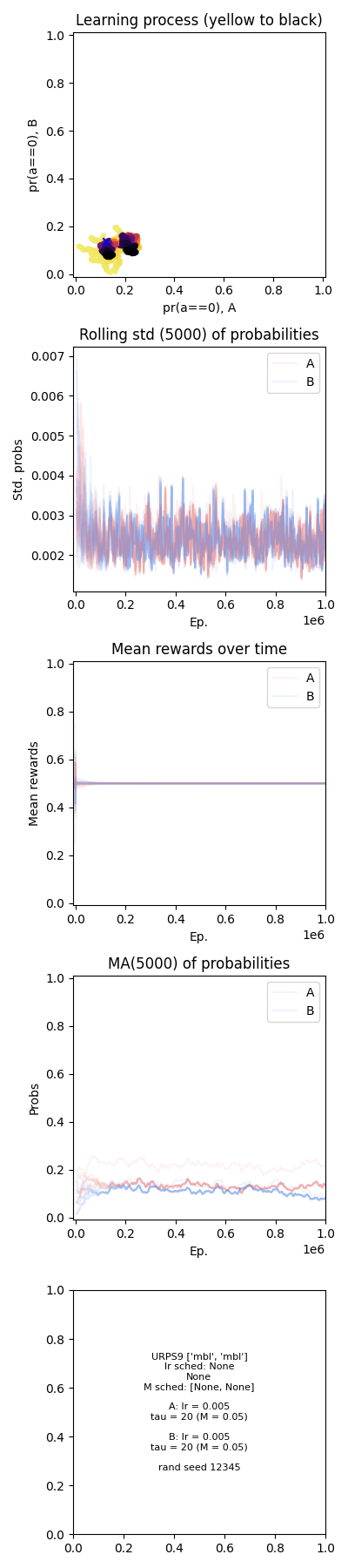}
        \captionsetup{margin=8pt}
    	\subcaption{MBL-LC on RPS-9.}
        \label{fig:RPS9_MBL-LC_20}
    	\end{subfigure}
    \end{center}
        \caption[MBL-LC in self-play on the RPS games.]{Self-play of MBL-LC on RPS-3, RPS-5 and RPS-9 games, with $M^{-1} = \tau = 20$.
        }
        \label{fig:RPS_MBL-LC_20}
\end{figure}
\begin{figure}[h] %
    \begin{center}
    	\begin{subfigure}[t]{130pt}
    	\centering
        \hspace{-8bp}
        \includegraphics[width=125pt, trim={0 1030bp 0 0}, clip ]{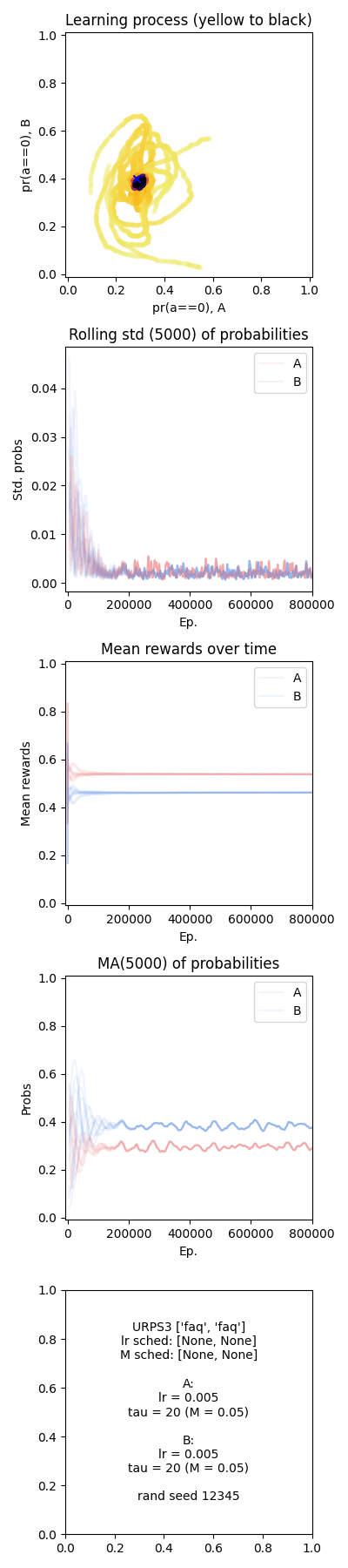}
        \captionsetup{margin=8pt}
    	\subcaption{FAQ on RPS-3.}
        \label{fig:RPS3_FAQ_20}
    	\end{subfigure}
    	\begin{subfigure}[t]{130pt}
    	\centering
        \hspace{-8bp}
        \includegraphics[width=125pt, trim={0 1030bp 0 0}, clip ]{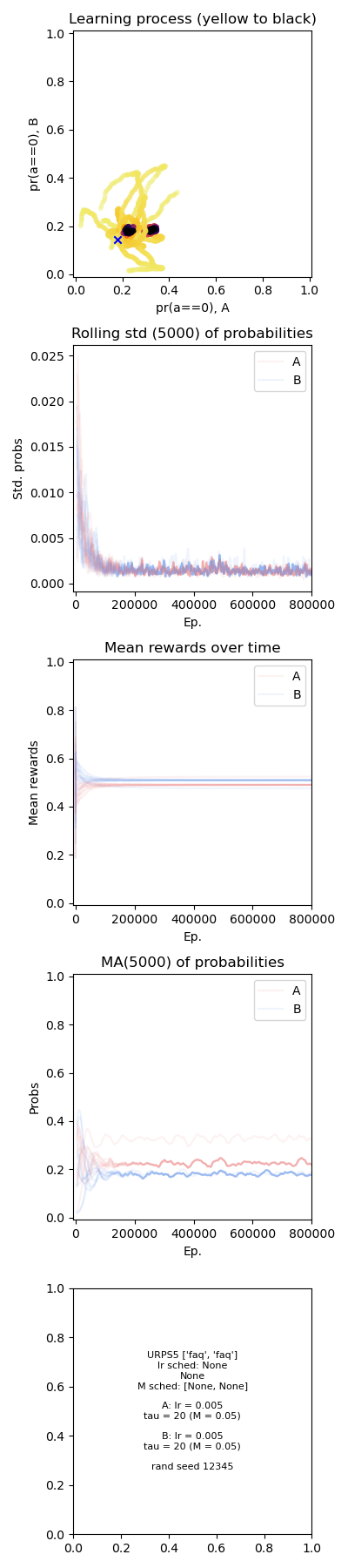}
        \captionsetup{margin=8pt}
    	\subcaption{FAQ on RPS-5.}
        \label{fig:RPS5_FAQ_20}
        \end{subfigure}
    	\begin{subfigure}[t]{130pt}
    	\centering
        \hspace{-8bp}
        \includegraphics[width=125pt, trim={0 1030bp 0 0}, clip ]{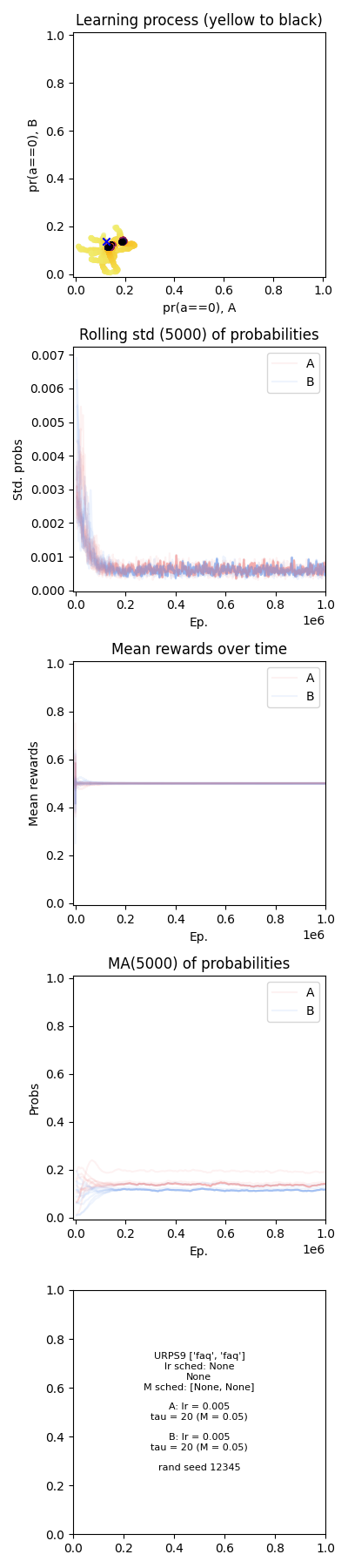}
        \captionsetup{margin=8pt}
    	\subcaption{FAQ on RPS-9.}
        \label{fig:RPS9_FAQ_20}
    	\end{subfigure}
    \end{center}
        \caption[FAQ in self-play on the RPS games.]{Self-play of FAQ-learning on RPS-3, RPS-5 and RPS-9 games, with $\tau = 20$.
        }
        \label{fig:RPS_FAQ_20}
\end{figure}
\begin{figure}[h] %
    \begin{center}
    	\begin{subfigure}[t]{130pt}
    	\centering
        \hspace{-8bp}
        \includegraphics[width=125pt, trim={0 1030bp 0 0}, clip ]{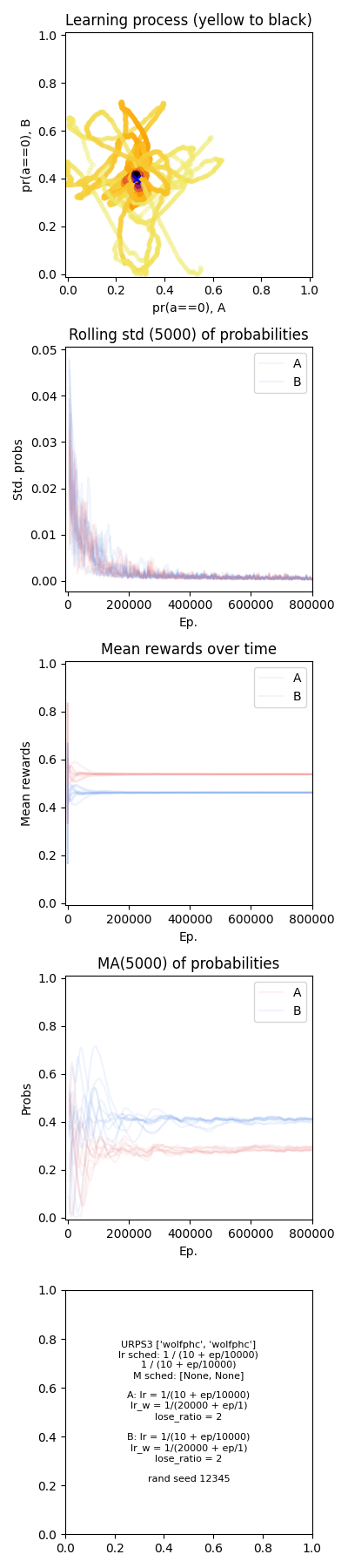}
        \captionsetup{margin=8pt}
    	\subcaption{WoLF-PHC on RPS-3.}
        \label{fig:RPS3_WoLF-PHC}
    	\end{subfigure}
    	\begin{subfigure}[t]{130pt}
    	\centering
        \hspace{-8bp}
        \includegraphics[width=125pt, trim={0 1030bp 0 0}, clip ]{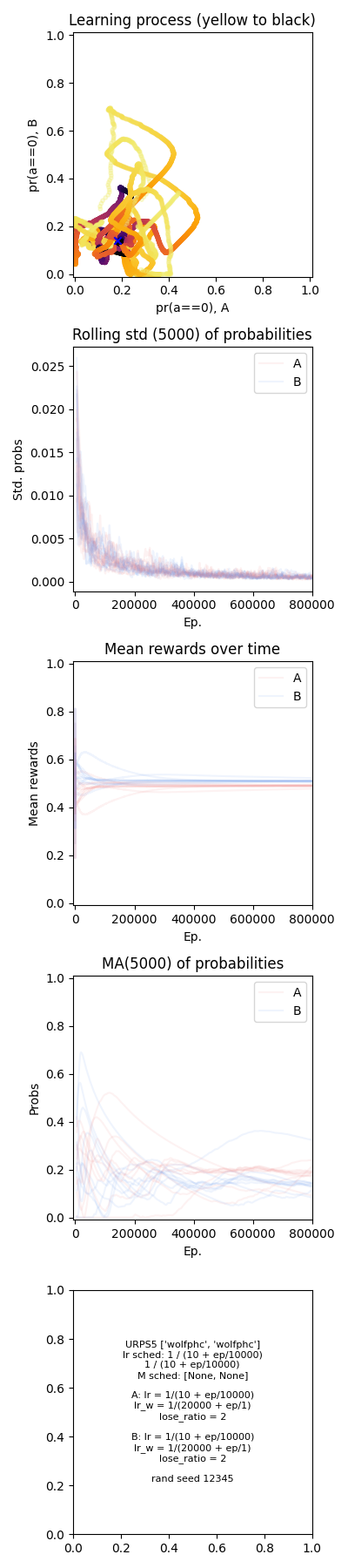}
        \captionsetup{margin=8pt}
    	\subcaption{WoLF-PHC on RPS-5.}
        \label{fig:RPS5_WoLF-PHC}
        \end{subfigure}
    	\begin{subfigure}[t]{130pt}
    	\centering
        \hspace{-8bp}
        \includegraphics[width=125pt, trim={0 1030bp 0 0}, clip ]{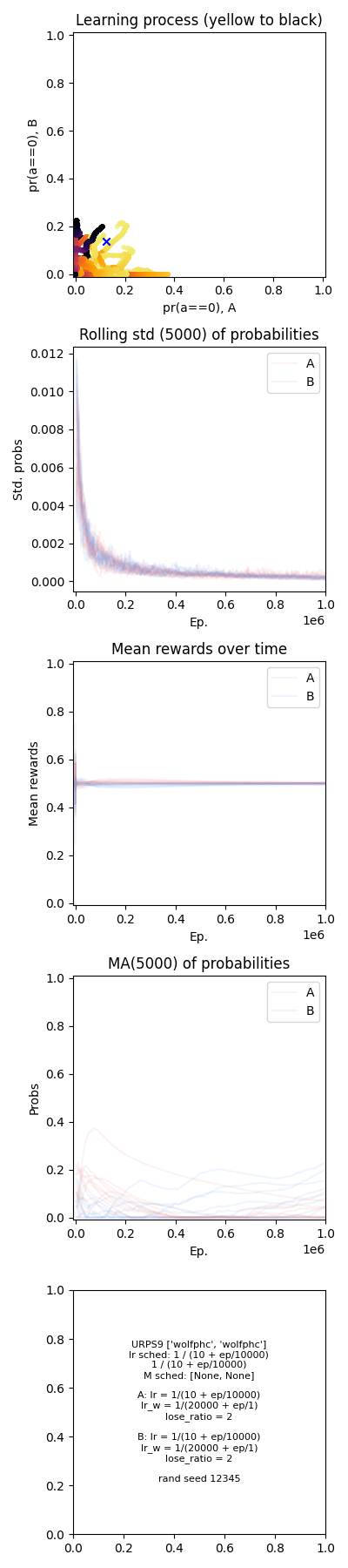}
        \captionsetup{margin=8pt}
    	\subcaption{WoLF-PHC on RPS-9.}
        \label{fig:RPS9_WoLF-PHC}
    	\end{subfigure}
    \end{center}
        \caption[WoLF-PHC in self-play on the RPS games.]{Self-play of WoLF-PHC-learning on RPS-3, RPS-5 and RPS-9 games, with initial learning rate $10^{-1}$ for $Q$, win learning rate $1/2 \cdot 10^{-4}$.
        }
        \label{fig:RPS_WoLF-PHC}
\end{figure}
\paragraph{Three-player Matching Pennies.}
Beyond the two-player case, we compare MBL in a three-player Matching Pennies setting introduced in \cite{jordan_three_1993}. In short, the three players have a shared pure strategy space, i.e. $A_1 = A_2 = A_3$, with two pure strategies, where player 1 wants to match player 2, player 2 wants to match player 3, and player 3 wants not to match player 1. The unique Nash equilibrium lies at the center of $\mathcal{D}$. 
All four algorithms fail to learn the Nash equilibrium, fig. \ref{fig:3MP_all} (MBL-LC not shown, cf. appendix \ref{app:exp_spec_3MP}). Instead, they seem to approach a seemingly stable periodic orbit.
\begin{figure}[h!] %
    \begin{center}
    	\begin{subfigure}[t]{130pt}
    	\centering
        \hspace{-8bp}
        \includegraphics[width=125pt, trim={0 1030bp 0 0}, clip ]{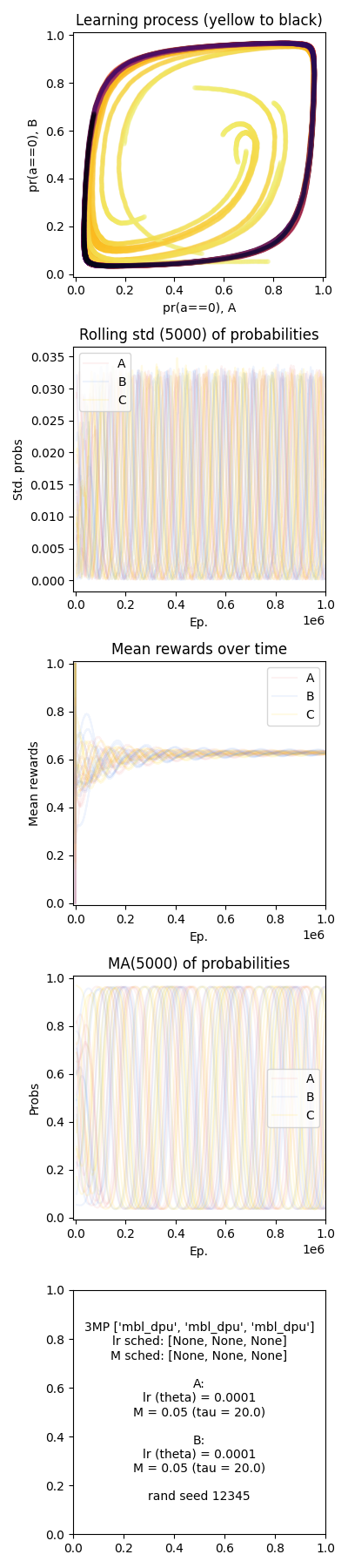}
        \captionsetup{margin=8pt}
    	\subcaption{MBL-DPU.}
        \label{fig:3MP_MBL-DPU_20}
    	\end{subfigure}
    	\begin{subfigure}[t]{130pt}
    	\centering
        \hspace{-8bp}
        \includegraphics[width=125pt, trim={0 1030bp 0 0}, clip ]{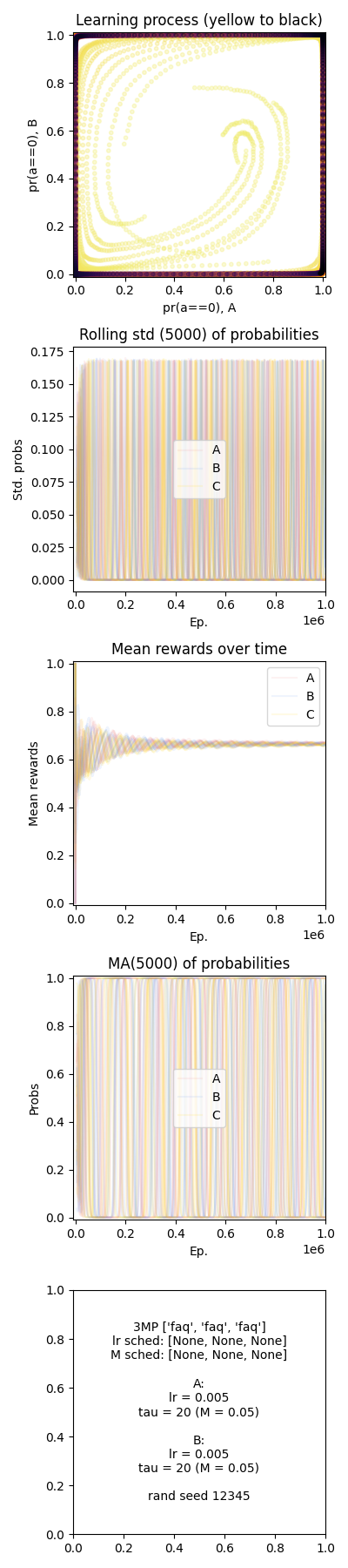}
        \captionsetup{margin=8pt}
    	\subcaption{FAQ.}
        \label{fig:3MP_FAQ_20}
        \end{subfigure}
    	\begin{subfigure}[t]{130pt}
    	\centering
        \hspace{-8bp}
        \includegraphics[width=125pt, trim={0 1030bp 0 0}, clip ]{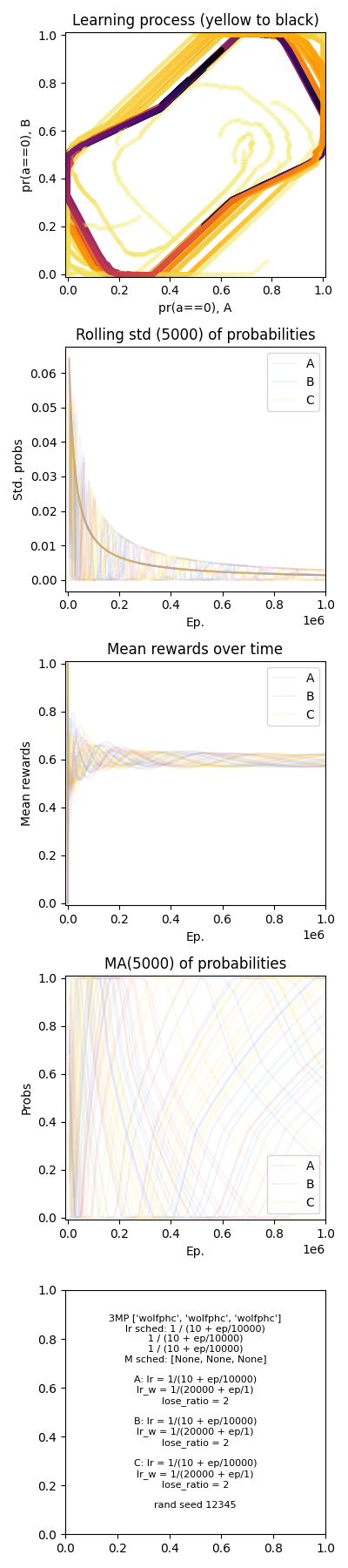}
        \captionsetup{margin=8pt}
        \subcaption{WoLF-PHC.}%
        \label{fig:3MP_WoLF-PHC}
    	\end{subfigure}
    \end{center}
        \caption[Self-play on the 3MP game.]{Self-play on 3MP by
        \protect\subref{fig:3MP_MBL-DPU_20}~MBL-DPU with $M^{-1} = 20$, 
        \protect\subref{fig:3MP_FAQ_20}~FAQ with $\tau = 20$,
        and \protect\subref{fig:3MP_WoLF-PHC}~WoLF-PHC with initial learning rate $10^{-1}$ for $Q$, win learning rate $1/2 \cdot 10^{-4}$.
        }
        \label{fig:3MP_all}
\end{figure}

\section{Discussion}%
The experimental results illustrate the difficulties in relying on experimental results alone. WoLF-PHC, FAQ and MBL-LC all show quicker convergence in those cases where they actually do converge and they would seem the better choice than MBL-DPU.
Not surprisingly, this is the case in PD, which has a strict Nash equilibrium, and in MP which is a planar system and cannot exhibit too complex behaviours. %
However, we see that behaviours start becoming less clear when we move to higher dimensions in the RPS variants.
While all algorithms seem to approximate the Nash equilibrium in RPS-3, we see unexpected behaviour in RPS-5 for FAQ with a split up convergence region. 
In RPS-9 we see FAQ deteriorate further and MBL-LC now also failing to converge with a split in the convergence regions.
WoLF-PHC now too fails to learn the Nash equilibrium, with trajectories stalling or getting stuck near the boundary. 
In RPS-9 no algorithm except for MBL-DPU--the simplest among the four--manages to reliably approach the Nash equilibrium. 
This loss of convergence for the more complex algorithms is unexpected, since RPS-9 does not fundamentally differ from RPS-3 in the game structure and the failure to learn when moving from RPS-3 to RPS-9 would be hard to anticipate \emph{a priori}. 
In contrast, with the results on MBL-DPU we have an indication of how well it will generalise to a structurally comparable but higher dimensional scenario.%

The failure of FAQ, WoLF-PHC and MBL-LC in RPS-9 does not imply that there are no parameter choices that could potentially restore the convergence of the respective algorithms.
E.g., tweaking the learning rates might restore convergence in these specific cases, without guaranteeing convergence in higher dimensional scenarios.
However, the absence of analytical tools leaves the existence of such parameter values an open question.
Even where such parameter choices exist the problem remains potentially intractable without an indication of where to look for them in the parameter space---even more so for algorithms with more parameters.
Together with the unpredictability of failure to converge when moving from a low to a higher dimensional setting, this questions the reliability of algorithms that seem to make sense intuitively and look promising in some experiments but for which we lack fundamental results---particularly for even more complicated algorithms not considered here.
In this situation, the utility of the mathematical guarantees available for MBL-DPU becomes obvious. Given a payoff structure, conditions for convergence can be checked by analysing the ODE system. In specific cases, this even allows the analysis of classes of settings, such as two-player zero-sum games, for which we have preliminary results that \ref{eq:RMD} stabilises equilibria and allows MBL-DPU to converge to the neighbourhood of the Nash equilibrium.
We further understand where exactly MBL-DPU is headed and that empirical non-convergence becomes less likely with smaller learning rates. This gives an indication of where to look for a suitable learning rate.
Finally, where MBL-DPU fails to converge, as in 3MP, just as the other algorithms, the ODE underpinning makes this expectable and understandable, since an analysis of the corresponding \ref{eq:RMD} system quickly shows that the Jacobian of the system has eigenvalues with positive real parts at the Nash equilibrium, making the equilibrium unstable for sufficiently small mutation strengths.
This demonstrates that such theoretical results enable us to understand when a given algorithm is not the best choice for a setting, instead of searching for parameter values that might or might not restore convergence, as we would be forced to do otherwise.

It should be noted that we have left out any modifications to further improve MBL-DPU.
In particular, the mutation strength was fixed, whereas the theoretical perspective makes it quite plausible that mutation strength can be chosen according to a reduction schedule, starting with high mutation and fast convergence and reducing mutation over time, increasing the accuracy with which the Nash equilibrium is approximated.
Note further that the mutation strength is linked to a measure of the Nash condition not being satisfied, since the equilibria of \ref{eq:RMD} are $\varepsilon$-equilibria.
Hence, every player can use the current violation of the Nash condition, i.e., its own distance from a current best-response, as a guide to adjust its mutation strength, e.g., by adjusting the mutation strength to be slightly lower than the current violation of the Nash condition.
We conjecture that this would result in the system being driven towards a state that is not worse than the current state, as measured by the Nash condition, while keeping the convergence speed as high as possible.
We would expect this to speed up convergence and improve the speed-accuracy trade-off, making MBL-DPU more attractive as a simple, predictable and theoretically founded MARL algorithm. 
Apart from such practical considerations, the current analysis still leaves open the questions of analysing MBL-DPU's behaviour in non-zero-sum games without strict Nash equilibria and its behaviour in a wider range of $n$-player settings with more than two players. Additionally, a clarification of the convergence properties of MBL-LC would allow to determine, whether a smaller learning rate would recover convergence, since the logistic choice policy shows much larger variance than the direct policy update and might thus be more sensitive to the learning rate.
Furthermore, the current analysis is limited to stateless repeated games and an extension of the analysis to settings with state-dependency would be desirable, e.g., where players have some limited memory of opponents' past play.

{
\small

\bibliographystyle{plainnat}
\bibliography{db}

\begin{thebibliography}{29}
\providecommand{\natexlab}[1]{#1}
\providecommand{\url}[1]{\texttt{#1}}
\expandafter\ifx\csname urlstyle\endcsname\relax
  \providecommand{\doi}[1]{doi: #1}\else
  \providecommand{\doi}{doi: \begingroup \urlstyle{rm}\Url}\fi

\bibitem[Bauer et~al.(2019)Bauer, Broom, and Alonso]{bauer_stabilization_2019}
Johann Bauer, Mark Broom, and Eduardo Alonso.
\newblock The stabilization of equilibria in evolutionary game dynamics through mutation: Mutation limits in evolutionary games.
\newblock \emph{Proceedings of the Royal Society A: Mathematical, Physical and Engineering Sciences}, 475\penalty0 (2231):\penalty0 20190355, 2019.
\newblock \doi{10.1098/rspa.2019.0355}.

\bibitem[B{\"o}rgers and Sarin(1997)]{borgers_learning_1997}
Tilman B{\"o}rgers and Rajiv Sarin.
\newblock Learning {{Through Reinforcement}} and {{Replicator Dynamics}}.
\newblock \emph{Journal of Economic Theory}, 77\penalty0 (1):\penalty0 1--14, 1997.
\newblock \doi{10.1006/jeth.1997.2319}.

\bibitem[Bowling and Veloso(2002)]{bowling_multiagent_2002}
Michael Bowling and Manuela Veloso.
\newblock Multiagent learning using a variable learning rate.
\newblock \emph{Artificial Intelligence}, 136\penalty0 (2):\penalty0 215--250, 2002.
\newblock \doi{10.1016/S0004-3702(02)00121-2}.

\bibitem[Chapman et~al.(2013)Chapman, Leslie, Rogers, and Jennings]{chapman_convergent_2013}
Archie~C. Chapman, David~S. Leslie, Alex Rogers, and Nicholas~R. Jennings.
\newblock Convergent {{Learning Algorithms}} for {{Unknown Reward Games}}.
\newblock \emph{SIAM Journal on Control and Optimization}, 51\penalty0 (4):\penalty0 3154--3180, 2013.
\newblock \doi{10.1137/120893501}.

\bibitem[Collins and Leslie(2003)]{collins_convergent_2003}
E.~J. Collins and David~S. Leslie.
\newblock Convergent multiple-timescales reinforcement learning algorithms in normal form games.
\newblock \emph{The Annals of Applied Probability}, 13\penalty0 (4):\penalty0 1231--1251, 2003.
\newblock \doi{10.1214/aoap/1069786497}.

\bibitem[Cross(1973)]{cross_stochastic_1973}
John~G. Cross.
\newblock A {{Stochastic Learning Model}} of {{Economic Behavior}}.
\newblock \emph{The Quarterly Journal of Economics}, 87\penalty0 (2):\penalty0 239--266, 1973.
\newblock \doi{10.2307/1882186}.

\bibitem[Hart and {Mas-Colell}(2003)]{hart_uncoupled_2003}
Sergiu Hart and Andreu {Mas-Colell}.
\newblock Uncoupled {{Dynamics Do Not Lead}} to {{Nash Equilibrium}}.
\newblock \emph{American Economic Review}, 93\penalty0 (5):\penalty0 1830--1836, 2003.
\newblock \doi{10.1257/000282803322655581}.

\bibitem[Hofbauer and Sigmund(1998)]{hofbauer_evolutionary_1998}
Josef Hofbauer and Karl Sigmund.
\newblock \emph{Evolutionary Games and Population Dynamics}.
\newblock Cambridge University Press, Cambridge, 1998.

\bibitem[Jordan(1993)]{jordan_three_1993}
J.S. Jordan.
\newblock Three {{Problems}} in {{Learning Mixed-Strategy Nash Equilibria}}.
\newblock \emph{Games and Economic Behavior}, 5\penalty0 (3):\penalty0 368--386, 1993.
\newblock \doi{10.1006/game.1993.1022}.

\bibitem[Kaisers and Tuyls(2010)]{kaisers_frequency_2010}
Michael Kaisers and Karl Tuyls.
\newblock Frequency {{Adjusted Multi-agent Q-learning}}.
\newblock In \emph{Proceedings of the 9th {{International Conference}} on {{Autonomous Agents}} and {{Multiagent Systems}}}, {{AAMAS}} '10, pages 309--316. {International Foundation for Autonomous Agents and Multiagent Systems}, 2010.

\bibitem[Kianercy and Galstyan(2012)]{kianercy_dynamics_2012}
Ardeshir Kianercy and Aram Galstyan.
\newblock Dynamics of {{Boltzmann Q}} learning in two-player two-action games.
\newblock \emph{Physical Review E}, 85\penalty0 (4):\penalty0 041145, 2012.
\newblock \doi{10.1103/PhysRevE.85.041145}.

\bibitem[Liu et~al.(2023)Liu, Weisz, Gy{\"o}rgy, Jin, and Szepesvari]{liu_optimistic_2023}
Qinghua Liu, Gellert Weisz, Andr{\'a}s Gy{\"o}rgy, Chi Jin, and Csaba Szepesvari.
\newblock Optimistic {{Natural Policy Gradient}}: A {{Simple Efficient Policy Optimization Framework}} for {{Online RL}}.
\newblock \emph{Advances in Neural Information Processing Systems}, 36:\penalty0 3560--3577, 2023.

\bibitem[Mali and Czibula(2023)]{mali_policy-based_2023}
Imre~Gergely Mali and Gabriela Czibula.
\newblock Policy-{{Based Reinforcement Learning}} in the {{Generalized Rock-Paper-Scissors Game}}.
\newblock In \emph{{{ESANN}} 2023 Proceedings}, pages 345--350, 2023.
\newblock \doi{10.14428/esann/2023.ES2023-92}.

\bibitem[Marsili et~al.(2000)Marsili, Challet, and Zecchina]{marsili_exact_2000}
Matteo Marsili, Damien Challet, and Riccardo Zecchina.
\newblock Exact solution of a modified {{El Farol}}'s bar problem: {{Efficiency}} and the role of market impact.
\newblock \emph{Physica A: Statistical Mechanics and its Applications}, 280\penalty0 (3-4):\penalty0 522--553, 2000.
\newblock \doi{10.1016/S0378-4371(99)00610-X}.

\bibitem[Mertikopoulos and Sandholm(2016)]{mertikopoulos_learning_2016}
Panayotis Mertikopoulos and William~H. Sandholm.
\newblock Learning in {{Games}} via {{Reinforcement}} and {{Regularization}}.
\newblock \emph{Mathematics of Operations Research}, 41\penalty0 (4):\penalty0 1297--1324, 2016.
\newblock \doi{10.1287/moor.2016.0778}.

\bibitem[Norman(1972)]{norman_markov_1972}
M.~Frank Norman.
\newblock \emph{Markov Processes and Learning Models}.
\newblock Number v. 84 in Mathematics in Science and Engineering. Academic Press, New York, 1972.

\bibitem[Omidshafiei et~al.(2019)Omidshafiei, Papadimitriou, Piliouras, Tuyls, Rowland, Lespiau, Czarnecki, Lanctot, Perolat, and Munos]{omidshafiei_-rank_2019}
Shayegan Omidshafiei, Christos Papadimitriou, Georgios Piliouras, Karl Tuyls, Mark Rowland, Jean-Baptiste Lespiau, Wojciech~M. Czarnecki, Marc Lanctot, Julien Perolat, and Remi Munos.
\newblock {$\alpha$}-{{Rank}}: {{Multi-Agent Evaluation}} by {{Evolution}}.
\newblock \emph{Scientific Reports}, 9\penalty0 (1), 2019.
\newblock \doi{10.1038/s41598-019-45619-9}.

\bibitem[Page and Nowak(2002)]{page_unifying_2002}
Karen~M. Page and Martin~A. Nowak.
\newblock Unifying {{Evolutionary Dynamics}}.
\newblock \emph{Journal of Theoretical Biology}, 219\penalty0 (1):\penalty0 93--98, 2002.
\newblock \doi{10.1006/jtbi.2002.3112}.

\bibitem[Ratcliffe et~al.(2019)Ratcliffe, Hofmann, and Devlin]{ratcliffe_win_2019}
Dino~Stephen Ratcliffe, Katja Hofmann, and Sam Devlin.
\newblock Win or {{Learn Fast Proximal Policy Optimisation}}.
\newblock In \emph{2019 {{IEEE Conference}} on {{Games}} ({{CoG}})}, pages 1--4, London, United Kingdom, 2019. IEEE.
\newblock \doi{10.1109/CIG.2019.8848100}.

\bibitem[Ritzberger and Weibull(1995)]{ritzberger_evolutionary_1995}
Klaus Ritzberger and Jorgen~W. Weibull.
\newblock Evolutionary {{Selection}} in {{Normal-Form Games}}.
\newblock \emph{Econometrica}, 63\penalty0 (6):\penalty0 1371--1399, 1995.
\newblock \doi{10.2307/2171774}.

\bibitem[Rustichini(1999)]{rustichini_optimal_1999}
Aldo Rustichini.
\newblock Optimal {{Properties}} of {{Stimulus}}---{{Response Learning Models}}.
\newblock \emph{Games and Economic Behavior}, 29\penalty0 (1-2):\penalty0 244--273, 1999.
\newblock \doi{10.1006/game.1999.0712}.

\bibitem[Sato and Crutchfield(2003)]{sato_coupled_2003}
Yuzuru Sato and James~P. Crutchfield.
\newblock Coupled replicator equations for the dynamics of learning in multiagent systems.
\newblock \emph{Physical Review E}, 67\penalty0 (1), 2003.
\newblock \doi{10.1103/PhysRevE.67.015206}.

\bibitem[Sato et~al.(2002)Sato, Akiyama, and Farmer]{sato_chaos_2002}
Yuzuru Sato, Eizo Akiyama, and J.~Doyne Farmer.
\newblock Chaos in learning a simple two-person game.
\newblock \emph{Proceedings of the National Academy of Sciences}, 99\penalty0 (7):\penalty0 4748--4751, 2002.
\newblock \doi{10.1073/pnas.032086299}.

\bibitem[Schulman et~al.(2017)Schulman, Wolski, Dhariwal, Radford, and Klimov]{schulman_proximal_2017}
John Schulman, Filip Wolski, Prafulla Dhariwal, Alec Radford, and Oleg Klimov.
\newblock Proximal {{Policy Optimization Algorithms}}.
\newblock \emph{arXiv:1707.06347}, 2017.
\newblock \doi{10.48550/arXiv.1707.06347}.

\bibitem[Silver et~al.(2017)Silver, Schrittwieser, Simonyan, Antonoglou, Huang, Guez, Hubert, Baker, Lai, Bolton, Chen, Lillicrap, Hui, Sifre, van~den Driessche, Graepel, and Hassabis]{silver_mastering_2017}
David Silver, Julian Schrittwieser, Karen Simonyan, Ioannis Antonoglou, Aja Huang, Arthur Guez, Thomas Hubert, Lucas Baker, Matthew Lai, Adrian Bolton, Yutian Chen, Timothy Lillicrap, Fan Hui, Laurent Sifre, George van~den Driessche, Thore Graepel, and Demis Hassabis.
\newblock Mastering the game of {{Go}} without human knowledge.
\newblock \emph{Nature}, 550\penalty0 (7676):\penalty0 354--359, 2017.
\newblock \doi{10.1038/nature24270}.

\bibitem[Teschl(2012)]{teschl_ordinary_2012}
Gerald Teschl.
\newblock \emph{Ordinary Differential Equations and Dynamical Systems}.
\newblock American Mathematical Society, Providence, RI, 2012.

\bibitem[Tuyls et~al.(2006)Tuyls, {'T Hoen}, and Vanschoenwinkel]{tuyls_evolutionary_2006}
Karl Tuyls, Pieter~Jan {'T Hoen}, and Bram Vanschoenwinkel.
\newblock An {{Evolutionary Dynamical Analysis}} of {{Multi-Agent Learning}} in {{Iterated Games}}.
\newblock \emph{Autonomous Agents and Multi-Agent Systems}, 12\penalty0 (1):\penalty0 115--153, 2006.
\newblock \doi{10.1007/s10458-005-3783-9}.

\bibitem[Watkins and Dayan(1992)]{watkins_q-learning_1992}
Christopher~J.C.H. Watkins and Peter Dayan.
\newblock Q-{{Learning}}.
\newblock \emph{Machine Learning}, 8:\penalty0 279--292, 1992.
\newblock \doi{10.1023/A:1022676722315}.

\bibitem[Weibull(1995)]{weibull_evolutionary_1995}
J{\"o}rgen~W. Weibull.
\newblock \emph{Evolutionary Game Theory}.
\newblock MIT Press, Cambridge, Mass., 1995.

\end{thebibliography}
}

\newpage

\appendix

\section{Proofs}\label{app:proofs}

The proofs employ a result proved in \cite[p. 118]{norman_markov_1972}, which we state in the following and then proceed to prove propositions \ref{prop:MBL-RMD} and \ref{prop:MBL-DPU_attr_mut_lim}.

\subsection{A theorem on learning with small steps}

The result from \cite{norman_markov_1972} we employ is phrased in the following terms:
Let $J \subset \RR_{>0}$ be a parameter set with $\inf J = 0$ and $N \in \NN$, such that for every $\theta \in J$, $\{X^\theta_n\}_{n \geq 0} \subset I_\theta \subset \RR^N$ is a Markov process with stationary probabilities. We denote by $\EE_x[X^\theta_n]$ the expected value of $X^\theta_n$ given $X^\theta_0 = x$.
Let further $I$ be the minimal closed convex set with $\bigcup_{\theta} I_{\theta} \subset I$.
Define
\[
    H^\theta_n = \Delta X^\theta_n / \theta
\]
and let $w(x,\theta)$, $S(x, \theta)$, $s(x,\theta)$ and $r(x,\theta)$ for $(x,\theta) \in I \times J$ be given as:
\begin{align*}
    \qquad\qquad\qquad \quad
    w(x,\theta) ={} & \EE[ H^\theta_n | X^\theta_n = x] \in \RR^N
    \\
    S(x, \theta) ={} & \EE[ (H^\theta_n)^2 | X^\theta_n = x] \in \RR^{N \times N}
    \\
    s(x,\theta) ={} & \EE[ (H^\theta_n - w(x,\theta))^2 | X^\theta_n = x]
    = S(x,\theta) - w^2(x,\theta) \in \RR^{N \times N}
    \\
    r(x,\theta) ={} & \EE[ \|H^\theta_n\|^3 | X^\theta_n = x] \in \RR \:.
\end{align*}
where $x^2 = xx^T$ and $\|x\| = \sqrt{x^T x}$ for $x \in \RR^N$.

We can now state theorem 8.1.1 from \cite[p. 118]{norman_markov_1972} (omitting part (C)):
\begin{thm}[Norman]\label{thm:Norman}
In the above situation, let the following conditions be satisfied:

The family of sets $(I_\theta)_\theta$ satisfies
\begin{align*}
    \tag{a.1}\label{eq:norman:a.1}
    \forall x \in I\:: \lim_{\theta \rightarrow 0} \inf_{y \in I_\theta} \|x - y\| = 0 \:.
\end{align*}

There are functions $w$ and $s$ on $I$ such that:
\begin{align*}
    \tag{a.2}\label{eq:norman:a.2}
    & \sup_{x \in I_\theta} \|w(x,\theta) - w(x)\| \in \mathcal{O}(\theta) \:,
    \\
    \tag{a.3}\label{eq:norman:a.3}
    & \sup_{x \in I_\theta} \|s(x,\theta) - s(x)\| \rightarrow 0 \:\text{ for }\: \theta \rightarrow 0 \:,
\end{align*}
where $\mathcal{O}$ refers to the Bachmann–Landau notation.

The function $w$ is differentiable, i.e., there is a function $w'$ such that for all $x \in I$:
\begin{align*}
    \tag{b.1}\label{eq:norman:b.1}
    \lim_{\stackrel{y \rightarrow x}{y \in I} } \frac{ \| w(y) - w(x) - w'(x) (y - x) \| }{ \| y - x \| } = 0 \:.
\end{align*}

The function $w'$ is bounded:
\begin{align*}
    \tag{b.2}\label{eq:norman:b.2}
    \sup_{x \in I}\|w'(x)\| < \infty \:.
\end{align*}

The functions $w'$ and $s$ satisfy the Lipschitz condition:
\begin{align*}
    \tag{b.3}\label{eq:norman:b.3}
    & \sup_{x,y \in I, x \neq y} \frac{\| w'(x) - w'(y) \|}{ \| x - y \| } < \infty \:,
    \\
    \tag{b.4}\label{eq:norman:b.4}
    & \sup_{x,y \in I, x \neq y} \frac{\|s(x) - s(y)\|}{\|x - y\|} < \infty \:.
\end{align*}

The function $r$ is bounded:
\begin{align*}
    \tag{c}\label{eq:norman:c}
    \sup_{\theta \in J, x \in I_\theta} r(x,\theta) < \infty \:.
\end{align*}

Let further for $\theta \in J$ and $x \in I_\theta$, $\mu_n(x, \theta) = \EE_x[X^\theta_n]$ and $\omega_n(x,\theta) = \EE_x[\| X^\theta_n - \mu_n(x, \theta) \|^2]$.

In this case, the following hold:
\begin{itemize}
\item[(A)]
$\omega_n(x,\theta) \in \mathcal{O}(\theta)$ uniformly in $x \in I_\theta$ and $n\theta \leq T$ for any $T < \infty$. 
\item[(B)]
For any $x \in I$, the differential equation
\[
    f'(t) = w(f(t))
\]
has a unique solution $f(t) = f(x, t)$ with $f(0) = x$.
For all $t \geq 0$, we have $f(t) \in I$, and
\[
\mu_n(x, \theta) - f(x, n\theta) \in \mathcal{O}(\theta)
\]
uniformly in $x \in I_\theta$ and $n\theta \leq T$.
\end{itemize}

\end{thm}

\begin{rem}\label{rem:approxInProb}
We note that parts (A) and (B) imply that for all $\epsilon > 0$,
\[
\sup_{x \in I_\theta} \Pr( \| X^\theta_n - f(x, T) \| > \epsilon ) \rightarrow 0
\]
for $n\theta \rightarrow T$, $\theta \rightarrow 0$, and given that $X^\theta_0 = x$ almost certainly for all $\theta$.
\end{rem}

\subsection{Convergence of MBL-DPU}

We restate the simple reinforcement-mutation rule of MBL-DPU in the setting layed out above, denoting the mixed strategies with an upper-case $X$ to underscore that this is a random variable and denoting the dependence on a parameter $\theta$, denoting the whole family of stochastic processes as $\{(X^\theta_{ih}(n))_{i \in P, h \in A_i}\}_{n \geq 0}$.
Let $U(x) = (U_{ih}(x))_{i \in P, h \in A_i}$ be a random variable whose probability distribution depends on $x \in I$ with a discrete, non-negative support which is independent of $x$, and let $M_i < \overline{M}$ for some upper bound $\overline{M} < \infty$ and all $i \in P$.

For a player $i \in P$ and a chosen pure strategy $h \in A_i$, the update rule then is given as follows:
\begin{align}\label{eq:defSimpleMBL}
\begin{aligned}
    X^\theta_{ih}(n+1)
    ={} &
    X^\theta_{ih}(n) + \theta \left( (1-X^\theta_{ih}(n)) U_{ih}(X^\theta(n)) \right) + \theta M_i \left( c_{ih} - X^\theta_{ih}(n) \right)
    \\
    X^\theta_{ik}(n+1)
    ={} &
    X^\theta_{ik}(n) + \theta \left( (-X^\theta_{ik}(n)) U_{ih}(X^\theta(n)) \right) + \theta M_i \left( c_{ik} - X^\theta_{ik}(n) \right) \; \text{ for } k \neq h
    \:.
\end{aligned}
\end{align}

We can now show proposition \ref{prop:MBL-RMD}, i.e., that this rule indeed approximates \ref{eq:RMD} for $\theta \rightarrow 0$ in the sense of remark \ref{rem:approxInProb}:

\begin{prop}\label{prop:MBL-DPU_conv}
There is $J$ such that the family of stochastic processes $\{(X^\theta_{ih}(n))_{i \in P, h \in A_i}\}_{n \geq 0}$ given by \eqref{eq:defSimpleMBL} approximates the replicator-mutator dynamics for $\theta \rightarrow 0$ in the sense of remark \ref{rem:approxInProb} if $X^\theta(0) \in I$ for all $\theta \in J$.
\end{prop}
\begin{proof}
The proof proceeds by showing that $\{(X^\theta_{ih}(n))_{i \in P, h \in A_i}\}_{n \geq 0}$ satisfies the conditions of theorem \ref{thm:Norman}.
For a player $i \in P$ and a chosen strategy $h \in A_i$ we have:
\begin{align*}
    H^\theta_{ih}(n+1)
    ={} & \Delta X^\theta_{ih}(n+1) / \theta
    =
    (1-X^\theta_{ih}(n)) U_{ih}(X^\theta(n)) + M_i (c_{ih} - X^\theta_{ih}(n))
    \\
    H^\theta_{ik}(n+1)
    ={} & \Delta X^\theta_{ik}(n+1) / \theta
    =
    -X^\theta_{ik}(n) U_{ih}(X^\theta(n)) + M_i (c_{ik} - X^\theta_{ik}(n)) \; \text{ for } k \neq h
\end{align*}
Note that in this case, $H^\theta_{ih}(n+1)$ is independent of $\theta$ if $X^\theta(n)$ is given, which simplifies the analysis.
Let us set $u_{ih}(x) = \EE[ U_{ih}(X^\theta(n)) | X^\theta(n) = x]$, where it is clear that there is no dependence on $n$. Note that $u$ is polynomial in the components of $x$ and hence smooth.

\emph{Condition (\ref{eq:norman:a.1}):}
In our case, $I$ is given as the polyhedron $\bigtimes_{i} \mathcal{D}_i$ and $I_\theta = I$ for all $\theta$ and thus condition \eqref{eq:norman:a.1} is satisfied.
It remains to show that indeed $\{(X^\theta_{ih}(n))_{i \in P, h \in A_i}\}_{n \geq 0} \subset I$:
Note that $U_{ih}$ is a discrete non-negative random variable and thus bounded by some $C < \infty$.
For $\theta < (C + \overline{M})^{-1}$, we have $\theta M_i \leq 1$.
Assume that $X^\theta_{ih}(n) = x \in I$, then for a player $i \in P$ and a chosen strategy $h \in A_i$ we have
\begin{align*}
    X^\theta_{ih}(n+1)
    ={} &
    x_{ih} + \theta \big( (1-x_{ih}) U_{ih}(n+1) + M_i (c_{ih} - x_{ih}) \big)
    \\
    ={} &
    x_{ih} (1 - \theta M_i) + \theta (1-x_{ih}) U_{ih}(n+1) + \theta M_i c_{ih} \geq 0
\end{align*}
and for some other pure strategy $k \neq h$, we have
\begin{align*}
    X^\theta_{ik}(n+1)
    ={} &
    x_{ik} + \theta \big( (-x_{ik}) U_{ih}(n+1) + M_i (c_{ik} - x_{ik}) \big)
    \\
    ={} &
    x_{ik} \Big( 1 - \underbrace{ \theta \big( U_{ih}(n+1) + M_i \big)}_{\leq 1} \Big) + \theta M_i c_{ik}
    \geq 0 \:.
\end{align*}
A simple calculation shows that $\sum_k X^\theta_{ik}(n+1) = 1$ if $x \in I$.
Thus we have that $\{(X^\theta_{ih}(n))_{i \in P, h \in A_i}\}_{n \geq 0} \subset I$ if $X^\theta(0) \in I$ for all $\theta$ and we can choose $J = (0, (C + \overline{M})^{-1} )$.

\emph{Conditions (\ref{eq:norman:a.2}) \& (\ref{eq:norman:a.3}):}
Consider first the function $w$:
\begin{align*}
    w_{ih}(x,\theta) ={} & \EE[ H^\theta(n) | X^\theta(n) = x]
    \\
    ={} &
    x_{ih} (1-x_{ih}) \EE[ U_{ih}(n+1) | X^\theta(n) = x] + x_{ih} M_i (c_{ih} - x_{ih})
    \\
    &
    +
    \sum_{k \neq h}
    x_{ik} (-x_{ih}) \EE[ U_{ik}(n+1) | X^\theta(n) = x] + x_{ik} M_i (c_{ih} - x_{ih})
    \\
    ={} &
    x_{ih} \left( u_{ih}(x) - \sum_{k} x_{ik} u_{ik}(x) \right) + M_i (c_{ih} - x_{ih})
\end{align*}
It is clear that $w$ does not depend on $\theta$ and that condition \eqref{eq:norman:a.2} is trivially satisfied.
Similarly, $S(x,\theta)$ and $s(x,\theta)$ do not depend on $\theta$ and condition \eqref{eq:norman:a.3} is trivially satisfied.

\emph{Conditions (\ref{eq:norman:b.1})--(\ref{eq:norman:b.4}):}
Since the function $u$ is smooth, so is $w$. In particular, we have that $\sup_{x \in I} \| w'(x) \| < \infty$ because $I$ is compact and $w'$ is continuously differentiable, from which follows that $w'$ satisfies the Lipschitz-condition \eqref{eq:norman:b.3} on $I$.
Similarly, $s$ is smooth and satisfies \eqref{eq:norman:b.4}.

\emph{Condition (\ref{eq:norman:c}):}
Again, $r$ does not depend on $\theta$, and is smooth on $I$, which is compact. Thus it is bounded on $I$ and condition \eqref{eq:norman:c} is satisfied.

As a consequence, we can apply theorem \ref{thm:Norman} to the family $\{X^\theta(n)\}_{n \geq 0}$ and with remark \ref{rem:approxInProb} we have that for all $\epsilon > 0$,
\[
\sup_{x \in I} \Pr( \| X^\theta(n) - \Phi(x,T) \| > \epsilon ) \rightarrow 0
\]
for $n\theta \rightarrow T$, $\theta \rightarrow 0$, and given that $X^\theta(0) = x$ for all $\theta$,
where for all $i \in P$ and $h \in A_i$, $\Phi$ is the unique solution of the differential equations
\begin{align*}
\dot{\Phi}_{ih}(x,t) & = w_{ih}(\Phi(x,t))
\\
& = \Phi_{ih}(t) \Big( u_{ih}(\Phi(x,t)) - \sum_{k} \Phi_{ik}(x,t) u_{ik}(\Phi(x,t)) \Big) + M_i (c_{ih} - \Phi_{ih}(x,t))
\end{align*}
with $\Phi(x,0) = x$.
\end{proof}

\begin{prop}\label{prop:finite_time}
Let $x^M$ be an equilibrium of \eqref{eq:RMD} and $U$ an open neighbourhood of $x^M$.
If $x^M$ is globally asymptotically stable, then there is $\theta > 0$ such that 
the stochastic process $\{(X^\theta_{ih}(n))_{i \in P, h \in A_i}\}_{n \geq 0}$ defined in \eqref{eq:defSimpleMBL}
visits $U$ almost surely after finitely many steps.
\end{prop}
\begin{proof}
Let $\Phi(x,\cdot): \RR_{\geq 0} \rightarrow \mathcal{D}$ satisfy \eqref{eq:RMD} with $\Phi(x,0) = x$ for all $x \in \mathcal{D}$.
Let further $U' \subset U$ such that $x^M \in U'$ and $\bigcup_{x \in U'} B_\delta(x) \subset U$ for some $\delta > 0$, where $B_\delta(x)$ denotes an open ball with radius $\delta$ around $x$.
As $x^M$ is globally asymptotically stable, there is for each $x \in \mathcal{D}$ a $t' < \infty$ such that for all $t > t'$: $\Phi(x,t) \in U'$.

This is because there is a neighbourhood $V \subset U'$ of $x^M$ such that $\forall x^0 \in V, t > 0: \Phi(x^0, t) \in U'$ due to the Lyapunov stability of $x^M$.
Since $x^M$ is asymptotically stable, for every $x$ there is a $t > 0$ such that $\Phi(x, t) \in V$ and hence the solution will remain in $U'$ afterwards.

Therefore, define $\tau: \mathcal{D} \rightarrow \RR$ such that:
\[
\tau(x) = \inf \{ T > 0 :\: \Phi(x,T) \in V \}
\]

Since the RHS of \eqref{eq:RMD} is continuously differentiable by assumption, it is also Lipschitz continuous. Thus, $\Phi$ is continuous in the first argument and so is $\tau$ as the following argument shows:

Let $x \in \mathcal{D}$ and $\epsilon_1 > 0$. Then there is $t > \tau(x)$ such that $\Phi(x,s) \in V$ for $s \in (\tau(x), t]$.
Choose $s \in (\tau(x), t]$ such that $|\tau(x) - s| < \epsilon_1$.
Then $\Phi(x,s) \in V$ and there is a neighbourhood $U_x$ of $x$ such that for all $y \in U_x$, $\Phi(y,s) \in V$.
Hence $\tau(y) < s < \tau(x) + \epsilon_1$.

We also have $\tau(y) > \tau(x) - \epsilon_1$ due to the following:\\
Consider $d := \inf \{ \| \Phi(x, \tau(x) - \epsilon_1) - v \| : v \in V \} > 0 $.
Note that the Lipschitz condition implies that there is $L > 0$ such that for all $t > 0$ and all $y \in \mathcal{D}$
\[
\| \Phi(x,t) - \Phi(y,t) \| \leq \| x - y \| e^{Lt}
\]
and for all $t \in [0, \tau(x) - \epsilon_1]$,
\[
\| \Phi(x, t) - \Phi(y, t) \| \leq \| x - y \| e^{L (\tau(x) - \epsilon_1) }
\]
and w.l.o.g. we can assume that $\forall y \in U_x$, we have $\| x - y \| e^{L (\tau(x) - \epsilon_1) } < \frac{d}{2}$.
Thus we have for all $v \in V$
\begin{align*}
0 < d
&
\leq \| \Phi(x, t) - v \| = \| \Phi(x, t) - \Phi(y, t) + \Phi(y, t) - v \|
\\
&
\leq \| \Phi(x, t) - \Phi(y, t) \| + \| \Phi(y, t) - v \|
\\
&
\leq \| x - y \| e^{L (\tau(x) - \epsilon_1) } + \| \Phi(y, t) - v \|
< \frac{d}{2} + \| \Phi(y, t) - v \|
\end{align*}
and so for all $y \in U_x$, we have $\inf \{ \| \Phi(y, t) - v \| : v \in V, t \in [0, \tau(x) - \epsilon_1] \} \geq \frac{d}{2} > 0$ and thus $\tau(y) > \tau(x) - \epsilon_1$.
So $\tau$ is continuous on $\mathcal{D}$.
Let then $T := \sup_{x \in \mathcal{D}} \tau(x) < \infty$. Note that for all $x \in \mathcal{D}$ we have that for all $t > T$, $\Phi(x,t) \in U'$ and $B_\delta(\Phi(x,t)) \subset U$.

Let further $\eta > 0$.
Then with proposition \ref{prop:MBL-DPU_conv}, there are $\theta > 0$, $n_\theta \in \NN$ such that for all $x \in \mathcal{D}$,
\[
\Pr( X^\theta (n_\theta) \in B_{\delta} ( \Phi(x, T) ) \subset U | X^\theta (0) = x) > \eta
\]
and so
\[
\Pr( X^\theta (n_\theta) \in U ) > \eta.
\]

From here it is easy to see that the first hit time of $U$ for $\{X^\theta(t)\}_{t \in \NN_0}$ is almost surely finite, i.e., the earliest time $t$ for which $X^\theta (t) \in U$:
Let $Z(k) := X^\theta (kn_\theta)$ for $k \in \NN_0$ and let $S$ be the first hit time of $U$ for $\{Z(k)\}_{k \in \NN_0}$, such that $S$ is a random variable with values in $\NN_0 \cup \{\infty\}$.
Clearly the first hit time of $U$ for $\{X^\theta(t)\}_{t \in \NN_0}$ is smaller than for $\{Z(k)\}_{k \in \NN_0}$.

We have that for all $z \in \mathcal{D}$ and all $k \in \NN$:
\[
\Pr( Z_{k+1} \in B_{\delta} ( \Phi(z, T) ) \subset U | Z_k = z ) > \eta
\]
and hence
\[
\Pr( Z_{k+1} \in U ) > \eta .
\]

Then we have for $S$,
\begin{align*}
\Pr( S \leq k + 1 )
& = \Pr( S \leq k ) + (1 - \Pr( S \leq k ) ) \Pr( Z_{k+1} \in U )
>
\Pr( S \leq k ) ( 1 - \eta) + \eta
\end{align*}
and a quick induction argument yields:
\begin{align*}
\Pr( S \leq k + 1 )
&
>
1 - (1-\eta)^k \big(1 - (1 - \eta) \Pr( S = 0 )\big)
\end{align*}

The probability of a finite hitting time is then:
\begin{align*}
\Pr( S \in \NN_0 )
= \lim_{k \rightarrow \infty} \Pr( S \leq k + 1 )
\geq 1 - \lim_{k \rightarrow \infty} (1-\eta)^k (1 - (1 - \eta) \Pr( S = 0 ))
= 1
\end{align*}
In particular, the hitting time of $U$ for $\{X^\theta (t)\}_{t \in \NN_0}$ is finite almost surely.
\end{proof}

The previous proposition \ref{prop:finite_time} together with the consideration that an attracting mutation limit is approximated by asymptotically stable mutation equilibria and the immediately following corollary show proposition \ref{prop:MBL-DPU_attr_mut_lim}:
\begin{corollary}
If $x^M$ is a globally asymptotically stable equilibrium of \eqref{eq:RMD} and $U$ an open neighbourhood of $x^M$,
then there is $\theta > 0$ such that the stochastic process $\{X^\theta(n)\}_{n \geq 0}$ defined in \eqref{eq:defSimpleMBL} visits $U$ infinitely often almost surely.
\end{corollary}
\begin{proof}
Consider for any finite $t' \in \NN_0$ the probability that $\{X^\theta(n)\}_{n \geq 0}$ will not visit $U$ afterwards.
This is clearly the same as the probability that the process $\{Z^\theta(n)\}_{n \geq 0}$ induced by \eqref{eq:defSimpleMBL} and starting in $X^\theta(t')$, i.e., $Z^\theta(0) = X^\theta(t')$ almost surely, will not visit $U$ at all.
The previous proposition \ref{prop:finite_time} shows that this probability is $0$, which concludes the proof.
\end{proof}

\section{Specification of experiments and further results}\label{app:experiments}%

This section provides the specification details for the experimental results of section \ref{sec:exp_res} and further results for a broader range of parameter values. It is structured as follows: Each game setting is introduced with its payoff structure together with further results and a short description of the results, in the order of Prisoner's Dilemma (\ref{app:exp_spec_PD}), Matching Pennies (\ref{app:exp_spec_MP}), RPS-$n$ games (\ref{app:exp_spec_RPS}), and three-player Matching Pennies (\ref{app:exp_spec_3MP}). 
For the two-player settings, the payoff values are given as matrices $R_1$ and $R_2$, giving the payoffs for players one and two respectively, such that if player one chooses the $i$-th pure strategy from $A_1$ and player two chooses the $j$-th pure strategy from $A_2$, then the payoffs are given as $r_1(i,j) = [R_1]_{ij}$ and $r_2(i,j) = [R_2]_{ij}$ respectively. 
The experiments were run on a small cluster of multi-kernel CPUs, but we have checked that they can easily be run on personal hardware.

\subsection{Prisoner's Dilemma}\label{app:exp_spec_PD}

The experimental results for the Prisoner's Dilemma are based on the following payoff structure:
\begin{align*}
R_1 =
\begin{pmatrix}
1 & 5 \\
0 & 3
\end{pmatrix}
& &
R_2 =
\begin{pmatrix}
1 & 0 \\
5 & 3
\end{pmatrix}
\end{align*}

This version has a strict unique Nash equilibrium $x^*$ at:
\begin{align*}
x^*_1 = &
\begin{pmatrix}
1 & 0
\end{pmatrix}^T
&
x^*_2 = &
\begin{pmatrix}
1 & 0
\end{pmatrix}^T
\end{align*}

\paragraph{MBL-DPU and MBL-LC.}
The experimental results (figures \ref{fig:PD_MBL-DPU_high_mut}, \ref{fig:PD_MBL-DPU_low_mut_appx}) illustrate the behaviour of MBL-DPU and its convergence for different mutation strengths $M$.
In accordance with intuition, convergence is quick for high mutation strength at the price of the mutation equilibrium being further away from the Nash equilibrium. For lower values of $M$, we have that the mutation equilibrium moves closer to the Nash equilibrium while convergence becomes slower. 
In comparison, MBL-LC (figures \ref{fig:PD_MBL-LC_high_mut}, \ref{fig:PD_MBL-LC_low_mut}) behaves similarly while converging much more quickly. An intuition for this is provided when considering that MBL-DPU can be viewed as a linear approximation to MBL-LC for small $\tau$.

\paragraph{FAQ-learning.}
For FAQ-learning (figures \ref{fig:PD_FAQ_high_mut}, \ref{fig:PD_FAQ_low_mut}), the role of $\tau$ corresponds to that of $M^{-1}$ in MBL. We have that, similarly to both MBL variants, with increasing values of $\tau$ (i.e., decreasing values of $M$), the dynamics approaches a region that lies closer to the Nash equilibrium.
The intuition here is provided by the fact that the deterministic limit of FAQ is claimed to be a replicator dynamics with a perturbative term whose effect depends on $\tau$ and which pulls the system towards the centre of $\mathcal{D}$. 
Furthermore, convergence is the slower the weaker the perturbative term is, much like in the two MBL variants. 
In contrast to the MBL variants, FAQ-learning defaults to the usual Q-learning when $x_{ih} \leq \beta$. This effectively neutralises the repelling dynamics at the boundary of $\mathcal{D}$, which would otherwise result in very large (unbounded) changes in the Q-values for very low values of $x_{ih}$.
Note that MBL-LC has $x_{ih}$ occurring in the denominator twice and hence retains the repelling effect at the boundary of $\mathcal{D}$.

\paragraph{WoLF-PHC.}
In contrast to the other algorithms, WoLF-PHC (figure \ref{fig:PD_WoLF-PHC}) follows a chosen direction for some time until it is replaced by a new direction, which results in a discrete sequence of directions and non-smooth trajectories. Convergence to the Nash equilibrium occurs much faster than for the other algorithms in the case of PD.
However, strict Nash equilibria are also asymptotically stable in RD and thus PD is a base case which illustrates the different behaviours in a clear-cut situation, as opposed to more challenging and ambiguous situations without strict Nash equilibria.

\begin{figure}[h] %
    \begin{center}
    	\begin{subfigure}[b]{0.3\linewidth}
    	\centering
    	\includegraphics[width=\textwidth, trim={0 770bp 0 0}, clip ]{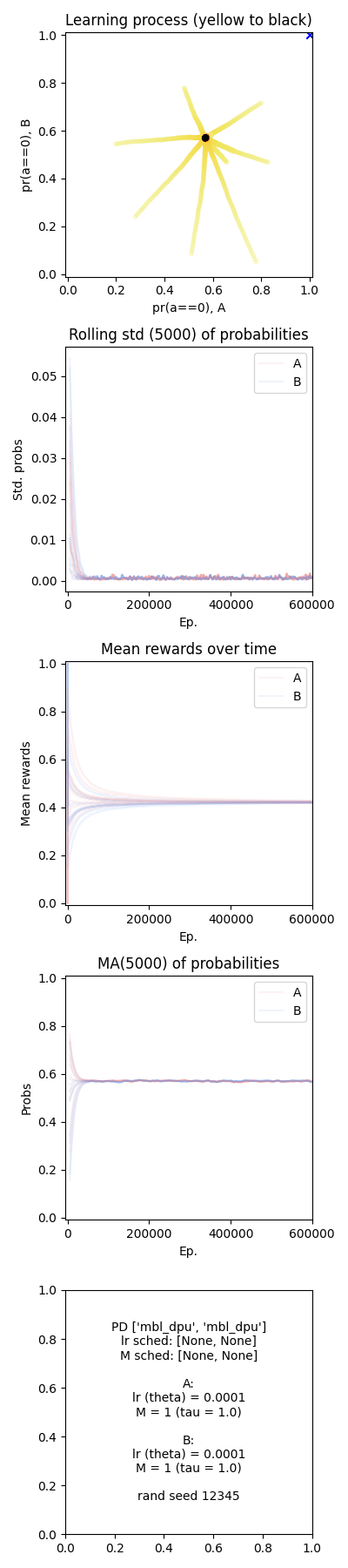}
    	\caption{$\tau = 1$, $M = 1$}
    	\end{subfigure}
    	\hfill
    	\begin{subfigure}[b]{0.3\linewidth}
    	\centering
    	\includegraphics[width=\textwidth, trim={0 770bp 0 0}, clip ]{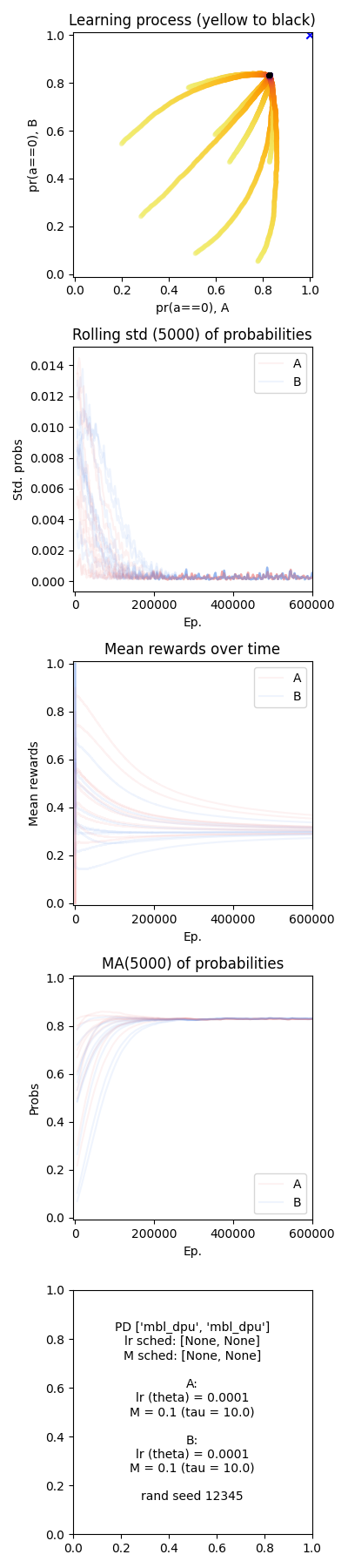}
        \caption{$\tau = 10$, $M = 10^{-1}$}
    	\end{subfigure}
    	\hfill
    	\begin{subfigure}[b]{0.3\linewidth}
    	\centering
    	\includegraphics[width=\textwidth, trim={0 770bp 0 0}, clip ]{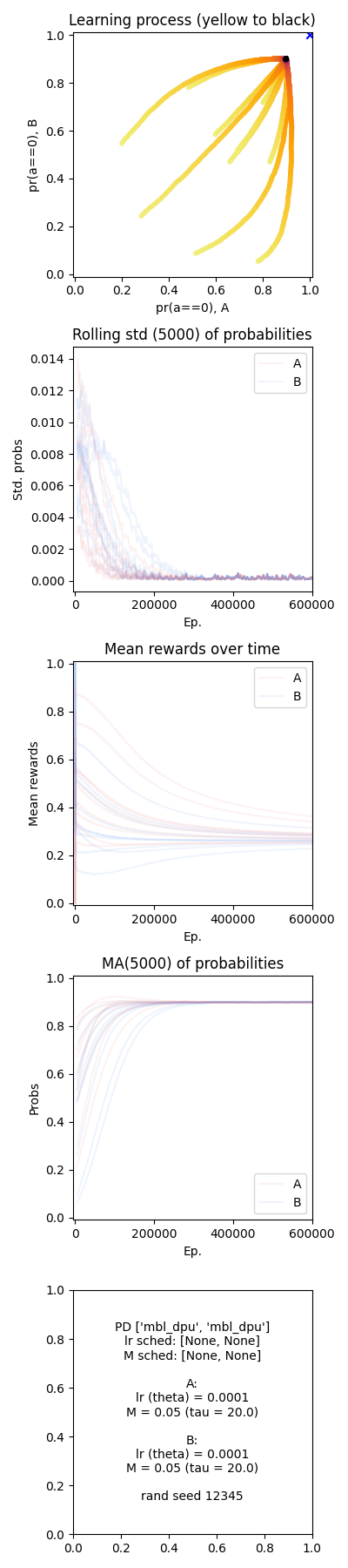}
    	\caption{$\tau = 20$, $M = 20^{-1}$}
    	\end{subfigure}
    \end{center}
    \caption[MBL-DPU in self-play on the PD game.]{MBL-DPU in self-play on the PD game with different values for $\tau$ ($1$, $10$, $20$) or $M$ ($1$, $10^{-1}$, $20^{-1}$) equivalently; $\theta = 10^{-4}$; for 10 different initial conditions. In each subfigure, the upper graph shows the ten trajectories in the projection on the first components of the players' strategies, in this case the `defect' strategy, with the first player given on the horizontal axis and the second player on the vertical axis. Points coloured yellow correspond to earlier points in time, changing over orange and violet to black for later points in time. The position of the game's Nash equilibrium is marked with a blue cross in the projection plane.
    The lower graph shows the standard deviation of all components of the players' strategies for each point in time over the past 5000 time steps, for each of the ten initial conditions, coloured red and blue for the two players. Time is given on the horizontal axis. The standard deviation is computed with the usual Euclidean metric.
    } %
    \label{fig:PD_MBL-DPU_high_mut}
\end{figure}
\begin{figure}[p] %
    \begin{center}
    	\begin{subfigure}[b]{0.3\linewidth}
    	\centering
    	\includegraphics[width=\textwidth, trim={0 770bp 0 0}, clip ]{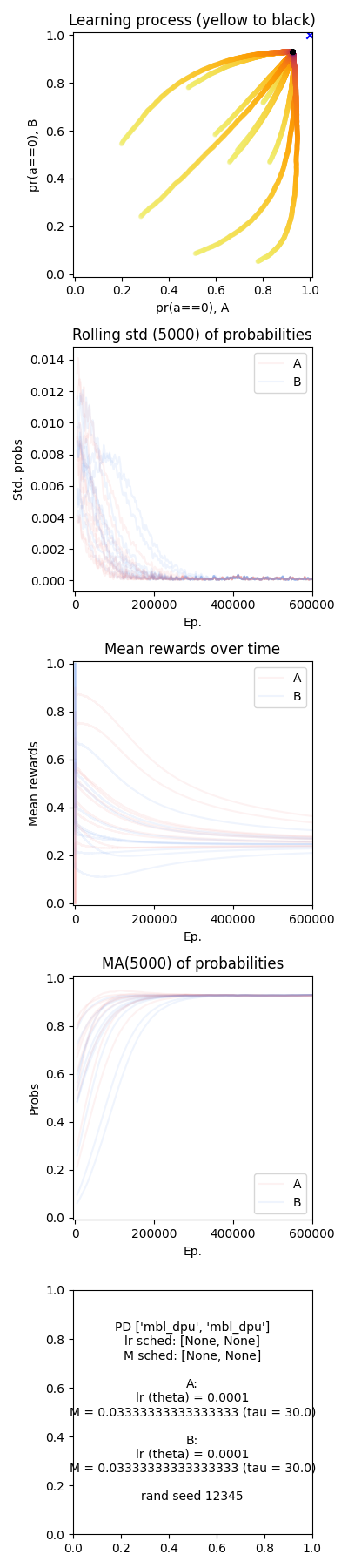}
    	\caption{$\tau = 30$, $M = 30^{-1}$}
    	\end{subfigure}
    	\hfill
    	\begin{subfigure}[b]{0.3\linewidth}
    	\centering
    	\includegraphics[width=\textwidth, trim={0 770bp 0 0}, clip ]{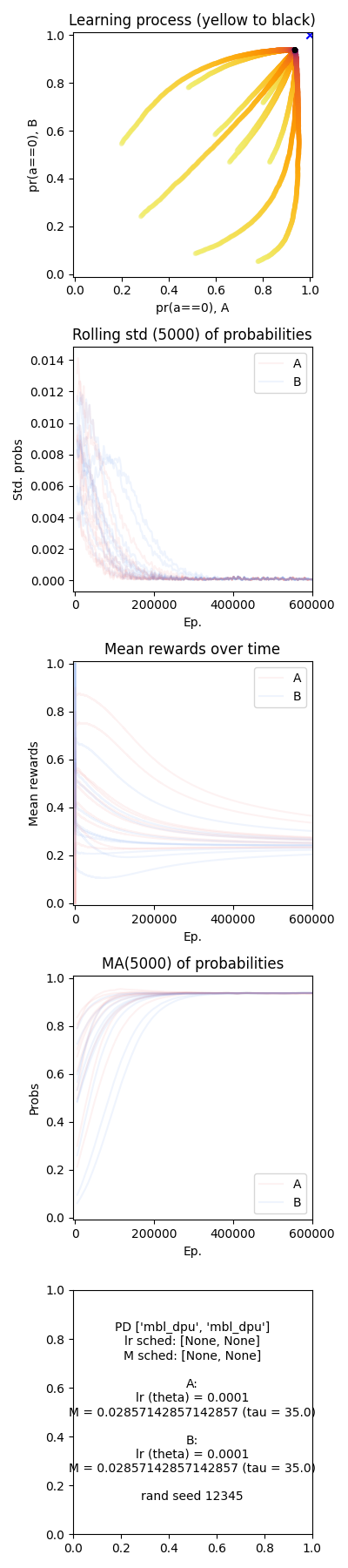}
        \caption{$\tau = 35$, $M = 35^{-1}$}
    	\end{subfigure}
    	\hfill
    	\begin{subfigure}[b]{0.3\linewidth}
    	\centering
    	\includegraphics[width=\textwidth, trim={0 770bp 0 0}, clip ]{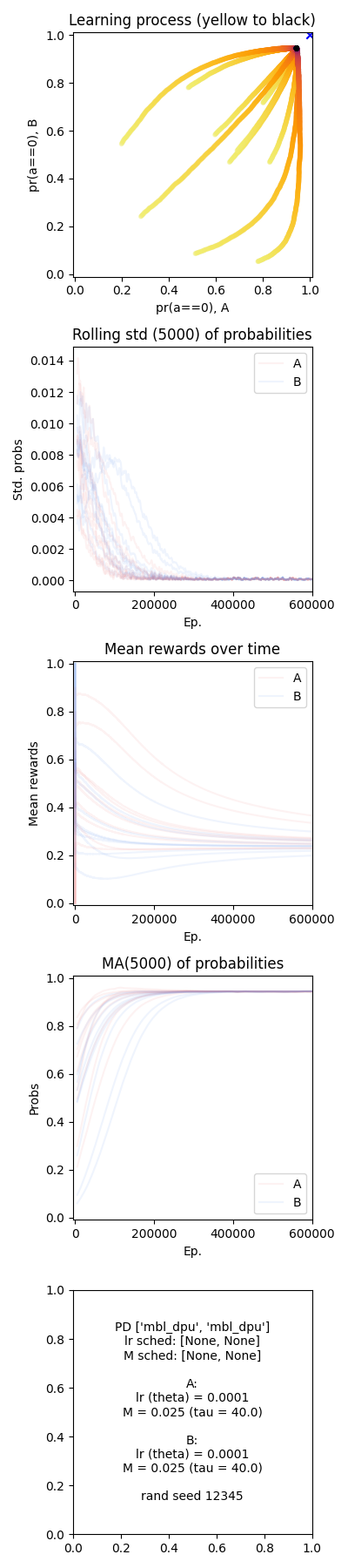}
    	\caption{$\tau = 40$, $M = 40^{-1}$}
    	\end{subfigure}
    \end{center}
    \caption[MBL-DPU in self-play on the PD game.]{MBL-DPU in self-play on the PD game with different values for $\tau$ ($30$, $35$, $40$) or $M$ ($30^{-1}$, $35^{-1}$, $40^{-1}$) equivalently; $\theta = 10^{-4}$; for 10 different initialisations.
    (See figure \ref{fig:PD_MBL-DPU_high_mut} for a detailed explanation of the graphs.)
    }
    \label{fig:PD_MBL-DPU_low_mut_appx}
\end{figure}
\begin{figure}[p] %
    \begin{center}
    	\begin{subfigure}[b]{0.3\linewidth}
    	\centering
    	\includegraphics[width=\textwidth, trim={0 770bp 0 0}, clip ]{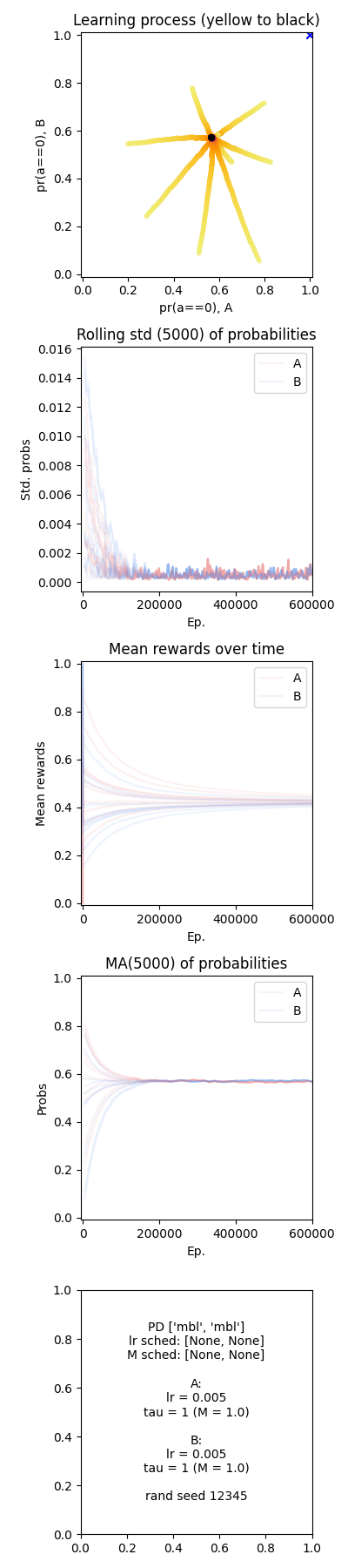}
        \caption{$\tau = 1$, $M = 1^{-1}$}
    	\end{subfigure}
    	\hfill
    	\begin{subfigure}[b]{0.3\linewidth}
    	\centering
    	\includegraphics[width=\textwidth, trim={0 770bp 0 0}, clip ]{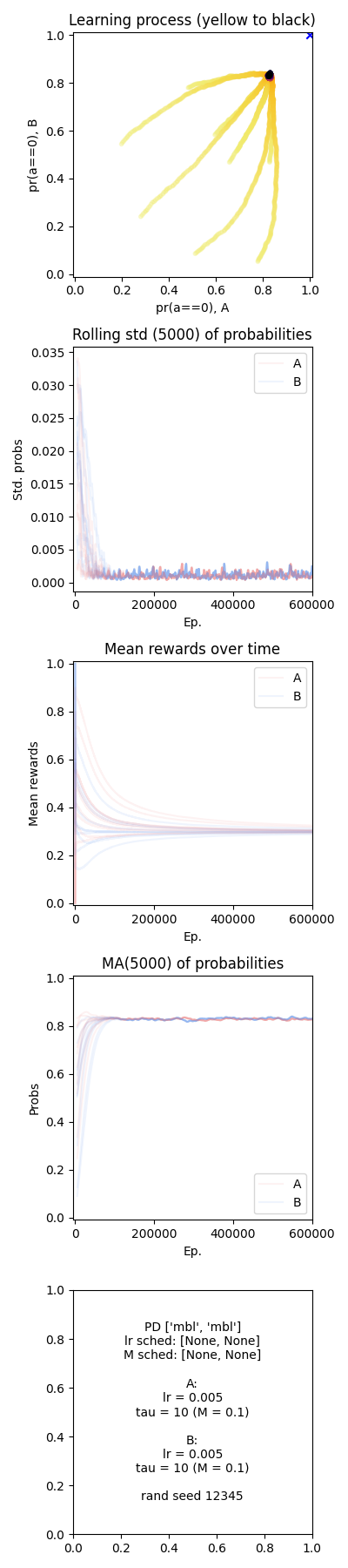}
    	\caption{$\tau = 10$, $M = 10^{-1}$}
    	\end{subfigure}
    	\hfill
    	\begin{subfigure}[b]{0.3\linewidth}
    	\centering
    	\includegraphics[width=\textwidth, trim={0 770bp 0 0}, clip ]{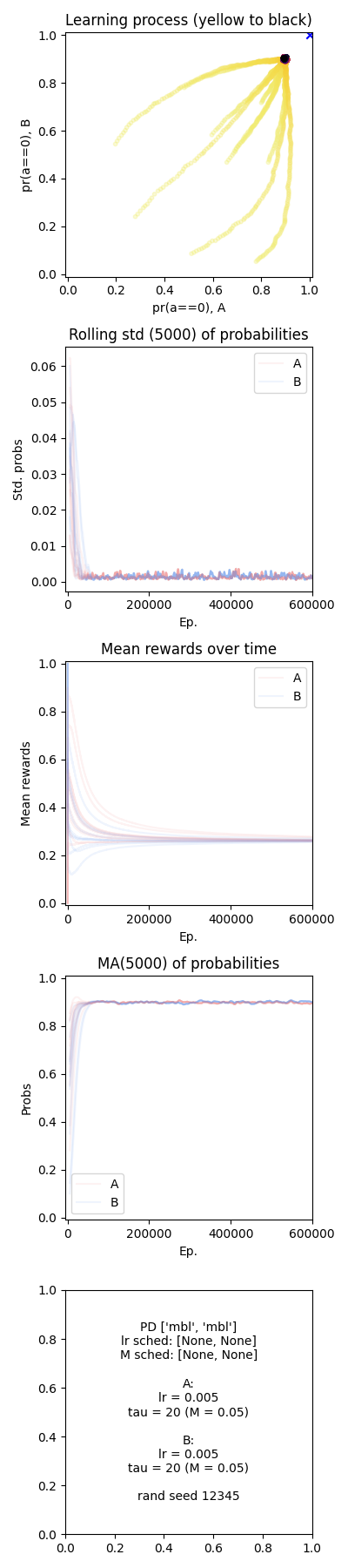}
    	\caption{$\tau = 20$, $M = 20^{-1}$}
    	\end{subfigure}
    \end{center}
    \caption[MBL-LC in self-play on the PD game.]{MBL-LC in self-play on the PD game with different values for $\tau$ ($1$, $10$, $20$) or $M$ ($1$, $10^{-1}$, $20^{-1}$) equivalently; $\theta = 5 \cdot 10^{-3}$; for 10 different initialisations.
    (See figure \ref{fig:PD_MBL-DPU_high_mut} for a detailed explanation of the graphs.)
    }
    \label{fig:PD_MBL-LC_high_mut}
\end{figure}
\begin{figure}[p] %
    \begin{center}
    	\begin{subfigure}[b]{0.3\linewidth}
    	\centering
    	\includegraphics[width=\textwidth, trim={0 770bp 0 0}, clip ]{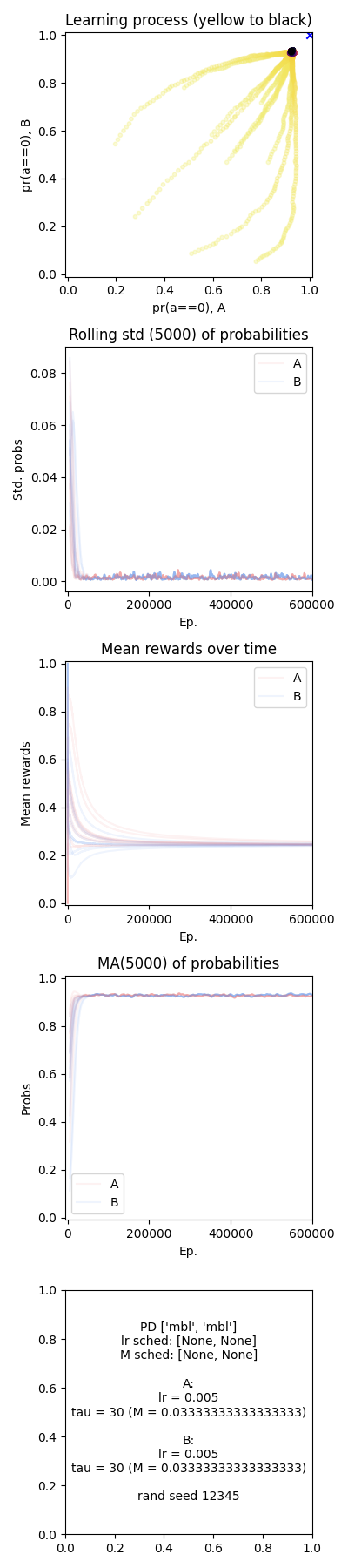}
    	\caption{$\tau = 30$, $M = 30^{-1}$}
    	\end{subfigure}
    	\hfill
    	\begin{subfigure}[b]{0.3\linewidth}
    	\centering
    	\includegraphics[width=\textwidth, trim={0 770bp 0 0}, clip ]{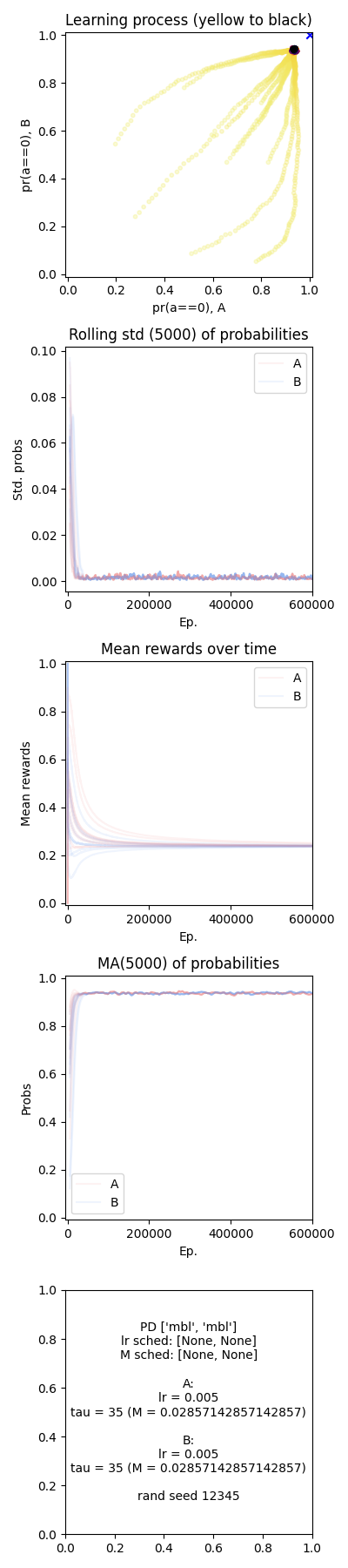}
        \caption{$\tau = 35$, $M = 35^{-1}$}
    	\end{subfigure}
    	\hfill
    	\begin{subfigure}[b]{0.3\linewidth}
    	\centering
    	\includegraphics[width=\textwidth, trim={0 770bp 0 0}, clip ]{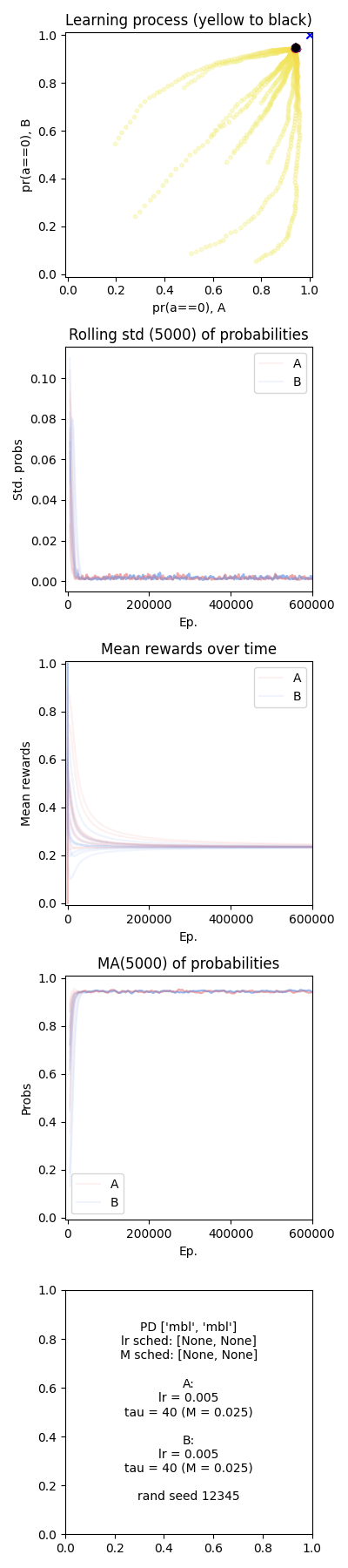}
    	\caption{$\tau = 40$, $M = 40^{-1}$}
    	\end{subfigure}
    \end{center}
    \caption[MBL-LC in self-play on the PD game.]{MBL-LC in self-play on the PD game with different values for $\tau$ ($30$, $35$, $40$) or $M$ ($30^{-1}$, $35^{-1}$, $40^{-1}$) equivalently; $\theta = 5 \cdot 10^{-3}$; for 10 different initialisations.
    (See figure \ref{fig:PD_MBL-DPU_high_mut} for a detailed explanation of the graphs.)
    }
    \label{fig:PD_MBL-LC_low_mut}
\end{figure}
\begin{figure}[p] %
    \begin{center}
    	\begin{subfigure}[b]{0.3\linewidth}
    	\centering
    	\includegraphics[width=\textwidth, trim={0 770bp 0 0}, clip ]{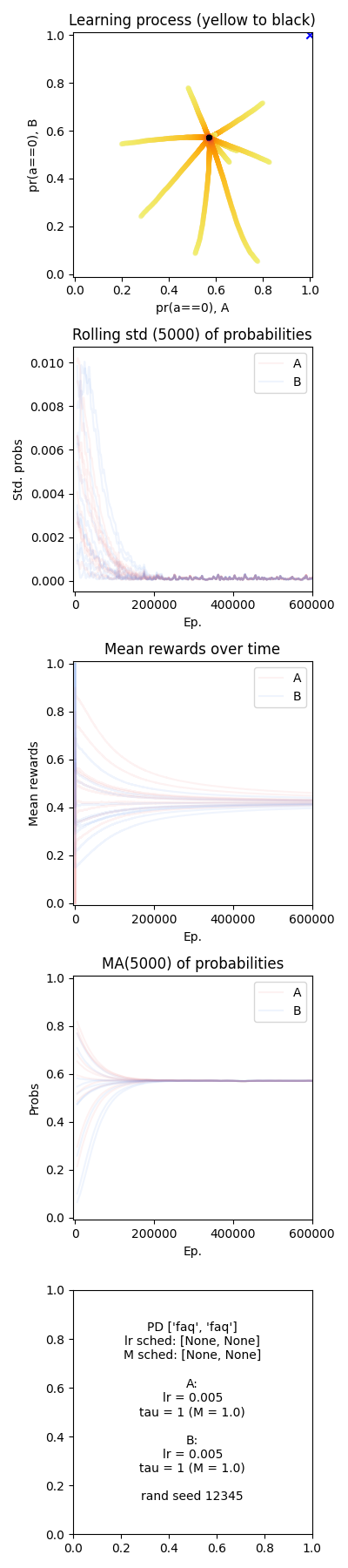}
    	\caption{$\tau = 1$, $M = 1^{-1}$}
    	\end{subfigure}
    	\hfill
    	\begin{subfigure}[b]{0.3\linewidth}
    	\centering
    	\includegraphics[width=\textwidth, trim={0 770bp 0 0}, clip ]{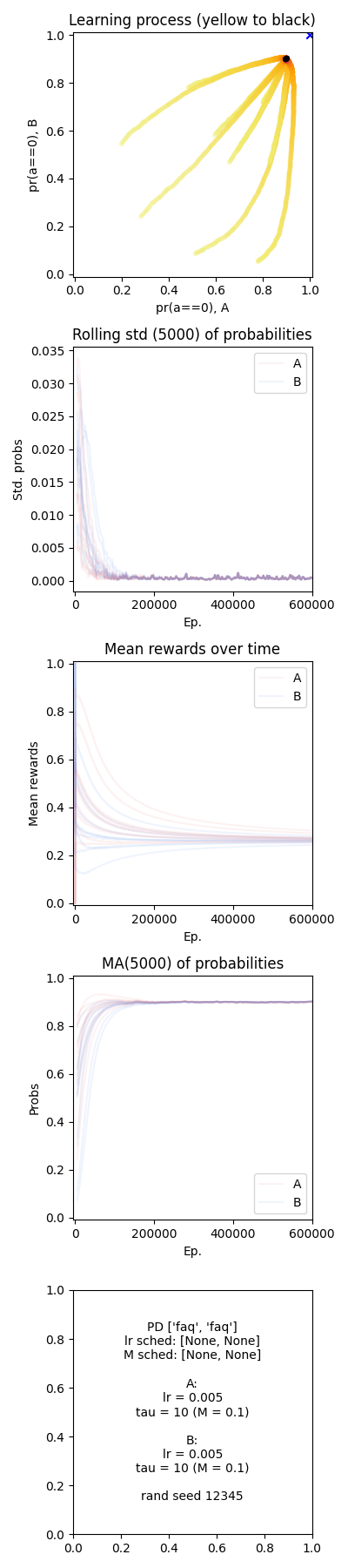}
        \caption{$\tau = 10$, $M = 10^{-1}$}
    	\end{subfigure}
    	\hfill
    	\begin{subfigure}[b]{0.3\linewidth}
    	\centering
    	\includegraphics[width=\textwidth, trim={0 770bp 0 0}, clip ]{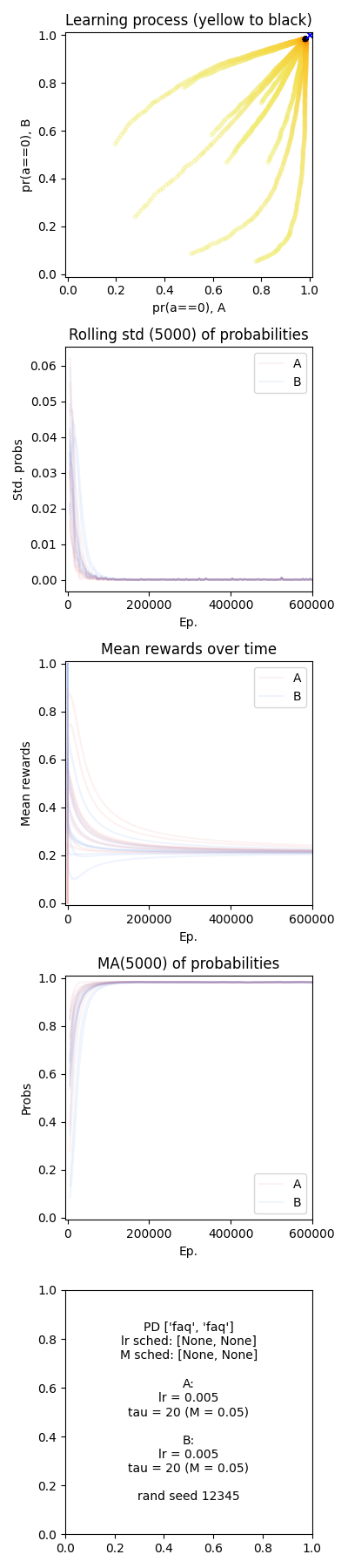}
    	\caption{$\tau = 20$, $M = 20^{-1}$}
    	\end{subfigure}
    \end{center}
    \caption[FAQ in self-play on the PD game.]{FAQ in self-play on the PD game with different values for $\tau$ ($1$, $10$, $20$) or $M$ ($1$, $10^{-1}$, $20^{-1}$) equivalently; $\theta = 5 \cdot 10^{-3}$; for 10 different initialisations.
    (See figure \ref{fig:PD_MBL-DPU_high_mut} for a detailed explanation of the graphs.)
    }
    \label{fig:PD_FAQ_high_mut}
\end{figure}
\begin{figure}[p] %
    \begin{center}
    	\begin{subfigure}[b]{0.3\linewidth}
    	\centering
    	\includegraphics[width=\textwidth, trim={0 770bp 0 0}, clip ]{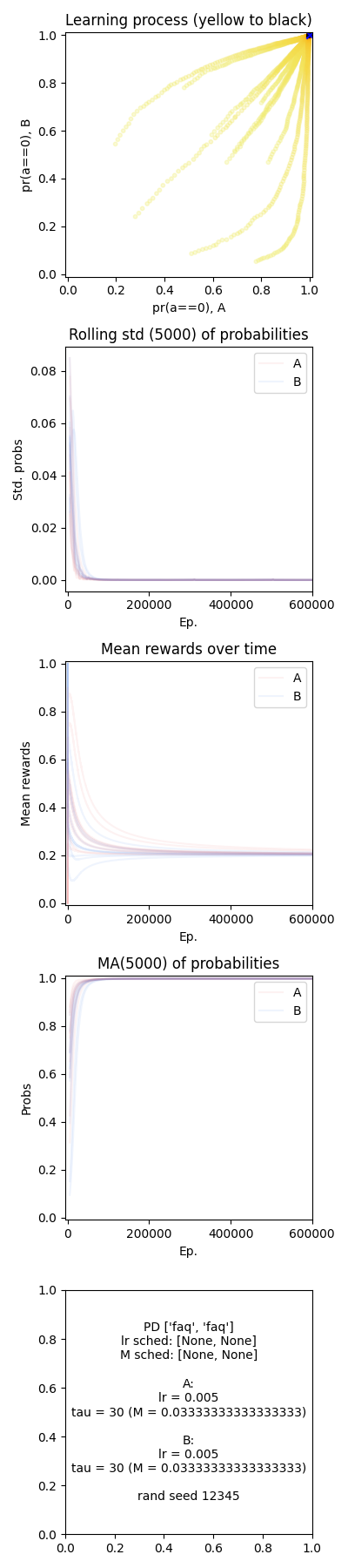}
    	\caption{$\tau = 30$, $M = 30^{-1}$}
    	\end{subfigure}
    	\hfill
    	\begin{subfigure}[b]{0.3\linewidth}
    	\centering
    	\includegraphics[width=\textwidth, trim={0 770bp 0 0}, clip ]{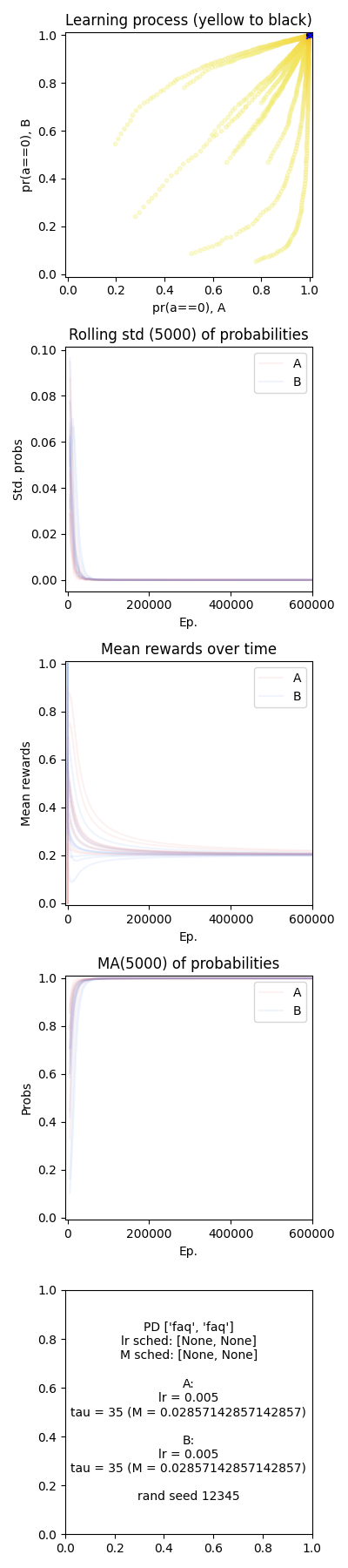}
        \caption{$\tau = 35$, $M = 35^{-1}$}
    	\end{subfigure}
    	\hfill
    	\begin{subfigure}[b]{0.3\linewidth}
    	\centering
    	\includegraphics[width=\textwidth, trim={0 770bp 0 0}, clip ]{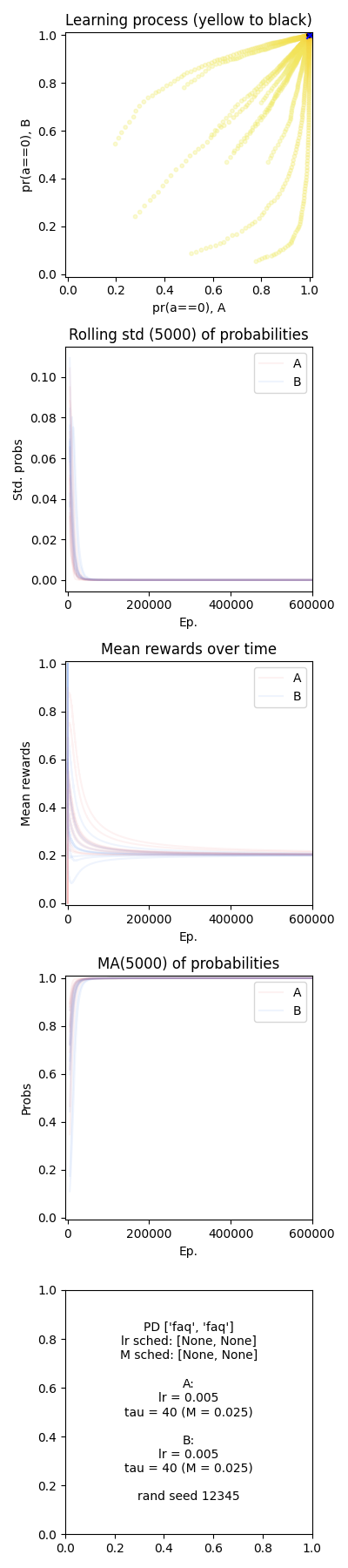}
    	\caption{$\tau = 40$, $M = 40^{-1}$}
    	\end{subfigure}
    \end{center}
    \caption[FAQ in self-play on the PD game.]{FAQ in self-play on the PD game with different values for $\tau$ ($30$, $35$, $40$) or $M$ ($30^{-1}$, $35^{-1}$, $40^{-1}$) equivalently; $\theta = 5 \cdot 10^{-3}$; for 10 different initialisations.
    (See figure \ref{fig:PD_MBL-DPU_high_mut} for a detailed explanation of the graphs.)
    }
    \label{fig:PD_FAQ_low_mut}
\end{figure}

\clearpage

\begin{figure}[!ht] %
    \begin{center}
    	\begin{subfigure}[b]{0.3\linewidth}
    	\centering
    	\includegraphics[width=\textwidth, trim={0 770bp 0 0}, clip ]{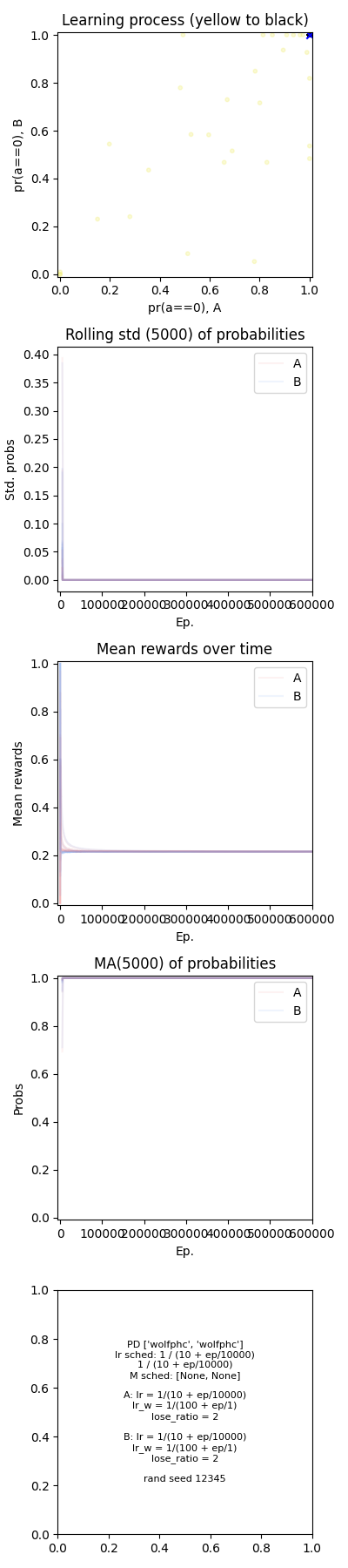}
    	\caption{Initial learning rate $10^{-1}$ for $Q$. Win learning rate $10^{-2}$.}
    	\end{subfigure}
    	\hfill
    	\begin{subfigure}[b]{0.3\linewidth}
    	\centering
    	\includegraphics[width=\textwidth, trim={0 770bp 0 0}, clip ]{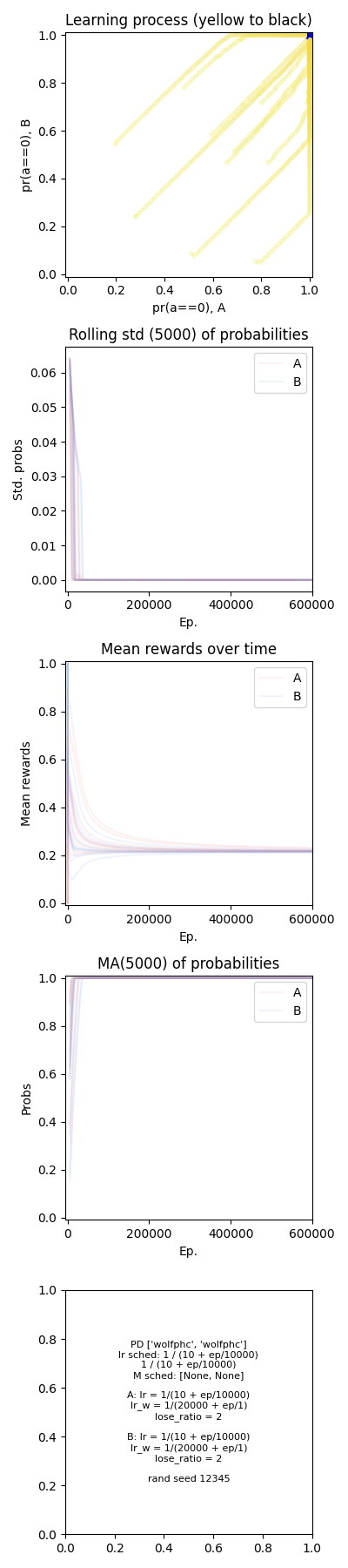}
        \caption{Initial learning rate $10^{-1}$ for $Q$. Win learning rate $1/2 \cdot 10^{-4}$.}
    	\end{subfigure}
    	\hfill
    	\begin{subfigure}[b]{0.3\linewidth}
    	\centering
    	\includegraphics[width=\textwidth, trim={0 770bp 0 0}, clip ]{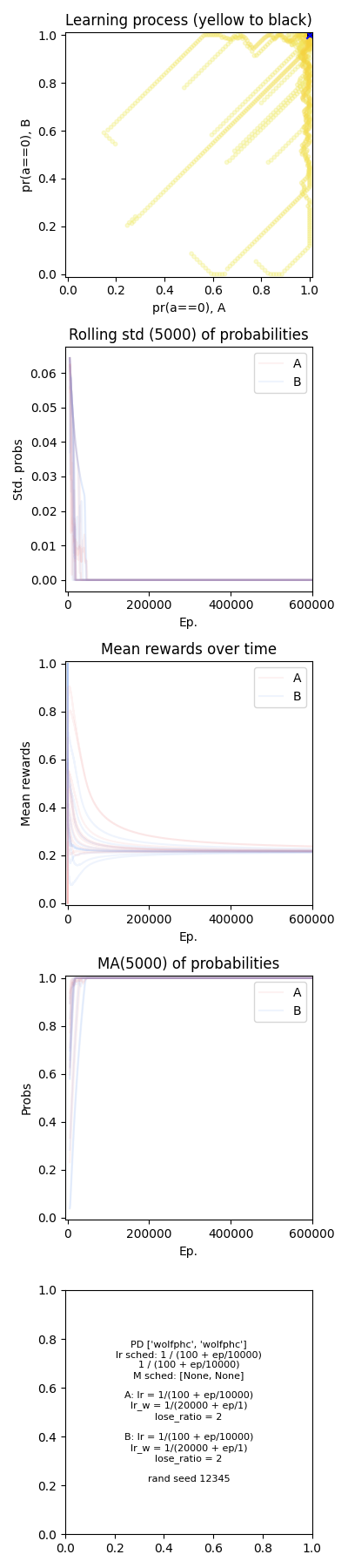}
    	\caption{Initial learning rate $10^{-2}$ for $Q$. Win learning rate $1/2 \cdot 10^{-4}$.}
    	\end{subfigure}
    \end{center}
    \caption[WoLF-PHC in self-play on the PD game.]{WoLF-PHC in self-play on the PD game with different learning schedules; for 10 different initialisations. Subgraph (a) has a high convergence speed such that only disconnected points can be seen.
    (See figure \ref{fig:PD_MBL-DPU_high_mut} for a detailed explanation of the graphs.)
    }
    \label{fig:PD_WoLF-PHC}
\end{figure}

\subsection{Zero-sum games}

For two-player zero-sum games, we have preliminary results showing that the Nash equilibrium is an attracting mutation limit. While RD (and Cross learning) would not converge to interior equilibria (with Cross learning eventually approaching the boundary), \ref{eq:RMD} converges to the mutation equilibrium for every choice of mutation probabilities, $c \in \setint{\mathcal{D}}$ and $M > 0$, and \todo{[C] Add reference to result.}{so does MBL-DPU.}\ 
Stability is induced by the perturbative terms and their varying strengths have two effects which have to be weighed against each other. We demonstrate the general idea in the simple situation of the Matching Pennies (MP) game. Further, we illustrate the changing behaviour when we grow the strategy space by considering different versions of the Rock-Paper-Scissors game, RPS-$n$, with $n=3,5,9$, where $n$ denotes the number of strategies available to each player.

\subsubsection{Matching Pennies}\label{app:exp_spec_MP}
The experimental results for the Matching Pennies game are based on the following payoff structure:
\begin{align*}
R_1 = 
\begin{pmatrix}
  1 & -23/10 \\
 -4/10 &  1
\end{pmatrix}
& &
R_2 =
\begin{pmatrix}
 -23/10 &  1 \\
  1 & -4/10
\end{pmatrix}
\end{align*}

Nash equilibrium $x^*$ at:
\begin{align*}
x^*_1 = &
\begin{pmatrix}
14/47 & 33/47
\end{pmatrix}^T
&
x^*_2 = &
\begin{pmatrix}
33/47 & 14/47
\end{pmatrix}^T
\end{align*}

The MP game is a particularly simple case of a zero-sum game and hence provides an informative perspective on the basic characteristics of the different algorithms.
In general, we see that the location of the mutation equilibrium depends on the mutation strength $M$, while convergence is slower for lower values of $M$ creating a trade-off between these.

\paragraph{MBL-DPU and MBL-LC.}
Comparing MBL-DPU and MBL-LC, we see again that the LC-variant (figures \ref{fig:MP_MBL-LC_high_mut}, \ref{fig:MP_MBL-LC_low_mut}) approaches the mutation equilibrium more quickly than the DPU-variant (figures \ref{fig:MP_MBL-DPU_high_mut}, \ref{fig:MP_MBL-DPU_low_mut}). However, we see that the DPU-variant exhibits a much smaller variance, more precisely standard deviation, in the vicinity of the mutation equilibrium due to its slower change, with both variants roughly differing by a factor between 5 and 10 (for $M=40^{-1}$).
This illustrates the stronger effect that single larger payoffs have on the LC-variant, producing a larger variance near the mutation equilibrium.

\paragraph{FAQ-learning.}
For FAQ-learning (figures \ref{fig:MP_FAQ_high_mut}, \ref{fig:MP_FAQ_low_mut}) we see a similar behaviour as MBL-LC, however with a smaller variance near the equilibrium for weaker perturbation (figure \ref{fig:MP_FAQ_low_mut}). As with the MBL variants, FAQ exhibits slower convergence for weaker perturbation with larger variance near its (apparently asymptotically stable) equilibrium. However, we also observe that with FAQ, solutions can get trapped near the boundary (note the trapped solution in the upper left corner in figure \ref{fig:MP_FAQ_low_mut}), which we do not observe for the MBL variants and \todo{[C] Add reference to result.}{ have proved not to be the case for MBL-DPU.}

\paragraph{WoLF-PHC.}
Similar to the other algorithms, WoLF-PHC (figure \ref{app:fig:MP_WoLF-PHC}) follows spiral-like trajectories towards a region close to the Nash equilibrium. It also shows a lower variance near the (apparently asymptotically stable) equilibrium. However, WoLF-PHC employs a learning rate schedule which reduces the learning rate over time and thus reduces variance.%
\footnote{It would be possible to evaluate WoLF-PHC with a fixed learning rate or use a reduction schedule for the other algorithms. However, the former would be a deviation from the canonical formulation of WoLF-PHC while the latter would not be based on a principled approach. Hence, this heterogeneous situation is an appropriate base scenario.} 
One should note that WoLF-PHC is considerably more complicated as it relies on a reliable way to estimate action-values as well as a long-term population average. It is clear that a player would require more resources for implementing WoLF-PHC than for the other algorithms.

\begin{figure}[htb] %
    \vspace{3em}
    \begin{center}
    	\begin{subfigure}[b]{0.3\linewidth}
    	\centering
    	\includegraphics[width=\textwidth, trim={0 770bp 0 0}, clip ]{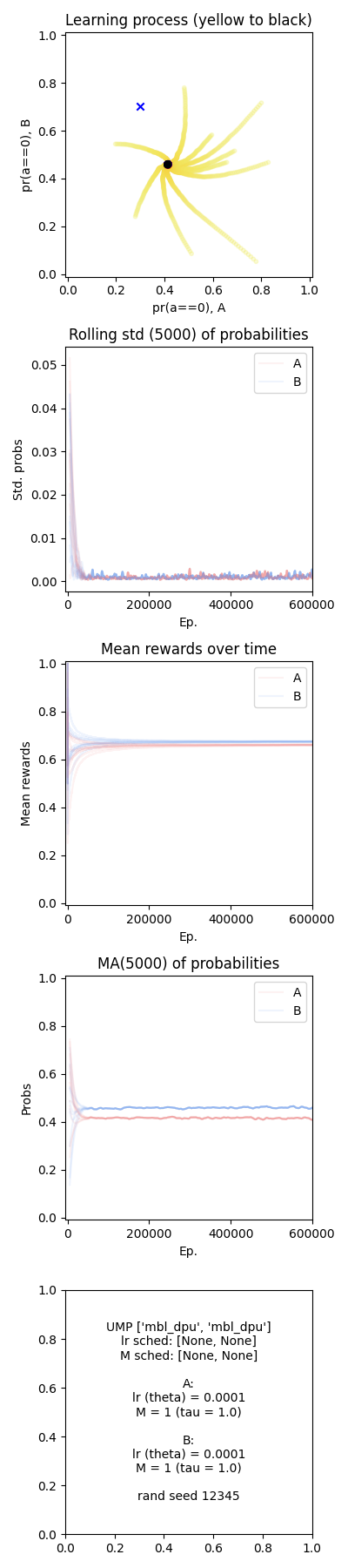}
    	\caption{$\tau = 1$, $M = 1^{-1}$}
    	\end{subfigure}
    	\hfill
    	\begin{subfigure}[b]{0.3\linewidth}
    	\centering
    	\includegraphics[width=\textwidth, trim={0 770bp 0 0}, clip ]{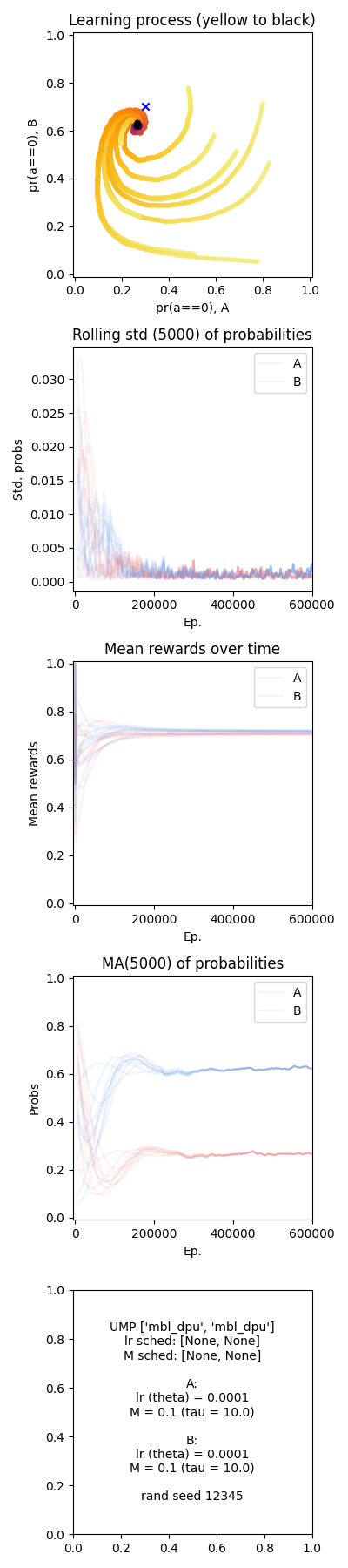}
    	\caption{$\tau = 10$, $M = 10^{-1}$}
    	\end{subfigure}
    	\hfill
    	\begin{subfigure}[b]{0.3\linewidth}
    	\centering
    	\includegraphics[width=\textwidth, trim={0 770bp 0 0}, clip ]{figures/results-solon/UMP/mbl_dpu_mbl_dpu/fixed_M_fixed_lr/viz/600000_eps_20200420_1526_3cfbf0f79404__combo.png}
        \caption{$\tau = 20$, $M = 20^{-1}$}
    	\end{subfigure}
    \end{center}
    \caption[MBL-DPU in self-play on the MP game.]{MBL-DPU in self-play on the MP game with different values for $\tau$ ($1$, $10$, $20$) or $M$ ($1$, $10^{-1}$, $20^{-1}$) equivalently; $\theta = 10^{-4}$; for 10 different initialisations.
    (See figure \ref{fig:PD_MBL-DPU_high_mut} for a detailed explanation of the graphs.)
    }
    \label{fig:MP_MBL-DPU_high_mut}
\end{figure}
\begin{figure}[ht] %
    \begin{center}
    	\begin{subfigure}[b]{0.3\linewidth}
    	\centering
    	\includegraphics[width=\textwidth, trim={0 770bp 0 0}, clip ]{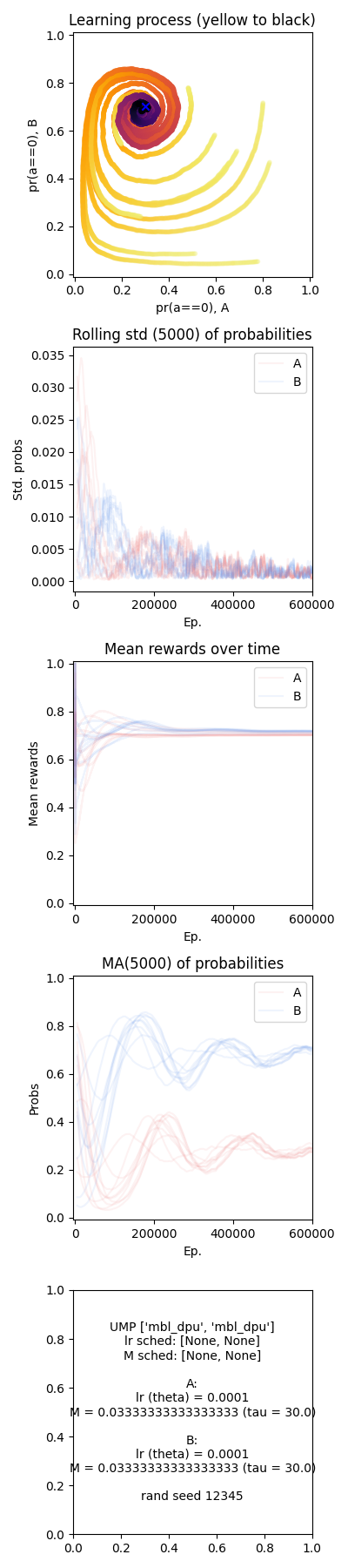}
        \caption{$\tau = 30$, $M = 30^{-1}$}
    	\end{subfigure}
    	\hfill
    	\begin{subfigure}[b]{0.3\linewidth}
    	\centering
    	\includegraphics[width=\textwidth, trim={0 770bp 0 0}, clip ]{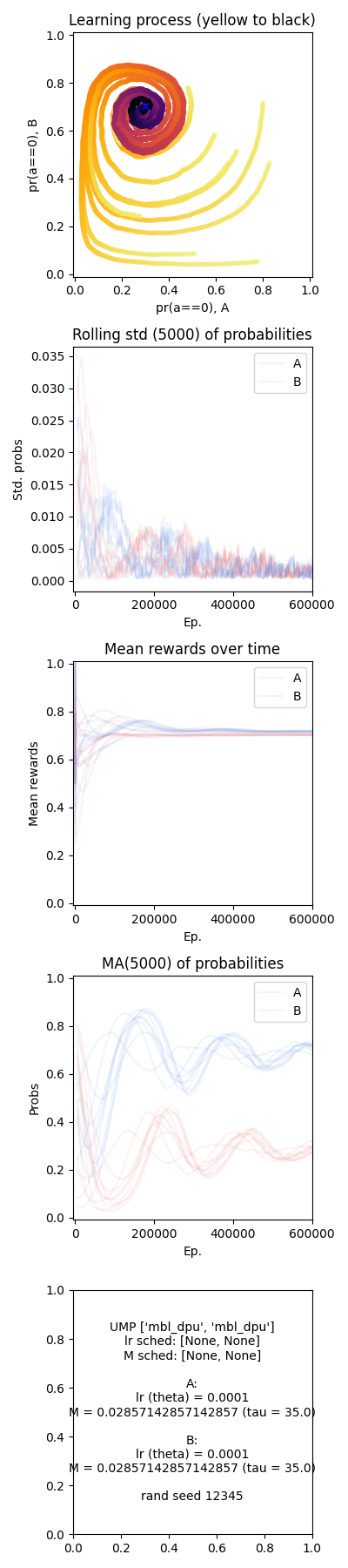}
    	\caption{$\tau = 35$, $M = 35^{-1}$}
    	\end{subfigure}
    	\hfill
    	\begin{subfigure}[b]{0.3\linewidth}
    	\centering
    	\includegraphics[width=\textwidth, trim={0 770bp 0 0}, clip ]{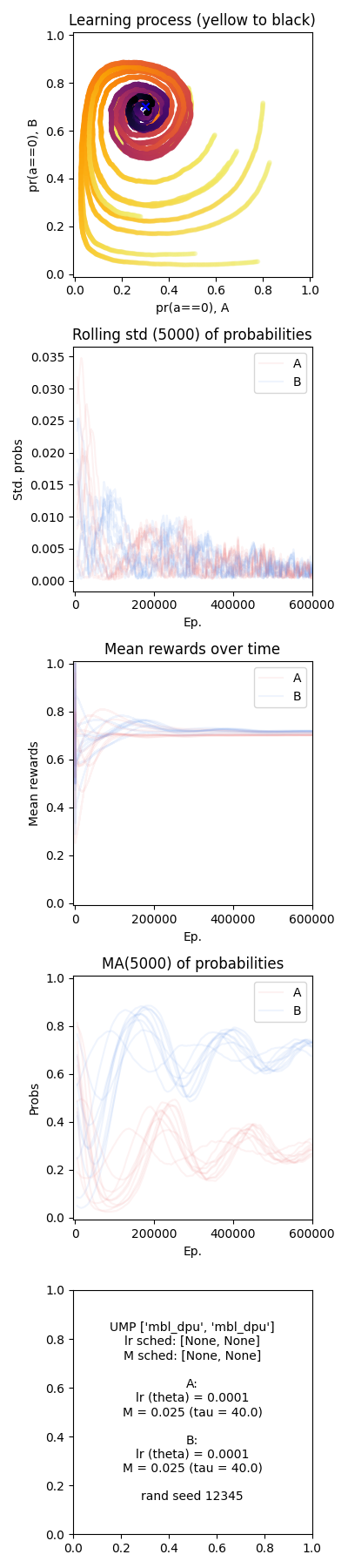}
    	\caption{$\tau = 40$, $M = 40^{-1}$}
    	\end{subfigure}
    \end{center}
    \caption[MBL-DPU in self-play on the MP game.]{MBL-DPU in self-play on the MP game with different values for $\tau$ ($30$, $35$, $40$) or $M$ ($30^{-1}$, $35^{-1}$, $40^{-1}$) equivalently; $\theta = 10^{-4}$; for 10 different initialisations.
    (See figure \ref{fig:PD_MBL-DPU_high_mut} for a detailed explanation of the graphs.)
    }
    \label{fig:MP_MBL-DPU_low_mut}
\end{figure}
\begin{figure}[ht] %
    \begin{center}
    	\begin{subfigure}[b]{0.3\linewidth}
    	\centering
    	\includegraphics[width=\textwidth, trim={0 770bp 0 0}, clip ]{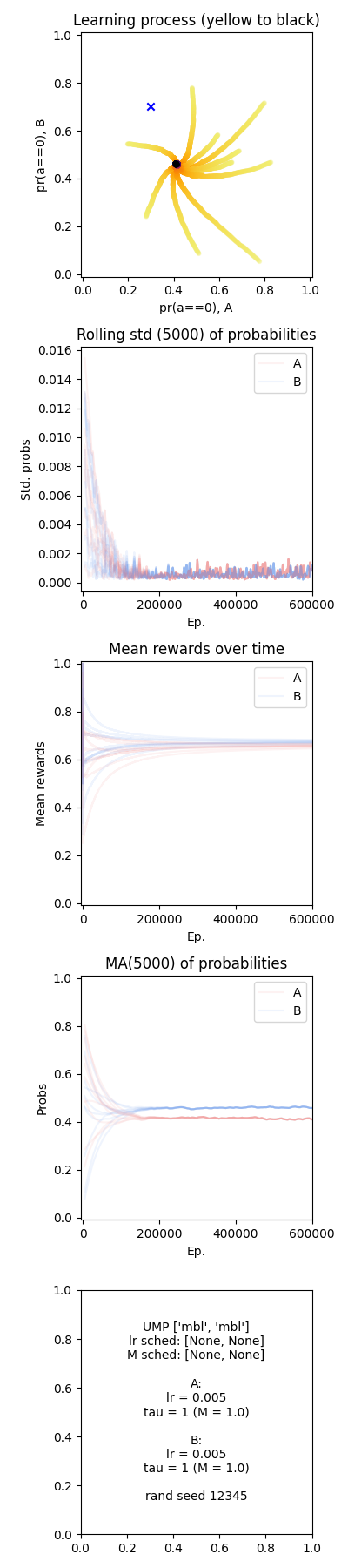}
        \caption{$\tau = 1$, $M = 1^{-1}$}
    	\end{subfigure}
    	\hfill
    	\begin{subfigure}[b]{0.3\linewidth}
    	\centering
    	\includegraphics[width=\textwidth, trim={0 770bp 0 0}, clip ]{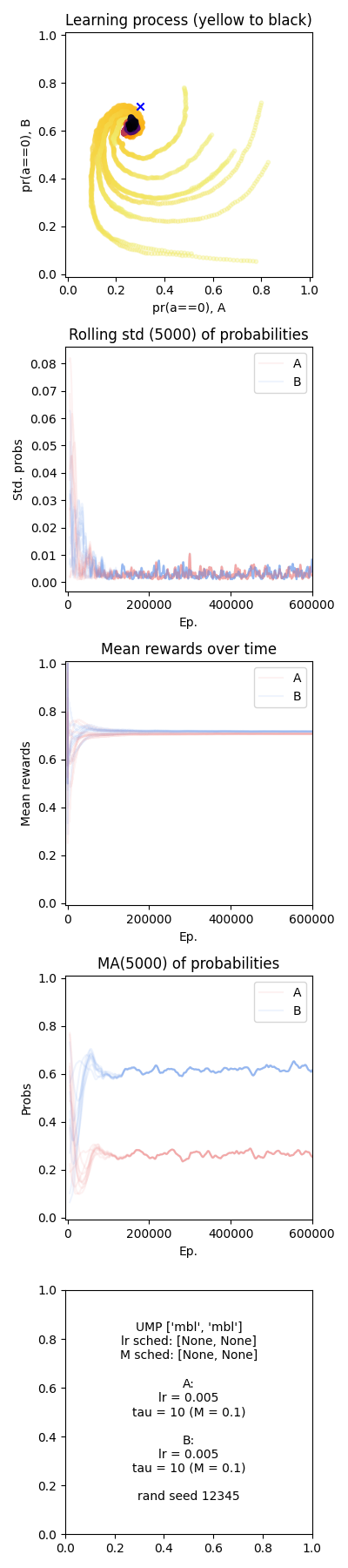}
    	\caption{$\tau = 10$, $M = 10^{-1}$}
    	\end{subfigure}
    	\hfill
    	\begin{subfigure}[b]{0.3\linewidth}
    	\centering
    	\includegraphics[width=\textwidth, trim={0 770bp 0 0}, clip ]{figures/results-solon/UMP/mbl_mbl/fixed_M_fixed_lr/viz/600000_eps_20200420_1533_4c2b89b5ad40__combo.png}
    	\caption{$\tau = 20$, $M = 20^{-1}$}
    	\end{subfigure}
    \end{center}
    \caption[MBL-LC in self-play on the MP game.]{MBL-LC in self-play on the MP game with different values for $\tau$ ($1$, $10$, $20$) or $M$ ($1$, $10^{-1}$, $20^{-1}$) equivalently; $\theta = 5 \cdot 10^{-3}$; for 10 different initialisations.
    (See figure \ref{fig:PD_MBL-DPU_high_mut} for a detailed explanation of the graphs.)
    }
    \label{fig:MP_MBL-LC_high_mut}
\end{figure}
\begin{figure}[ht] %
    \begin{center}
    	\begin{subfigure}[b]{0.3\linewidth}
    	\centering
    	\includegraphics[width=\textwidth, trim={0 770bp 0 0}, clip ]{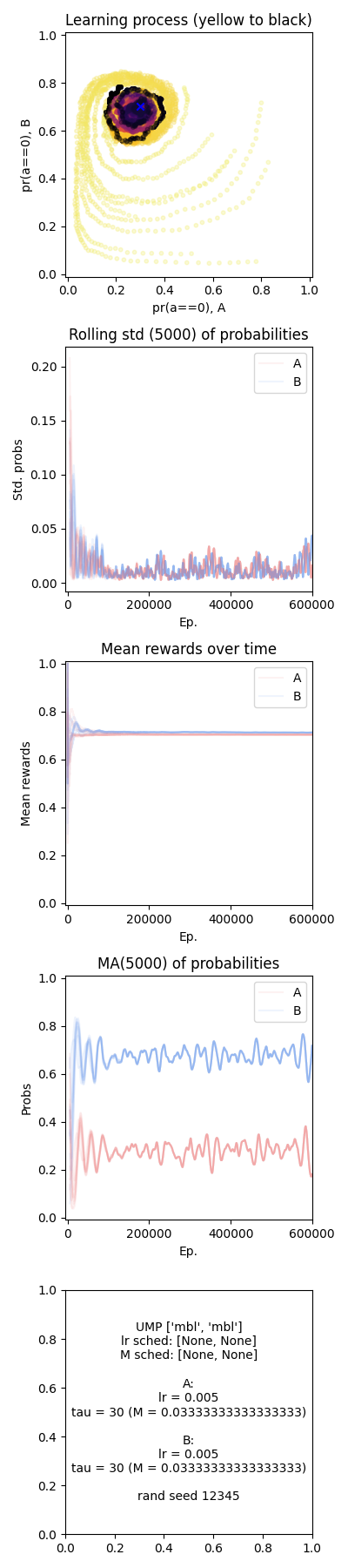}
    	\caption{$\tau = 30$, $M = 30^{-1}$}
    	\end{subfigure}
    	\hfill
    	\begin{subfigure}[b]{0.3\linewidth}
    	\centering
    	\includegraphics[width=\textwidth, trim={0 770bp 0 0}, clip ]{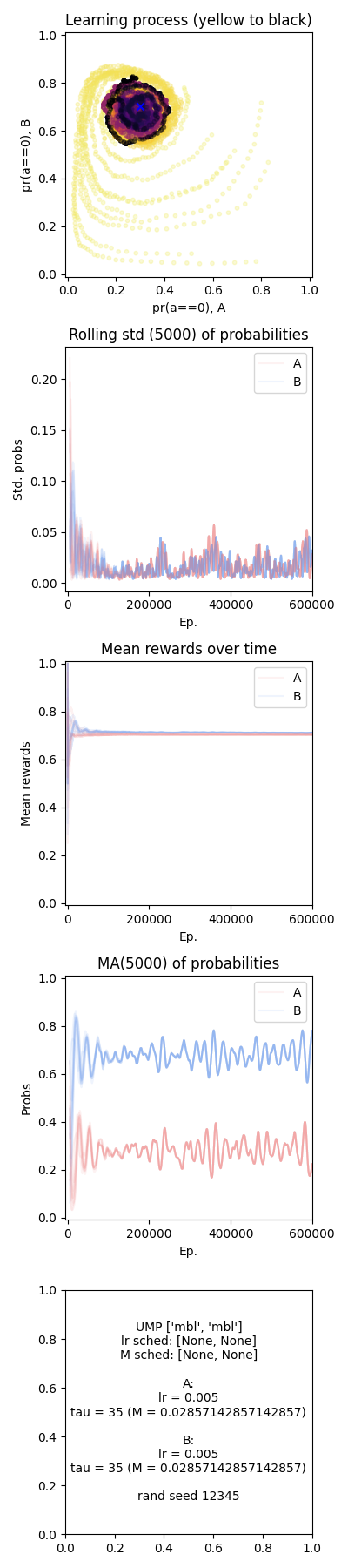}
        \caption{$\tau = 35$, $M = 35^{-1}$}
    	\end{subfigure}
    	\hfill
    	\begin{subfigure}[b]{0.3\linewidth}
    	\centering
    	\includegraphics[width=\textwidth, trim={0 770bp 0 0}, clip ]{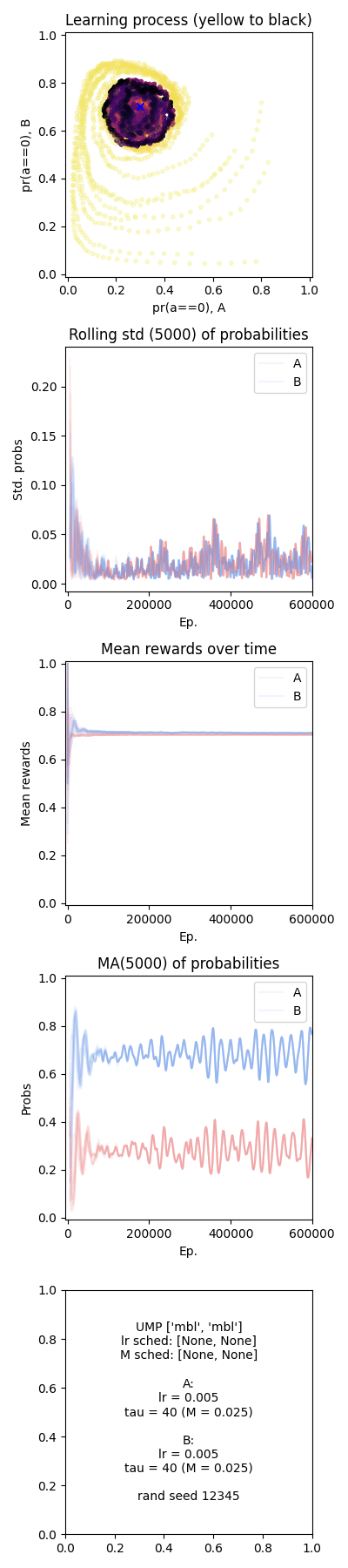}
    	\caption{$\tau = 40$, $M = 40^{-1}$}
    	\end{subfigure}
    \end{center}
    \caption[MBL-LC in self-play on the MP game.]{MBL-LC in self-play on the MP game with different values for $\tau$ ($30$, $35$, $40$) or $M$ ($30^{-1}$, $35^{-1}$, $40^{-1}$) equivalently; $\theta = 5 \cdot 10^{-3}$; for 10 different initialisations.
    (See figure \ref{fig:PD_MBL-DPU_high_mut} for a detailed explanation of the graphs.)
    }
    \label{fig:MP_MBL-LC_low_mut}
\end{figure}
\begin{figure}[ht] %
    \begin{center}
    	\begin{subfigure}[b]{0.3\linewidth}
    	\centering
    	\includegraphics[width=\textwidth, trim={0 770bp 0 0}, clip ]{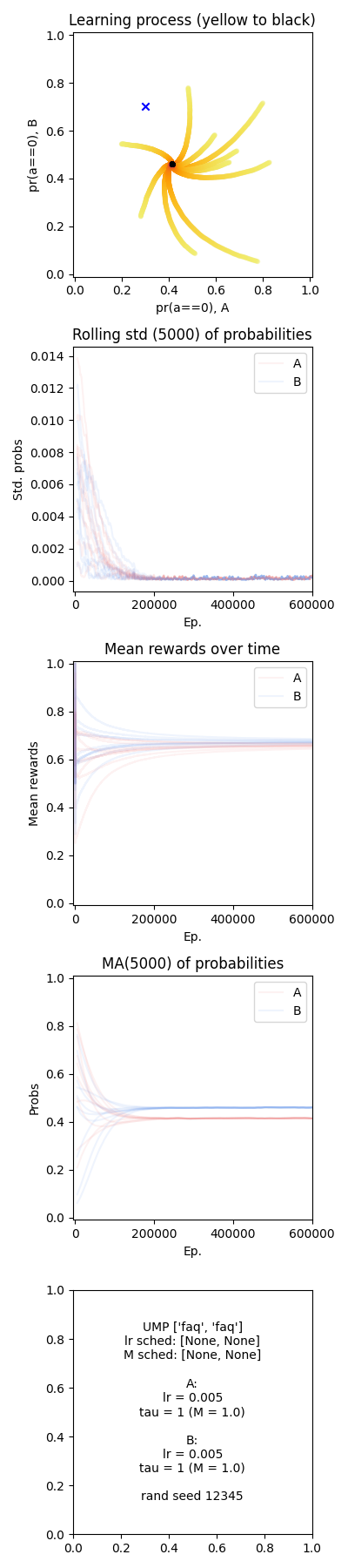}
    	\caption{$\tau = 1$, $M = 1^{-1}$}
    	\end{subfigure}
    	\hfill
    	\begin{subfigure}[b]{0.3\linewidth}
    	\centering
    	\includegraphics[width=\textwidth, trim={0 770bp 0 0}, clip ]{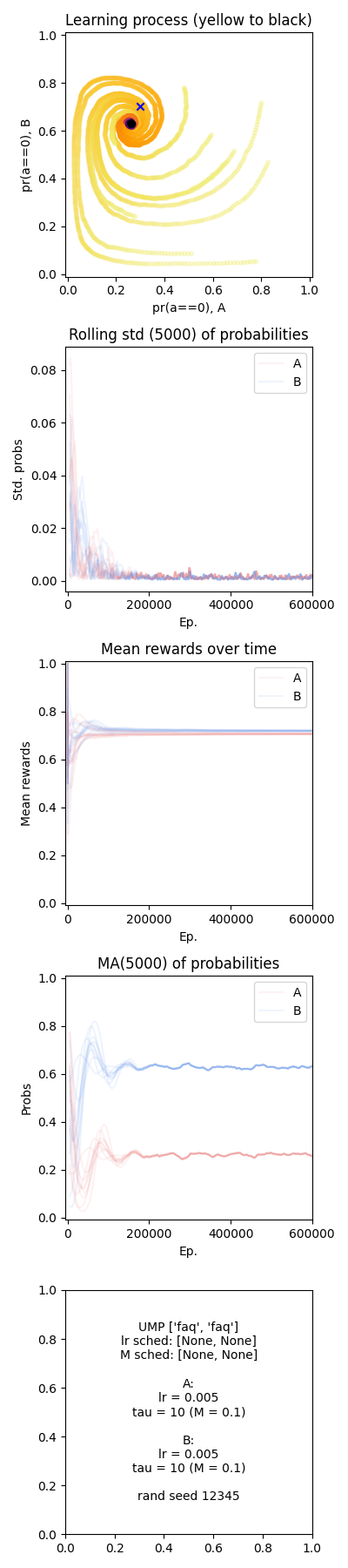}
        \caption{$\tau = 10$, $M = 10^{-1}$}
    	\end{subfigure}
    	\hfill
    	\begin{subfigure}[b]{0.3\linewidth}
    	\centering
    	\includegraphics[width=\textwidth, trim={0 770bp 0 0}, clip ]{figures/results-solon/UMP/faq_faq/fixed_M_fixed_lr/viz/600000_eps_20200420_1621_f232beab088b__combo.png}
    	\caption{$\tau = 20$, $M = 20^{-1}$}
    	\end{subfigure}
    \end{center}
    \caption[FAQ in self-play on the MP game.]{FAQ in self-play on the MP game with different values for $\tau$ ($1$, $10$, $20$) or $M$ ($1$, $10^{-1}$, $20^{-1}$) equivalently; $\theta = 5 \cdot 10^{-3}$; for 10 different initialisations.
    (See figure \ref{fig:PD_MBL-DPU_high_mut} for a detailed explanation of the graphs.)
    }
    \label{fig:MP_FAQ_high_mut}
\end{figure}
\begin{figure}[ht] %
    \begin{center}
    	\begin{subfigure}[b]{0.3\linewidth}
    	\centering
    	\includegraphics[width=\textwidth, trim={0 770bp 0 0}, clip ]{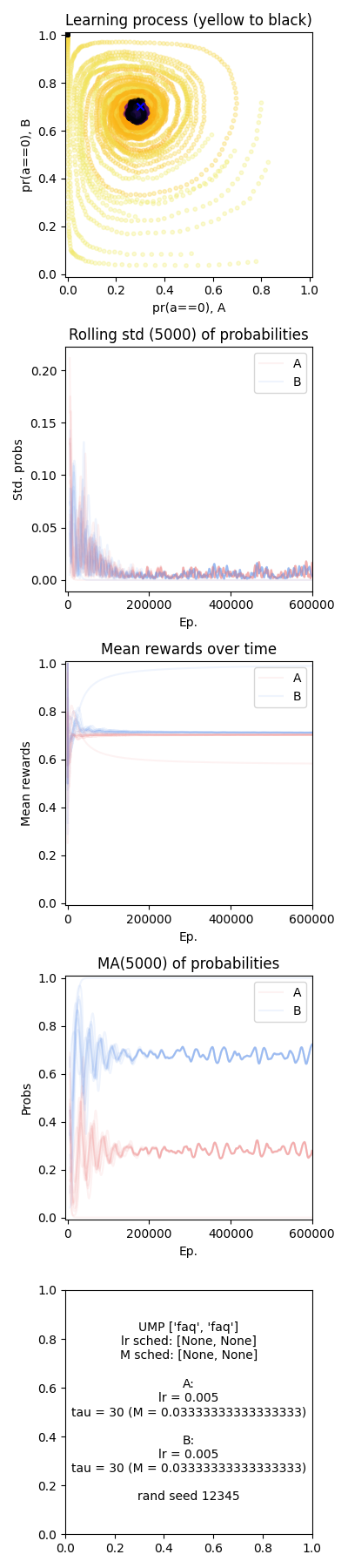}
    	\caption{$\tau = 30$, $M = 30^{-1}$}
    	\end{subfigure}
    	\hfill
    	\begin{subfigure}[b]{0.3\linewidth}
    	\centering
    	\includegraphics[width=\textwidth, trim={0 770bp 0 0}, clip ]{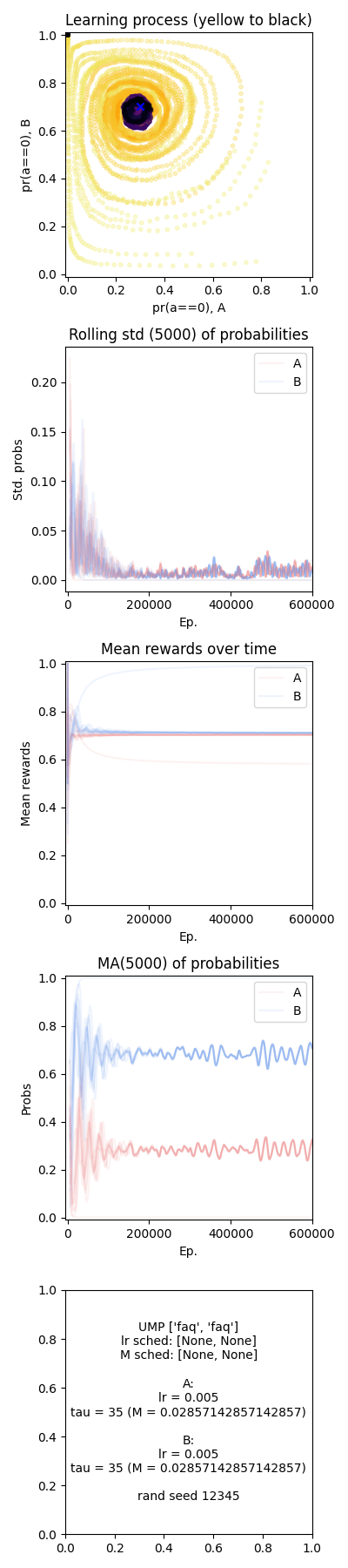}
        \caption{$\tau = 35$, $M = 35^{-1}$}
    	\end{subfigure}
    	\hfill
    	\begin{subfigure}[b]{0.3\linewidth}
    	\centering
    	\includegraphics[width=\textwidth, trim={0 770bp 0 0}, clip ]{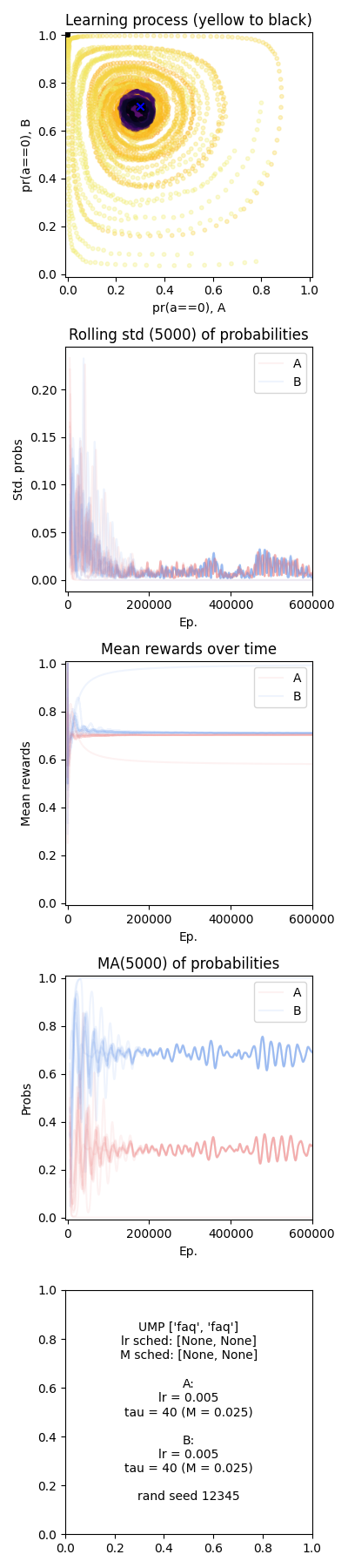}
    	\caption{$\tau = 40$, $M = 40^{-1}$}
    	\end{subfigure}
    \end{center}
    \caption[FAQ in self-play on the MP game.]{FAQ in self-play on the MP game with different values for $\tau$ ($30$, $35$, $40$) or $M$ ($30^{-1}$, $35^{-1}$, $40^{-1}$) equivalently; $\theta = 5 \cdot 10^{-3}$; for 10 different initialisations.
    (See figure \ref{fig:PD_MBL-DPU_high_mut} for a detailed explanation of the graphs.)
    }
    \label{fig:MP_FAQ_low_mut}
\end{figure}
\begin{figure}[ht] %
    \begin{center}
    	\begin{subfigure}[b]{0.3\linewidth}
    	\centering
    	\includegraphics[width=\textwidth, trim={0 770bp 0 0}, clip ]{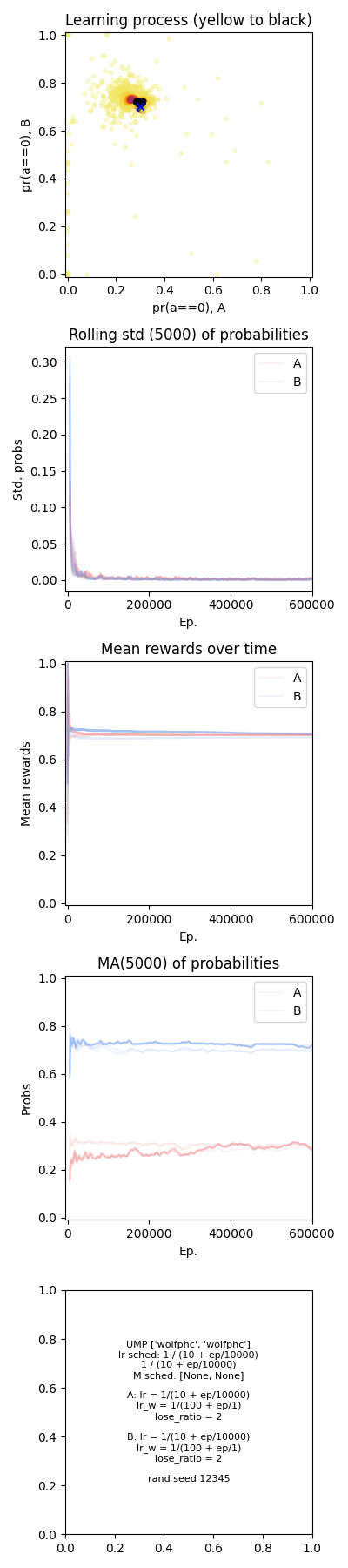}
    	\caption{Initial learning rate $10^{-1}$ for $Q$. Win learning rate $10^{-2}$.}
    	\end{subfigure}
    	\hfill
    	\begin{subfigure}[b]{0.3\linewidth}
    	\centering
    	\includegraphics[width=\textwidth, trim={0 770bp 0 0}, clip ]{figures/results-solon/UMP/wolfphc_wolfphc/fixed_M_reducing_lr/viz/600000_eps_20201109_1502_0a3084819347__combo.png}
        \caption{Initial learning rate $10^{-1}$ for $Q$. Win learning rate $1/2 \cdot 10^{-4}$.}
    	\end{subfigure}
    	\hfill
    	\begin{subfigure}[b]{0.3\linewidth}
    	\centering
    	\includegraphics[width=\textwidth, trim={0 770bp 0 0}, clip ]{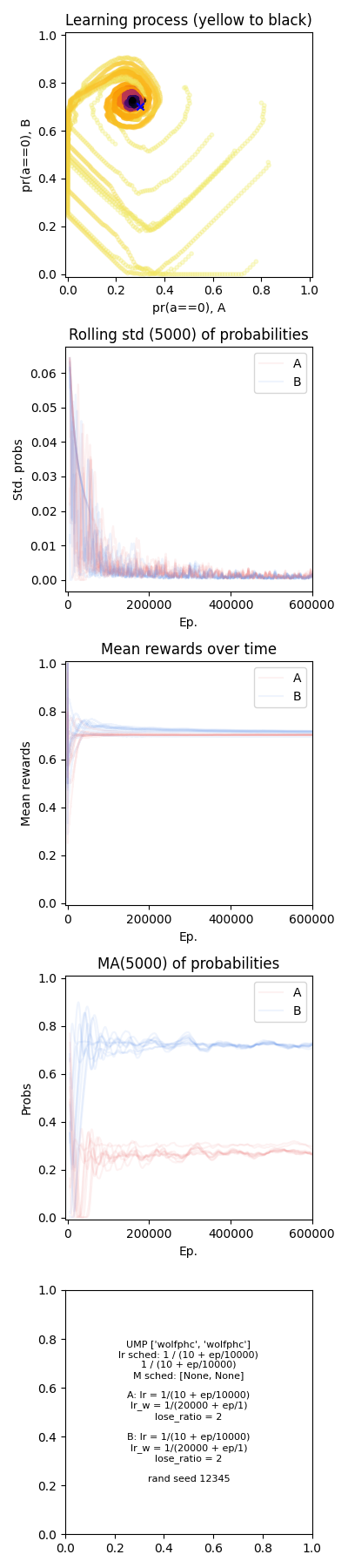}
    	\caption{Initial learning rate $10^{-2}$ for $Q$. Win learning rate $1/2 \cdot 10^{-4}$.}
    	\end{subfigure}
    \end{center}
    \caption[WoLF-PHC in self-play on the MP game.]{WoLF-PHC in self-play on the MP game with different learning schedules; for 10 different initialisations.
    (See figure \ref{fig:PD_MBL-DPU_high_mut} for a detailed explanation of the graphs.)
    }
    \label{app:fig:MP_WoLF-PHC}
\end{figure}

\clearpage

\subsubsection{Zero-sum games with larger action spaces}\label{app:exp_spec_RPS}

The experimental results for the RPS-$n$ games are based on the following payoff structures.

\paragraph{RPS-3.}%
\begin{align*}
R_1 = 
\begin{pmatrix}
 0  & -2  &  3 \\
 2  &  0  & -2 \\
-1  &  2  &  0
\end{pmatrix}
& &
R_2 = -R_1
\end{align*}

Nash equilibrium $x^*$ at:
\begin{align*}
x^*_1 = &
\begin{pmatrix}
2/7 & 11/35 & 2/5
\end{pmatrix}^T
&
x^*_2 = &
\begin{pmatrix}
2/5 & 11/35 & 2/7
\end{pmatrix}^T
\end{align*}

\paragraph{RPS-5.}%
\begin{align*}
R_1 = 
\begin{pmatrix}
 0 &  4 & -2 &  2 & -2 \\
-4 &  0 &  2 & -1 &  1 \\
 2 & -4 &  0 &  4 & -1 \\
-4 &  1 & -4 &  0 &  2 \\
 2 & -1 &  1 & -2 &  0
\end{pmatrix}
& &
R_2 = -R_1
\end{align*}

Nash equilibrium $x^*$ at:
\begin{align*}
x^*_1 = &
\begin{pmatrix}
11/61 & 510/2989 & 8/61 & 50/427 & 1198/2989
\end{pmatrix}^T \\
x^*_2 = &
\begin{pmatrix}
1/7 & 68/427 & 6/49 & 502/2989 & 174/427
\end{pmatrix}^T
\end{align*}

\paragraph{RPS-9.}%
\begin{align*}
R_1 = 
\begin{pmatrix}
 0 &  2 &  1 &  3 &  1 & -1 & -1 & -2 & -1  \\
-1 &  0 &  1 &  3 &  1 &  1 & -1 & -2 & -1  \\
-1 & -2 &  0 &  3 &  1 &  1 &  1 & -2 & -1  \\
-2 & -4 & -2 &  0 &  2 &  2 &  2 &  4 & -2  \\
-1 & -2 & -1 & -3 &  0 &  1 &  1 &  2 &  1  \\
 1 & -2 & -1 & -3 & -1 &  0 &  1 &  2 &  1  \\
 2 &  4 & -2 & -6 & -2 & -2 &  0 &  4 &  2  \\
 1 &  2 &  1 & -3 & -1 & -1 & -1 &  0 &  1  \\
 1 &  2 &  1 &  3 & -1 & -1 & -1 & -2 &  0
\end{pmatrix}
& &
R_2 = -R_1
\end{align*}

Nash equilibrium $x^*$ at:
\begin{align*}
x^*_1 = &
\begin{pmatrix}
1/8 & 1/8 & 1/8 & 1/16 & 1/8 & 1/8 & 1/16 & 1/8 & 1/8
\end{pmatrix}^T \\
x^*_2 = &
\begin{pmatrix}
3/22 & 3/44 & 3/22 & 1/22 & 3/22 & 3/22 & 3/22 & 3/44 & 3/22
\end{pmatrix}^T
\end{align*}

While MP is an informative illustration of the different behaviours, MP reduces to a planar dynamical system, which does not allow many complex behaviours, as exemplified by the Poincaré-Bendixson theorem, e.g., \cite[theorem 7.16]{teschl_ordinary_2012} holding for planar systems.
Hence, higher-dimensional zero-sum games allow a further understanding of the differences between the algorithms and shed light on the effect of larger state spaces while preserving the neutral stability of interior equilibria.
We consider here the Rock-Paper-Scissors game of different sizes (3, 5 and 9 actions).

\paragraph{MBL-DPU and MBL-LC.}
In RPS-3, MBL-DPU (figures \ref{fig:RPS3_MBL-DPU_high_mut}, \ref{fig:RPS3_MBL-DPU_low_mut}) shows a similar behaviour to MP with a marked dependence of the behaviour of the variance on the value of $M$.
In contrast, MBL-LC (figures \ref{fig:RPS3_MBL-LC_high_mut}, \ref{fig:RPS3_MBL-LC_low_mut}) shows a much quicker convergence, with the variance dropping after similar numbers of episodes (around $10^5$) for all values of $M$. As with MBL-DPU, the residual variance increases with weaker mutation. This is in accordance with the neutral stability of the Nash equilibrium, allowing for larger fluctuations.

In RPS-5, both MBL variants (figures \ref{fig:RPS5_MBL-DPU_high_mut}, \ref{fig:RPS5_MBL-DPU_low_mut} for MBL-DPU and figures \ref{fig:RPS5_MBL-LC_high_mut}, \ref{fig:RPS5_MBL-LC_low_mut} for MBL-LC) show behaviours similar to their RPS-3 counterparts.
In RPS-9, MBL-DPU (figures \ref{fig:RPS9_MBL-DPU_high_mut}, \ref{fig:RPS9_MBL-DPU_low_mut}) again shows similar behaviour, with slower convergence compared to its RPS-3 and RPS-5 counterparts.
Interestingly, MBL-LC (figures \ref{fig:RPS9_MBL-LC_high_mut}, \ref{fig:RPS9_MBL-LC_low_mut}) seems to have two distinct regions to which trajectories evolve, suggesting a potentially stronger sensitivity to the choice of $\theta$.

\paragraph{FAQ-learning.}
Like for MP, we see a quicker convergence for FAQ in RPS-3 (figures \ref{fig:RPS3_FAQ_high_mut}, \ref{fig:RPS3_FAQ_low_mut}) compared to the MBL variants, but with trajectories similar to those of MBL-LC when considering low values of $M$, in which case the replicator dynamics makes a stronger contribution to the trajectories.
Similar to MBL-LC, but already in RPS-5, FAQ shows two distinct regions to which trajectories evolve when perturbation is weak (figures \ref{fig:RPS5_FAQ_high_mut}, \ref{fig:RPS5_FAQ_low_mut}), whereas the former does not show such a split for RPS-5.
In RPS-9, FAQ shows such a split for stronger perturbation levels already and shows even three distinct such regions for weaker perturbation (figures \ref{fig:RPS9_FAQ_high_mut}, \ref{fig:RPS9_FAQ_low_mut}).

\paragraph{WoLF-PHC.}
For WoLF-PHC, we see a still quicker convergence in RPS-3 (figure \ref{app:fig:RPS3_WoLF-PHC}) than for the other algorithms, similar to the MP case. However, the behaviour is much less clear in RPS-5 (figure \ref{app:fig:RPS5_WoLF-PHC}). Here, trajectories do not consistently approach a specific region. It is possible that the reduction schedules for the learning rates, which force each trajectory to converge, lead to trajectories stalling prematurely.
This becomes even more pronounced in RPS-9 (figure \ref{fig:app_RPS9_WoLF-PHC}), where WoLF-PHC seems to initially move away from the Nash equilibrium and to get stuck along the boundaries of $\mathcal{D}$.

\begin{figure}[hbt] %
    \vspace{5\baselineskip}
	\hfill
   	\begin{subfigure}[b]{0.3\linewidth}
   	\centering
   	\includegraphics[width=\textwidth, trim={0 770bp 0 0}, clip ]{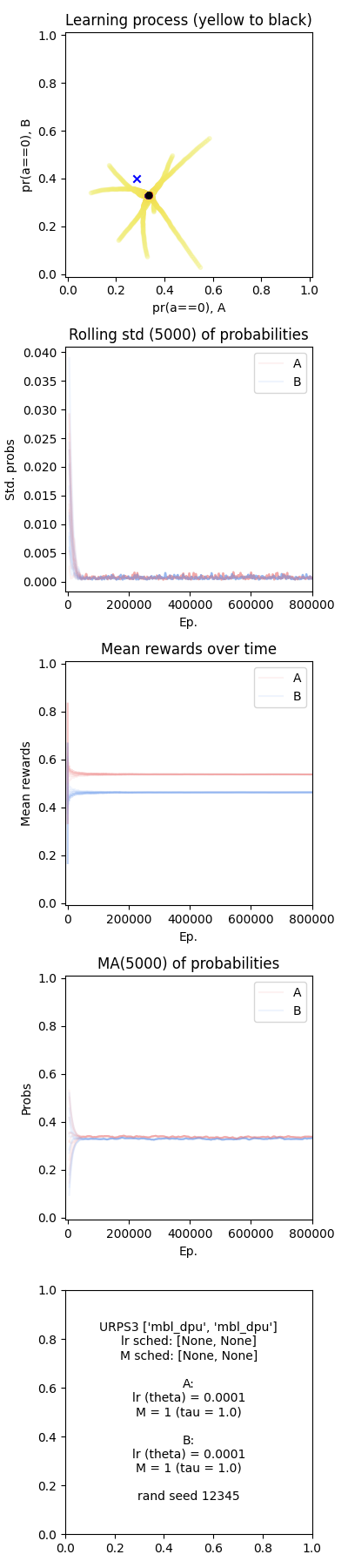}
   	\caption{$\tau = 1$, $M = 1^{-1}$}
   	\end{subfigure}
   	\hfill
   	\begin{subfigure}[b]{0.3\linewidth}
   	\centering
   	\includegraphics[width=\textwidth, trim={0 770bp 0 0}, clip ]{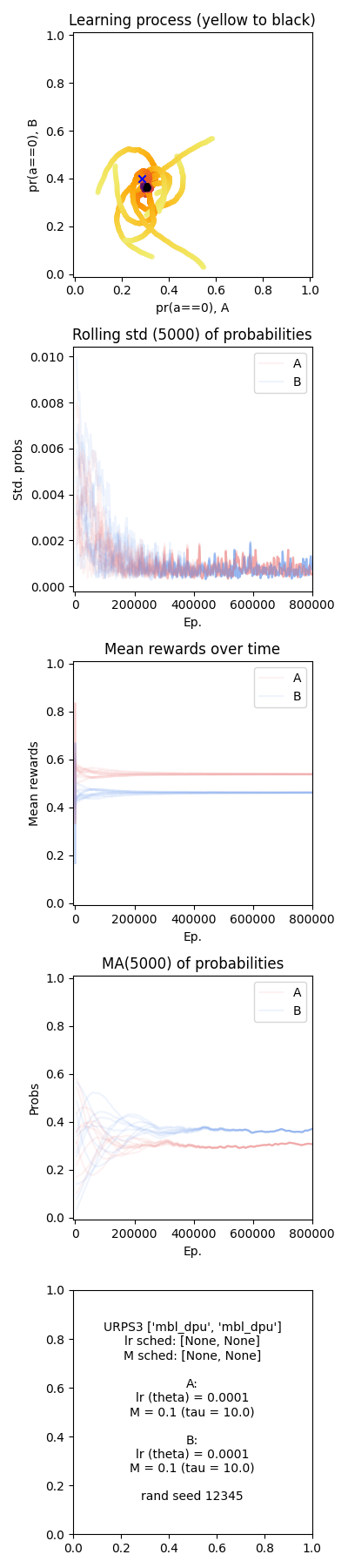}
   	\caption{$\tau = 10$, $M = 10^{-1}$}
   	\end{subfigure}
   	\hfill
   	\begin{subfigure}[b]{0.3\linewidth}
   	\centering
   	\includegraphics[width=\textwidth, trim={0 770bp 0 0}, clip ]{figures/results-solon/URPS3/mbl_dpu_mbl_dpu/fixed_M_fixed_lr/viz/800000_eps_20200420_1746_bfdbb7dcc89f__combo.png}
    \caption{$\tau = 20$, $M = 20^{-1}$}
   	\end{subfigure}
    \hfill
    \caption[MBL-DPU in self-play on the RPS-3 game.]{MBL-DPU in self-play on the RPS-3 game with different values for $\tau$ ($1$, $10$, $20$) or $M$ ($1$, $10^{-1}$, $20^{-1}$) equivalently; $\theta = 10^{-4}$; for 10 different initialisations.
    (See figure \ref{fig:PD_MBL-DPU_high_mut} for a detailed explanation of the graphs.)
    }
    \label{fig:RPS3_MBL-DPU_high_mut}
\end{figure}
\begin{figure}[h] %
    \hfill
   	\begin{subfigure}[b]{0.3\linewidth}
   	\centering
   	\includegraphics[width=\textwidth, trim={0 770bp 0 0}, clip ]{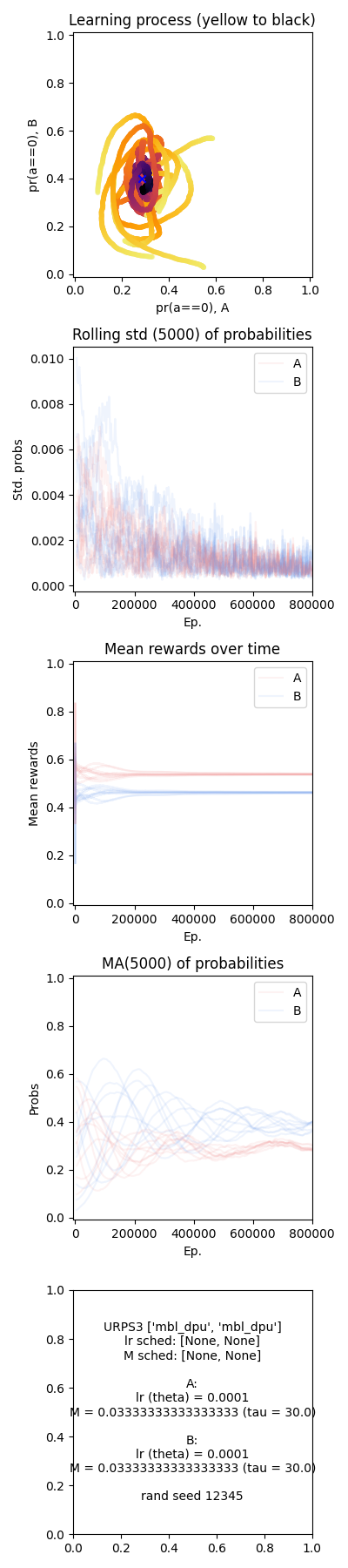}
    \caption{$\tau = 30$, $M = 30^{-1}$}
   	\end{subfigure}
   	\hfill
   	\begin{subfigure}[b]{0.3\linewidth}
   	\centering
   	\includegraphics[width=\textwidth, trim={0 770bp 0 0}, clip ]{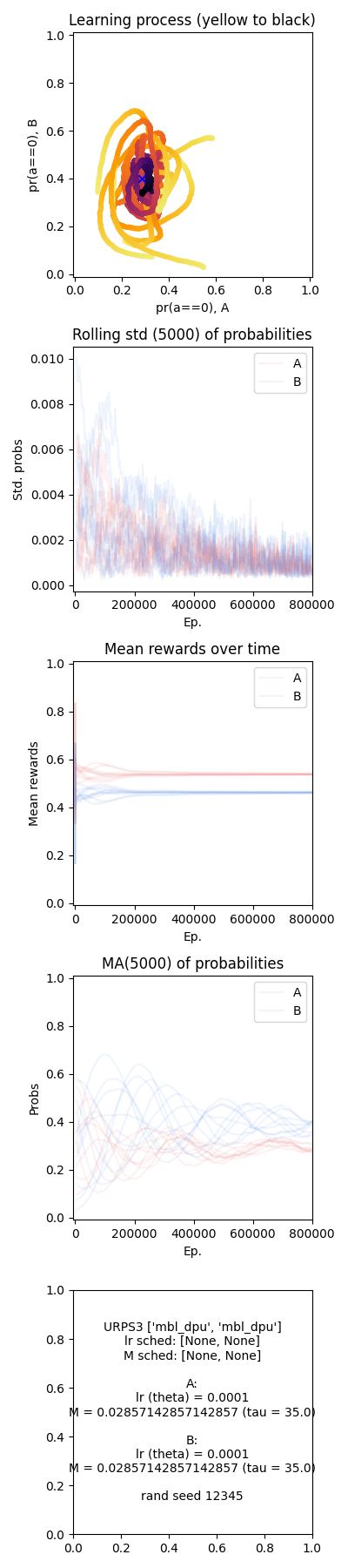}
   	\caption{$\tau = 35$, $M = 35^{-1}$}
   	\end{subfigure}
   	\hfill
   	\begin{subfigure}[b]{0.3\linewidth}
   	\centering
   	\includegraphics[width=\textwidth, trim={0 770bp 0 0}, clip ]{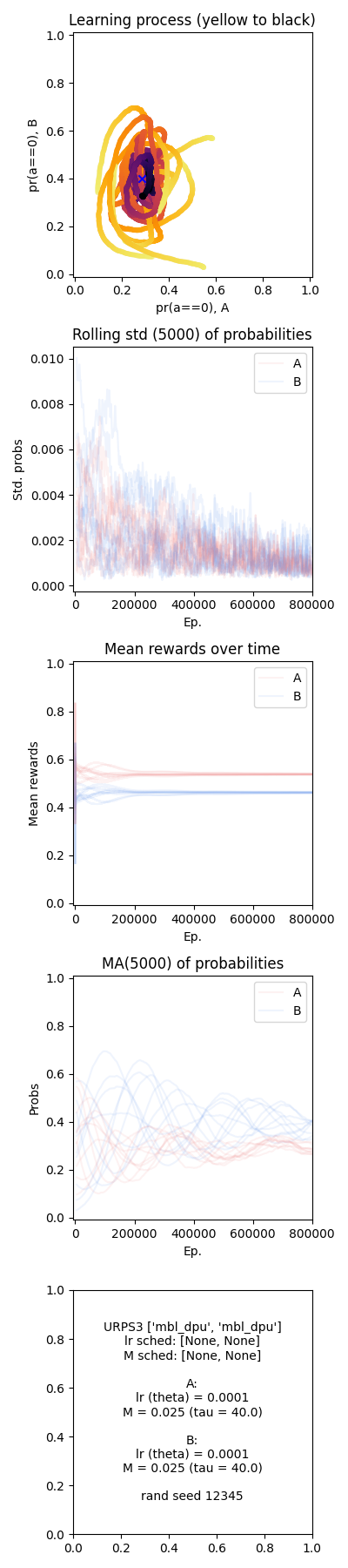}
   	\caption{$\tau = 40$, $M = 40^{-1}$}
   	\end{subfigure}
    \hfill
    \caption[MBL-DPU in self-play on the RPS-3 game.]{MBL-DPU in self-play on the RPS-3 game with different values for $\tau$ ($30$, $35$, $40$) or $M$ ($30^{-1}$, $35^{-1}$, $40^{-1}$) equivalently; $\theta = 10^{-4}$; for 10 different initialisations.
    (See figure \ref{fig:PD_MBL-DPU_high_mut} for a detailed explanation of the graphs.)
    }
    \label{fig:RPS3_MBL-DPU_low_mut}    
\end{figure}
\begin{figure}[h] %
    \begin{center}
    	\begin{subfigure}[b]{0.3\linewidth}
    	\centering
    	\includegraphics[width=\textwidth, trim={0 770bp 0 0}, clip ]{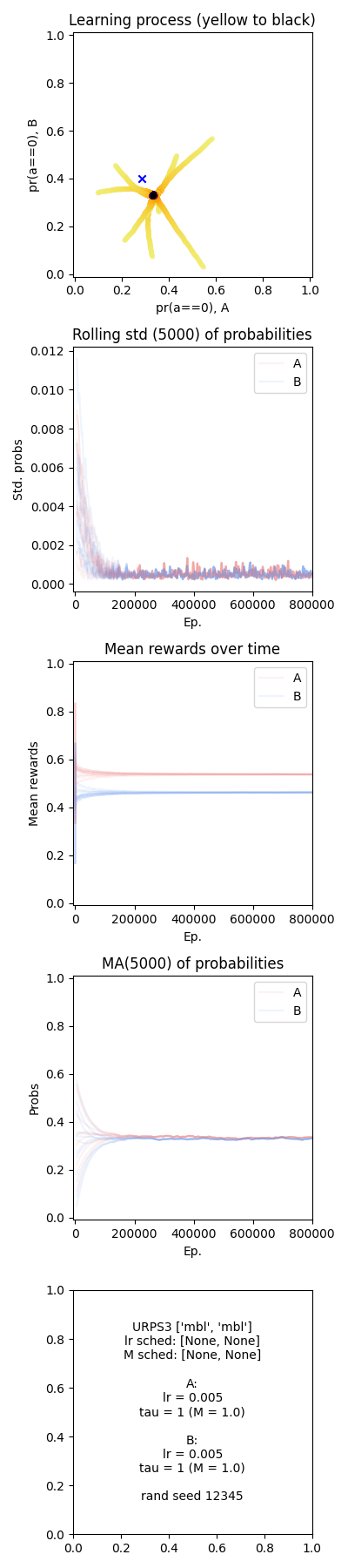}
        \caption{$\tau = 1$, $M = 1^{-1}$}
    	\end{subfigure}
    	\hfill
    	\begin{subfigure}[b]{0.3\linewidth}
    	\centering
    	\includegraphics[width=\textwidth, trim={0 770bp 0 0}, clip ]{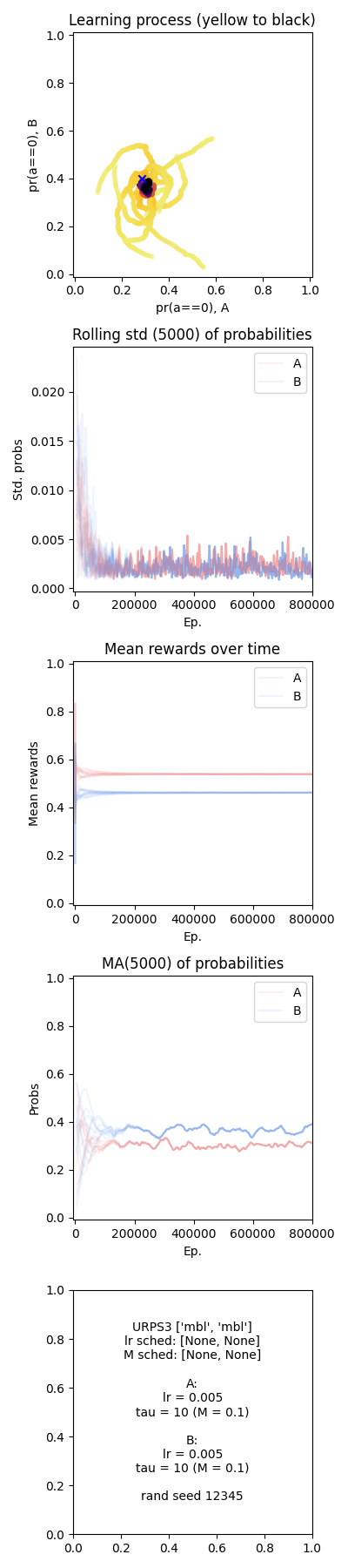}
    	\caption{$\tau = 10$, $M = 10^{-1}$}
    	\end{subfigure}
    	\hfill
    	\begin{subfigure}[b]{0.3\linewidth}
    	\centering
    	\includegraphics[width=\textwidth, trim={0 770bp 0 0}, clip ]{figures/results-solon/URPS3/mbl_mbl/fixed_M_fixed_lr/viz/800000_eps_20200420_1755_b8ffa87dcbe3__combo.png}
    	\caption{$\tau = 20$, $M = 20^{-1}$}
    	\end{subfigure}
    \end{center}
    \caption[MBL-LC in self-play on the RPS-3 game.]{MBL-LC in self-play on the RPS-3 game with different values for $\tau$ ($1$, $10$, $20$) or $M$ ($1$, $10^{-1}$, $20^{-1}$) equivalently; $\theta = 5 \cdot 10^{-3}$; for 10 different initialisations.
    (See figure \ref{fig:PD_MBL-DPU_high_mut} for a detailed explanation of the graphs.)
    }
    \label{fig:RPS3_MBL-LC_high_mut}
\end{figure}
\begin{figure}[h] %
    \begin{center}
    	\begin{subfigure}[b]{0.3\linewidth}
    	\centering
    	\includegraphics[width=\textwidth, trim={0 770bp 0 0}, clip ]{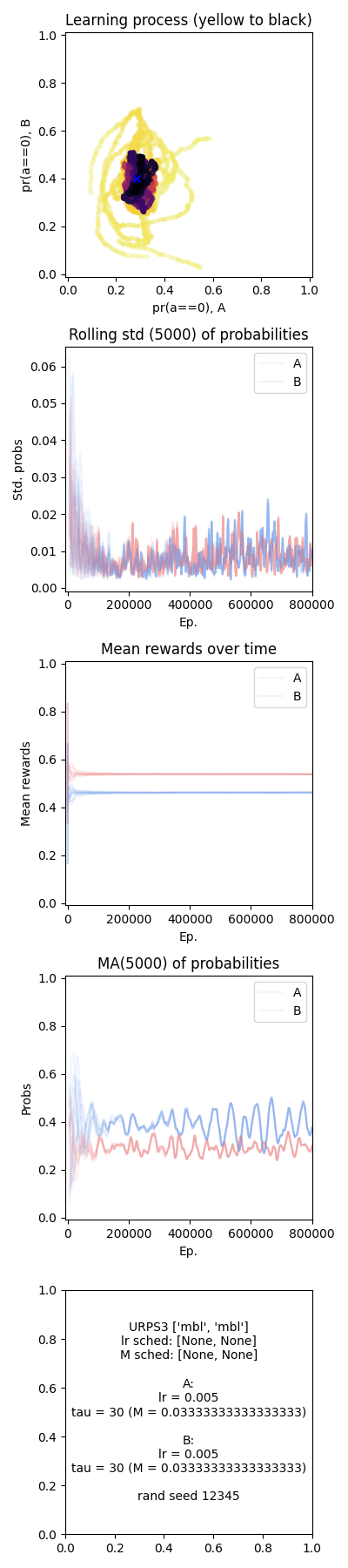}
    	\caption{$\tau = 30$, $M = 30^{-1}$}
    	\end{subfigure}
    	\hfill
    	\begin{subfigure}[b]{0.3\linewidth}
    	\centering
    	\includegraphics[width=\textwidth, trim={0 770bp 0 0}, clip ]{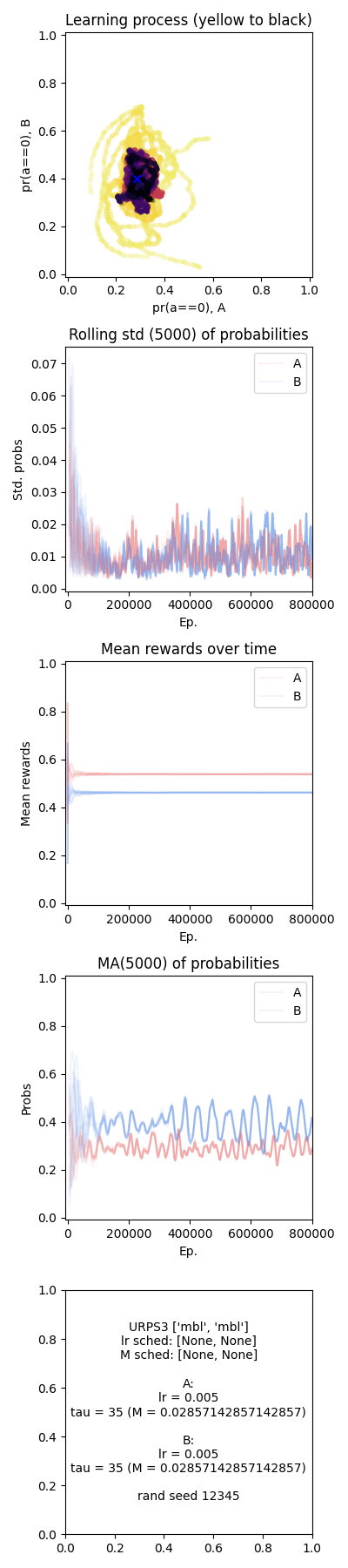}
        \caption{$\tau = 35$, $M = 35^{-1}$}
    	\end{subfigure}
    	\hfill
    	\begin{subfigure}[b]{0.3\linewidth}
    	\centering
    	\includegraphics[width=\textwidth, trim={0 770bp 0 0}, clip ]{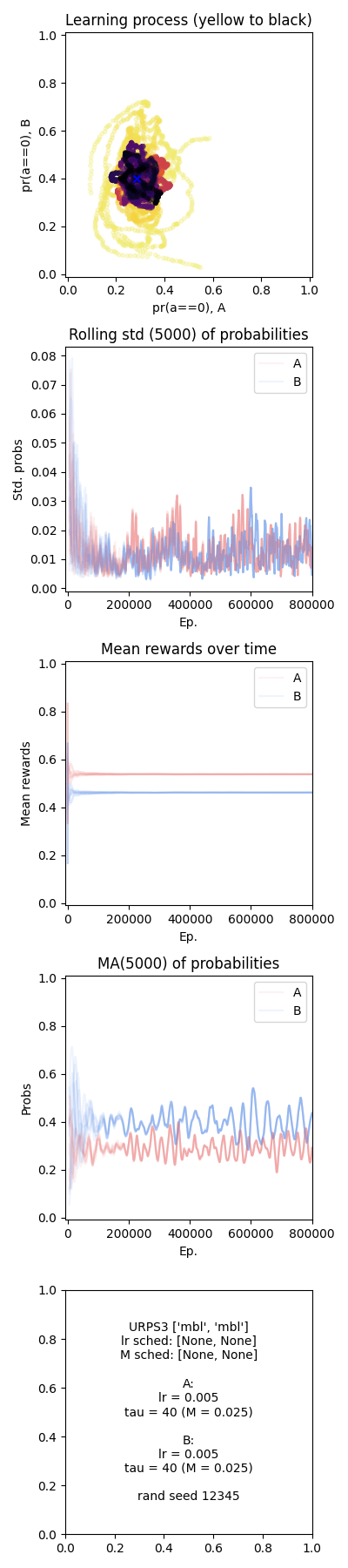}
    	\caption{$\tau = 40$, $M = 40^{-1}$}
    	\end{subfigure}
    \end{center}
    \caption[MBL-LC in self-play on the RPS-3 game.]{MBL-LC in self-play on the RPS-3 game with different values for $\tau$ ($30$, $35$, $40$) or $M$ ($30^{-1}$, $35^{-1}$, $40^{-1}$) equivalently; $\theta = 5 \cdot 10^{-3}$; for 10 different initialisations.
    (See figure \ref{fig:PD_MBL-DPU_high_mut} for a detailed explanation of the graphs.)
    }
    \label{fig:RPS3_MBL-LC_low_mut}
\end{figure}
\begin{figure}[h] %
    \begin{center}
    	\begin{subfigure}[b]{0.3\linewidth}
    	\centering
    	\includegraphics[width=\textwidth, trim={0 770bp 0 0}, clip ]{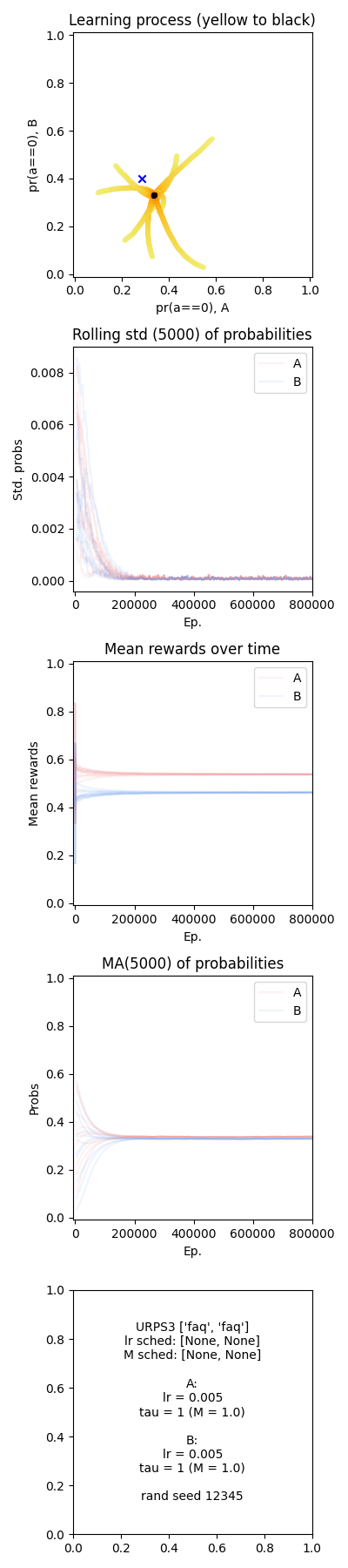}
    	\caption{$\tau = 1$, $M = 1^{-1}$}
    	\end{subfigure}
    	\hfill
    	\begin{subfigure}[b]{0.3\linewidth}
    	\centering
    	\includegraphics[width=\textwidth, trim={0 770bp 0 0}, clip ]{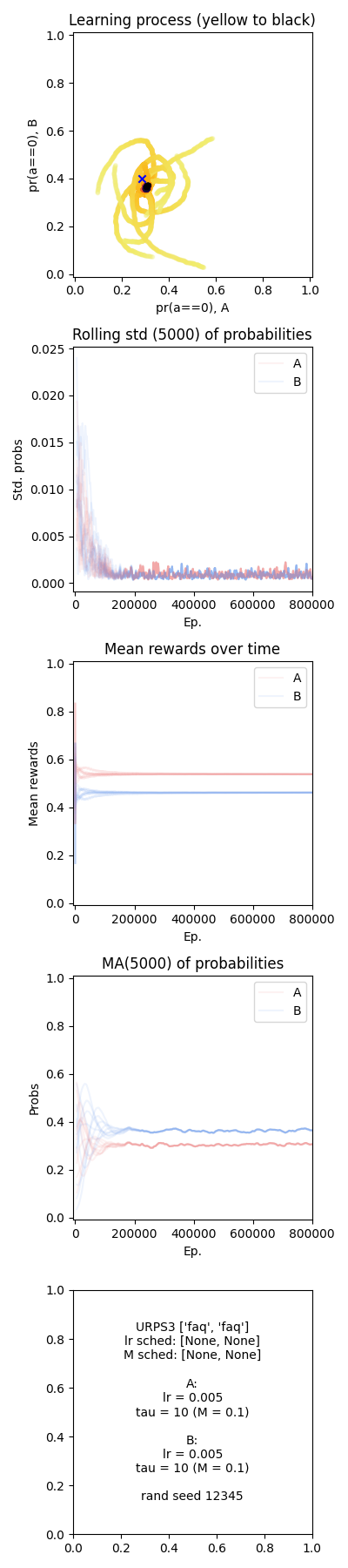}
        \caption{$\tau = 10$, $M = 10^{-1}$}
    	\end{subfigure}
    	\hfill
    	\begin{subfigure}[b]{0.3\linewidth}
    	\centering
    	\includegraphics[width=\textwidth, trim={0 770bp 0 0}, clip ]{figures/results-solon/URPS3/faq_faq/fixed_M_fixed_lr/viz/800000_eps_20200420_1830_beda32fb5550__combo.png}
    	\caption{$\tau = 20$, $M = 20^{-1}$}
    	\end{subfigure}
    \end{center}
    \caption[FAQ in self-play on the RPS-3 game.]{FAQ in self-play on the RPS-3 game with different values for $\tau$ ($1$, $10$, $20$) or $M$ ($1$, $10^{-1}$, $20^{-1}$) equivalently; $\theta = 5 \cdot 10^{-3}$; for 10 different initialisations.
    (See figure \ref{fig:PD_MBL-DPU_high_mut} for a detailed explanation of the graphs.)
    }
    \label{fig:RPS3_FAQ_high_mut}
\end{figure}
\begin{figure}[h] %
    \begin{center}
    	\begin{subfigure}[b]{0.3\linewidth}
    	\centering
    	\includegraphics[width=\textwidth, trim={0 770bp 0 0}, clip ]{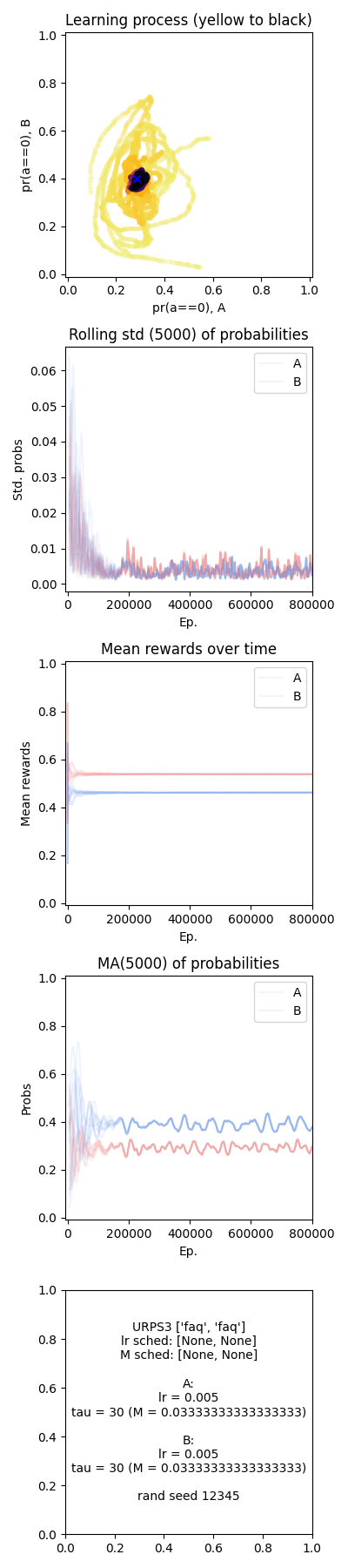}
    	\caption{$\tau = 30$, $M = 30^{-1}$}
    	\end{subfigure}
    	\hfill
    	\begin{subfigure}[b]{0.3\linewidth}
    	\centering
    	\includegraphics[width=\textwidth, trim={0 770bp 0 0}, clip ]{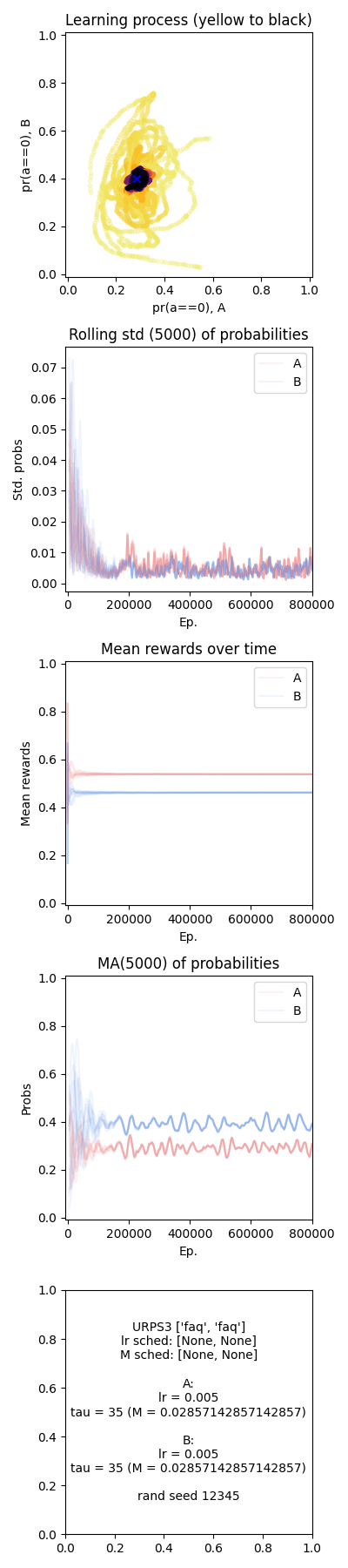}
        \caption{$\tau = 35$, $M = 35^{-1}$}
    	\end{subfigure}
    	\hfill
    	\begin{subfigure}[b]{0.3\linewidth}
    	\centering
    	\includegraphics[width=\textwidth, trim={0 770bp 0 0}, clip ]{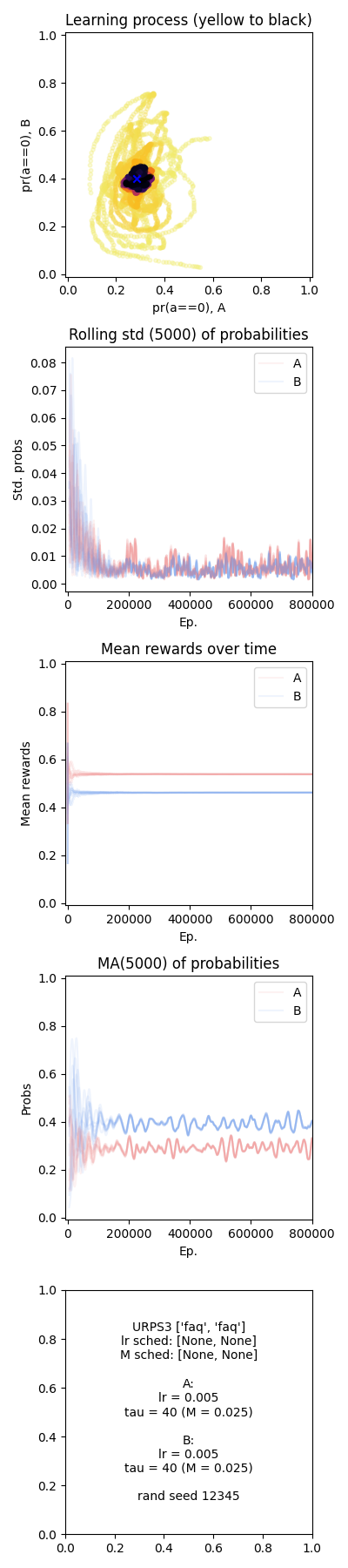}
    	\caption{$\tau = 40$, $M = 40^{-1}$}
    	\end{subfigure}
    \end{center}
    \caption[FAQ in self-play on the RPS-3 game.]{FAQ in self-play on the RPS-3 game with different values for $\tau$ ($30$, $35$, $40$) or $M$ ($30^{-1}$, $35^{-1}$, $40^{-1}$) equivalently; $\theta = 5 \cdot 10^{-3}$; for 10 different initialisations.
    (See figure \ref{fig:PD_MBL-DPU_high_mut} for a detailed explanation of the graphs.)
    }
    \label{fig:RPS3_FAQ_low_mut}
\end{figure}
\begin{figure}[h] %
    \begin{center}
    	\begin{subfigure}[b]{0.3\linewidth}
    	\centering
    	\includegraphics[width=\textwidth, trim={0 770bp 0 0}, clip ]{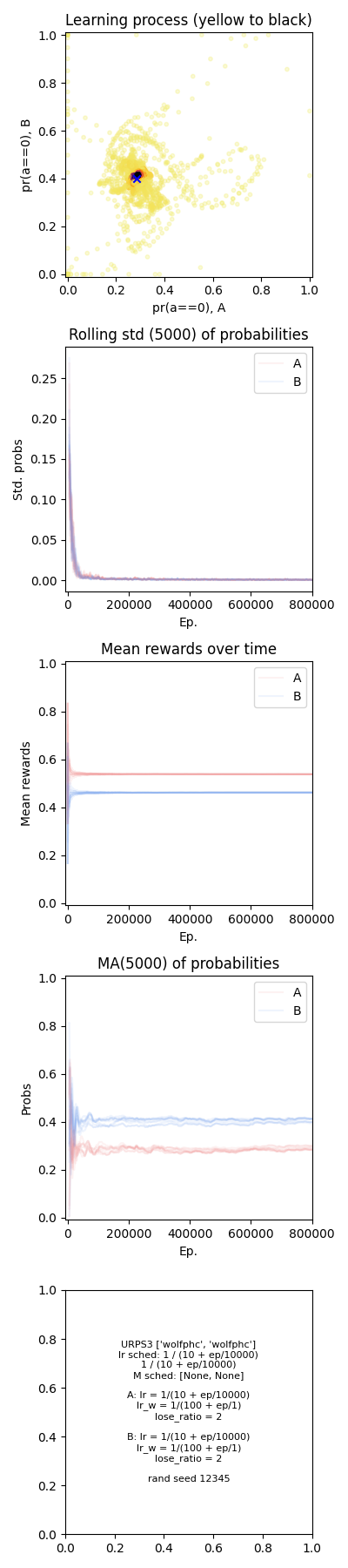}
    	\caption{Initial learning rate $10^{-1}$ for $Q$. Win learning rate $10^{-2}$.}
    	\end{subfigure}
    	\hfill
    	\begin{subfigure}[b]{0.3\linewidth}
    	\centering
    	\includegraphics[width=\textwidth, trim={0 770bp 0 0}, clip ]{figures/results-solon/URPS3/wolfphc_wolfphc/fixed_M_reducing_lr/viz/800000_eps_20201110_1849_9a085005141e__combo.png}
        \caption{Initial learning rate $10^{-1}$ for $Q$. Win learning rate $1/2 \cdot 10^{-4}$.}
    	\end{subfigure}
    	\hfill
    	\begin{subfigure}[b]{0.3\linewidth}
    	\centering
    	\includegraphics[width=\textwidth, trim={0 770bp 0 0}, clip ]{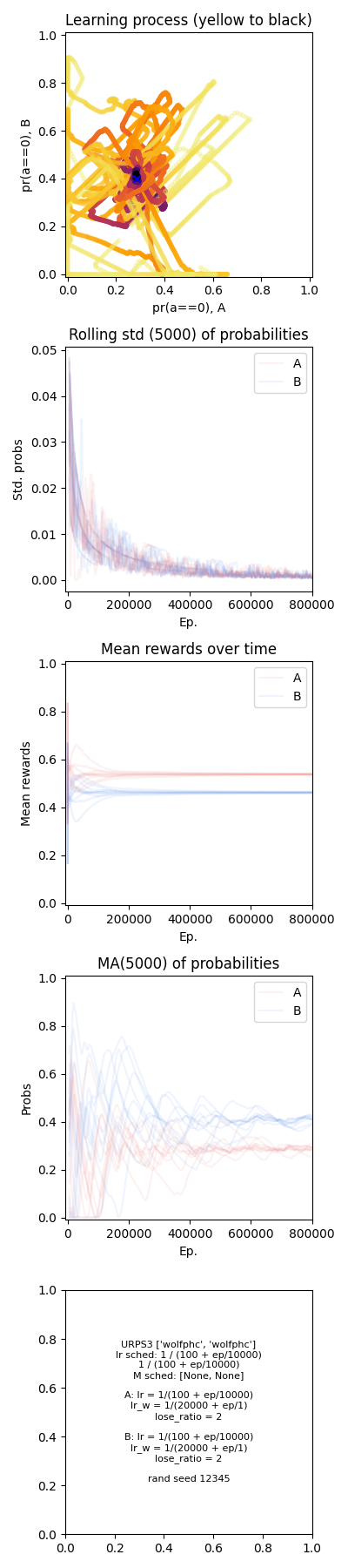}
    	\caption{Initial learning rate $10^{-2}$ for $Q$. Win learning rate $1/2 \cdot 10^{-4}$.}
    	\end{subfigure}
    \end{center}
    \caption[WoLF-PHC in self-play on the RPS-3 game.]{WoLF-PHC in self-play on the RPS-3 game with different learning schedules; for 10 different initialisations.
    (See figure \ref{fig:PD_MBL-DPU_high_mut} for a detailed explanation of the graphs.)
    }
    \label{app:fig:RPS3_WoLF-PHC}
\end{figure}
\begin{figure}[h] %
    \begin{center}
    	\begin{subfigure}[b]{0.3\linewidth}
    	\centering
    	\includegraphics[width=\textwidth, trim={0 770bp 0 0}, clip ]{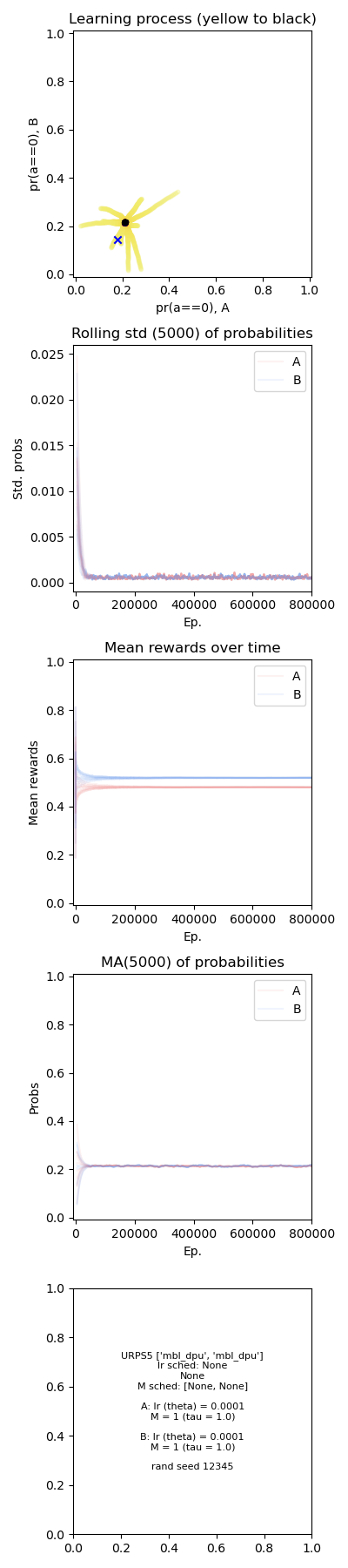}
    	\caption{$\tau = 1$, $M = 1^{-1}$}
    	\end{subfigure}
    	\hfill
    	\begin{subfigure}[b]{0.3\linewidth}
    	\centering
    	\includegraphics[width=\textwidth, trim={0 770bp 0 0}, clip ]{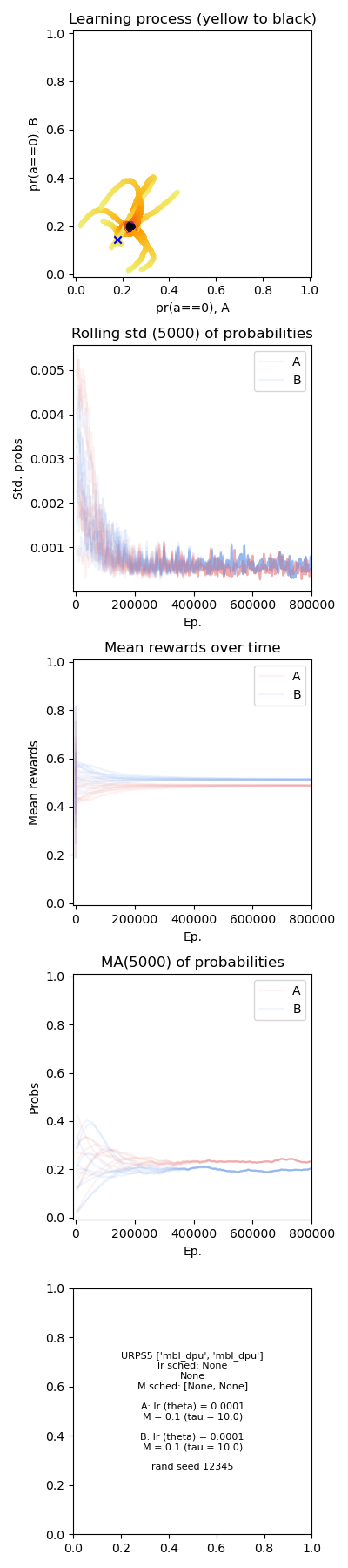}
    	\caption{$\tau = 10$, $M = 10^{-1}$}
    	\end{subfigure}
    	\hfill
    	\begin{subfigure}[b]{0.3\linewidth}
    	\centering
    	\includegraphics[width=\textwidth, trim={0 770bp 0 0}, clip ]{figures/results-solon/URPS5/mbl_dpu_mbl_dpu/fixed_M_fixed_lr/viz/800000_eps_20201111_0952_d788550ec584__combo.png}
        \caption{$\tau = 20$, $M = 20^{-1}$}
    	\end{subfigure}
    \end{center}
    \caption[MBL-DPU in self-play on the RPS-5 game.]{MBL-DPU in self-play on the RPS-5 game with different values for $\tau$ ($1$, $10$, $20$) or $M$ ($1$, $10^{-1}$, $20^{-1}$) equivalently; $\theta = 10^{-4}$; for 10 different initialisations.
    (See figure \ref{fig:PD_MBL-DPU_high_mut} for a detailed explanation of the graphs.)
    }
    \label{fig:RPS5_MBL-DPU_high_mut}
\end{figure}
\begin{figure}[h] %
    \begin{center}
    	\begin{subfigure}[b]{0.3\linewidth}
    	\centering
    	\includegraphics[width=\textwidth, trim={0 770bp 0 0}, clip ]{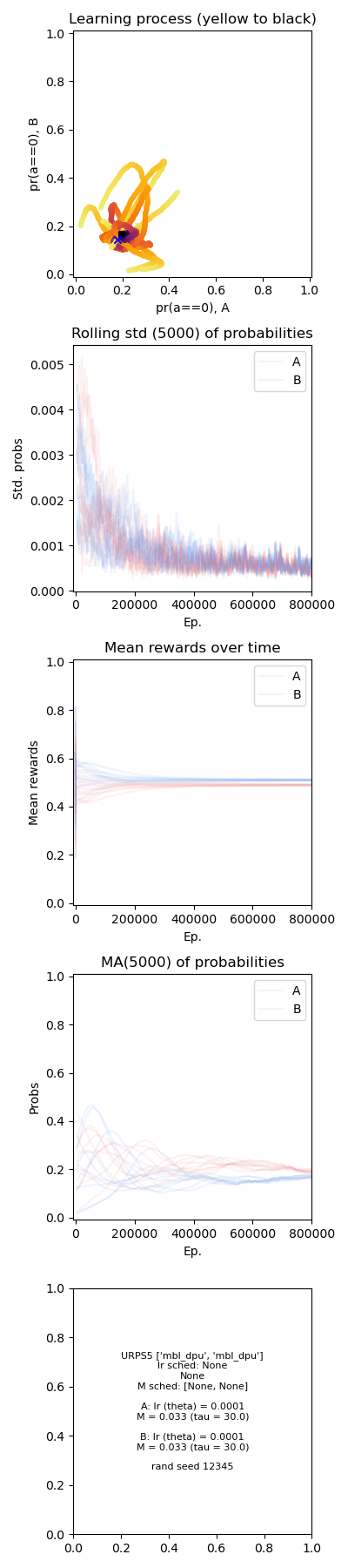}
        \caption{$\tau = 30$, $M = 30^{-1}$}
    	\end{subfigure}
    	\hfill
    	\begin{subfigure}[b]{0.3\linewidth}
    	\centering
    	\includegraphics[width=\textwidth, trim={0 770bp 0 0}, clip ]{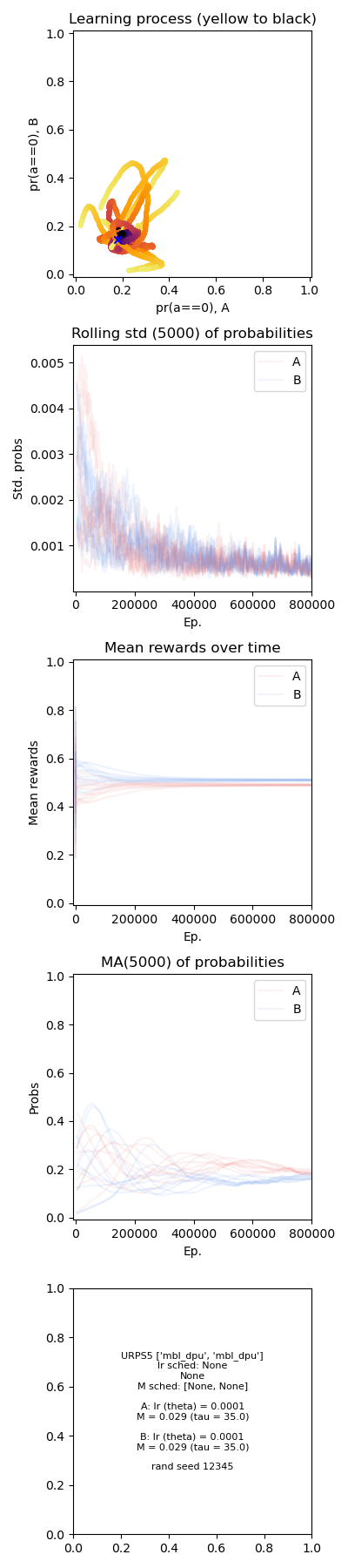}
    	\caption{$\tau = 35$, $M = 35^{-1}$}
    	\end{subfigure}
    	\hfill
    	\begin{subfigure}[b]{0.3\linewidth}
    	\centering
    	\includegraphics[width=\textwidth, trim={0 770bp 0 0}, clip ]{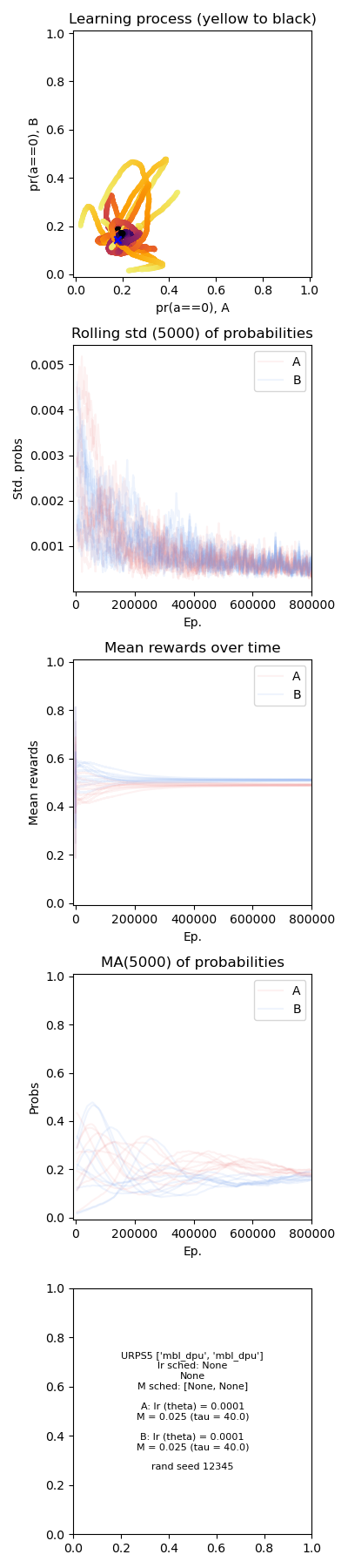}
    	\caption{$\tau = 40$, $M = 40^{-1}$}
    	\end{subfigure}
    \end{center}
    \caption[MBL-DPU in self-play on the RPS-5 game.]{MBL-DPU in self-play on the RPS-5 game with different values for $\tau$ ($30$, $35$, $40$) or $M$ ($30^{-1}$, $35^{-1}$, $40^{-1}$) equivalently; $\theta = 10^{-4}$; for 10 different initialisations.
    (See figure \ref{fig:PD_MBL-DPU_high_mut} for a detailed explanation of the graphs.)
    }
    \label{fig:RPS5_MBL-DPU_low_mut}
\end{figure}
\begin{figure}[h] %
    \begin{center}
    	\begin{subfigure}[b]{0.3\linewidth}
    	\centering
    	\includegraphics[width=\textwidth, trim={0 770bp 0 0}, clip ]{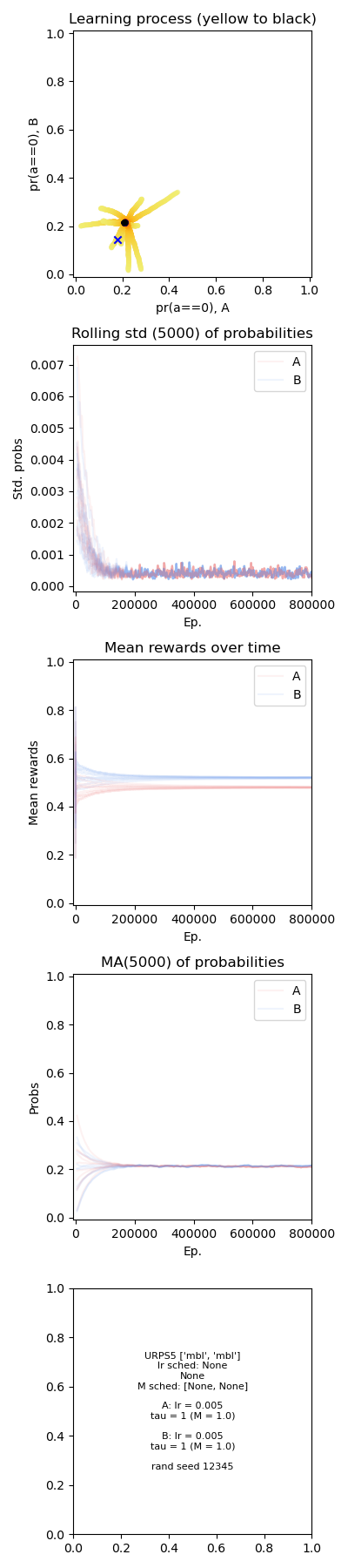}
        \caption{$\tau = 1$, $M = 1^{-1}$}
    	\end{subfigure}
    	\hfill
    	\begin{subfigure}[b]{0.3\linewidth}
    	\centering
    	\includegraphics[width=\textwidth, trim={0 770bp 0 0}, clip ]{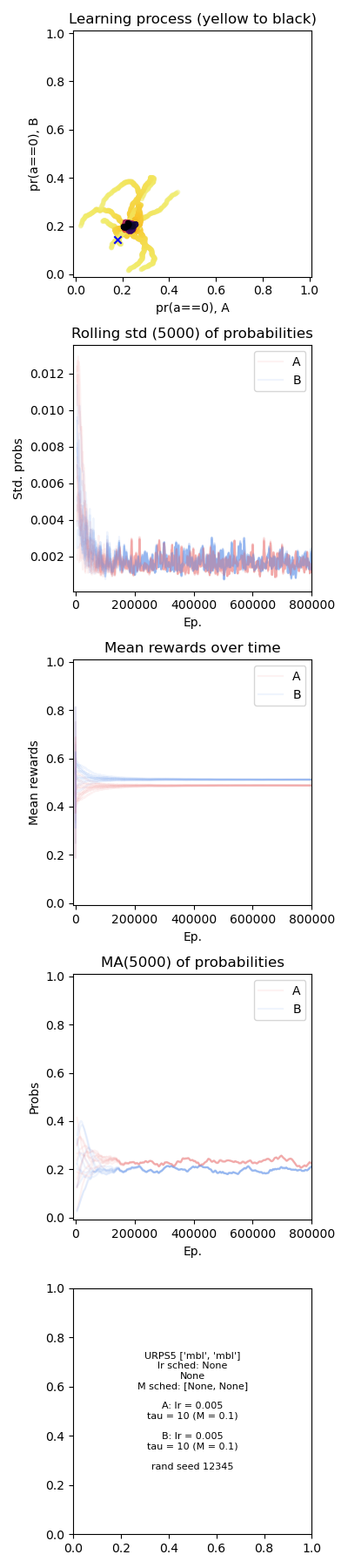}
    	\caption{$\tau = 10$, $M = 10^{-1}$}
    	\end{subfigure}
    	\hfill
    	\begin{subfigure}[b]{0.3\linewidth}
    	\centering
    	\includegraphics[width=\textwidth, trim={0 770bp 0 0}, clip ]{figures/results-solon/URPS5/mbl_mbl/fixed_M_fixed_lr/viz/800000_eps_20201111_1001_94e6d5a73c69__combo.png}
    	\caption{$\tau = 20$, $M = 20^{-1}$}
    	\end{subfigure}
    \end{center}
    \caption[MBL-LC in self-play on the RPS-5 game.]{MBL-LC in self-play on the RPS-5 game with different values for $\tau$ ($1$, $10$, $20$) or $M$ ($1$, $10^{-1}$, $20^{-1}$) equivalently; $\theta = 5 \cdot 10^{-3}$; for 10 different initialisations.
    (See figure \ref{fig:PD_MBL-DPU_high_mut} for a detailed explanation of the graphs.)
    }
    \label{fig:RPS5_MBL-LC_high_mut}
\end{figure}
\begin{figure}[h] %
    \begin{center}
    	\begin{subfigure}[b]{0.3\linewidth}
    	\centering
    	\includegraphics[width=\textwidth, trim={0 770bp 0 0}, clip ]{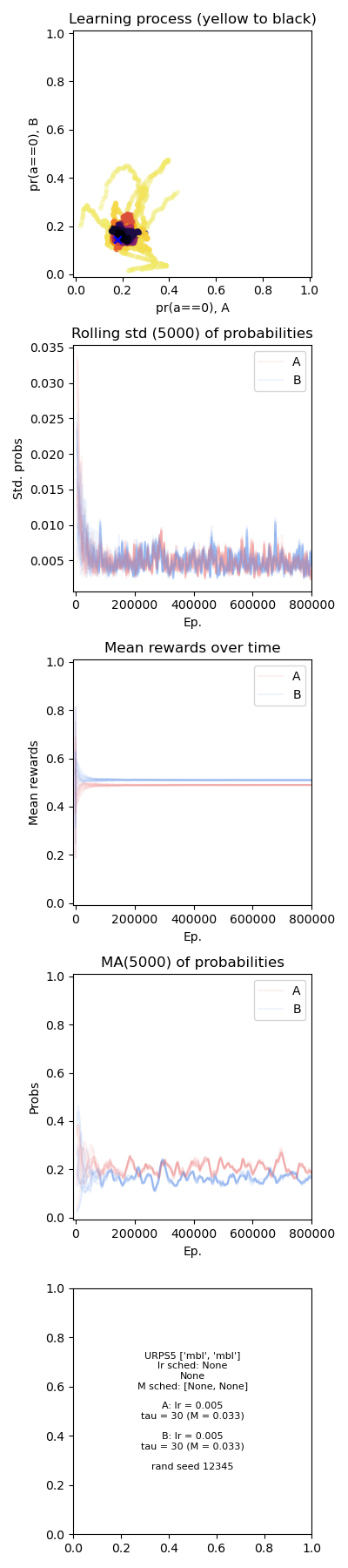}
    	\caption{$\tau = 30$, $M = 30^{-1}$}
    	\end{subfigure}
    	\hfill
    	\begin{subfigure}[b]{0.3\linewidth}
    	\centering
    	\includegraphics[width=\textwidth, trim={0 770bp 0 0}, clip ]{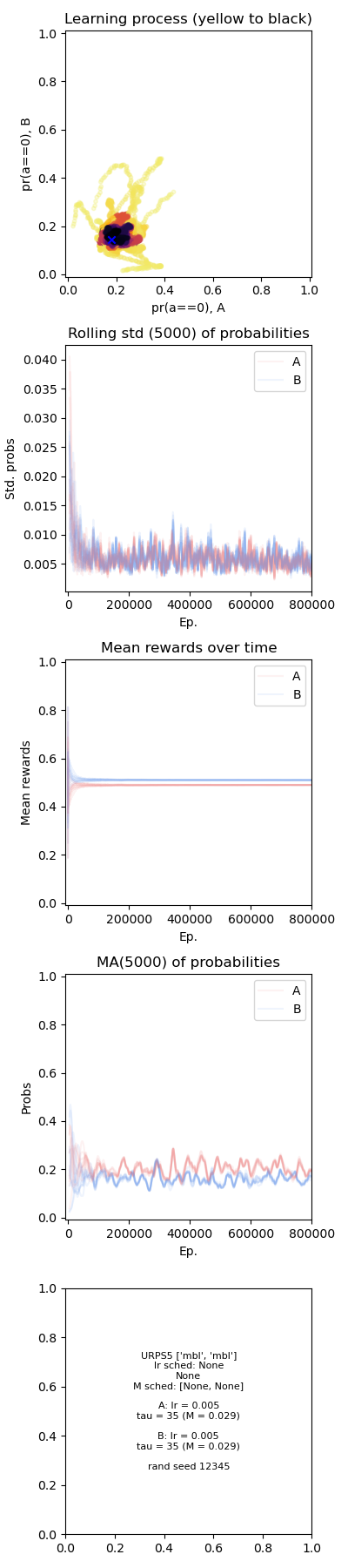}
        \caption{$\tau = 35$, $M = 35^{-1}$}
    	\end{subfigure}
    	\hfill
    	\begin{subfigure}[b]{0.3\linewidth}
    	\centering
    	\includegraphics[width=\textwidth, trim={0 770bp 0 0}, clip ]{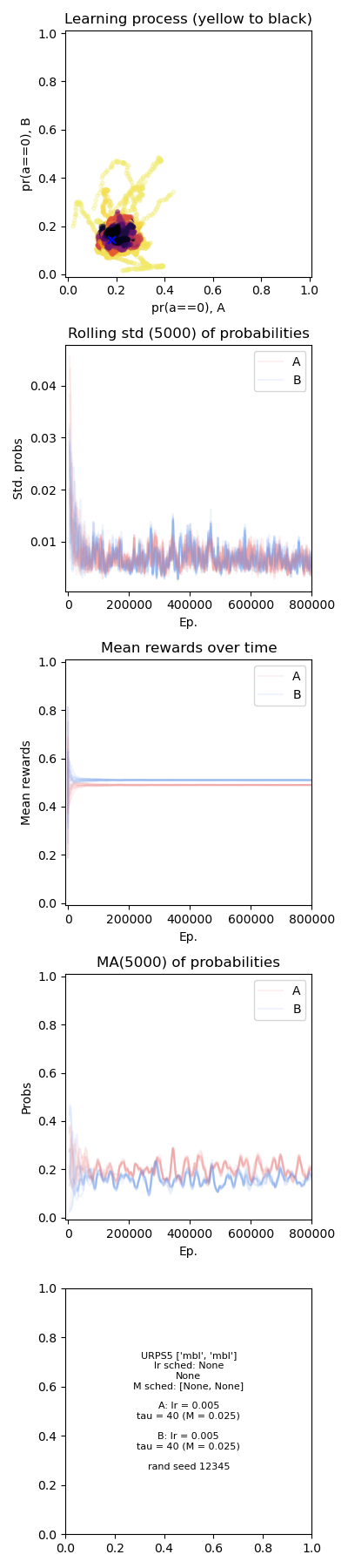}
    	\caption{$\tau = 40$, $M = 40^{-1}$}
    	\end{subfigure}
    \end{center}
    \caption[MBL-LC in self-play on the RPS-5 game.]{MBL-LC in self-play on the RPS-5 game with different values for $\tau$ ($30$, $35$, $40$) or $M$ ($30^{-1}$, $35^{-1}$, $40^{-1}$) equivalently; $\theta = 5 \cdot 10^{-3}$; for 10 different initialisations.
    (See figure \ref{fig:PD_MBL-DPU_high_mut} for a detailed explanation of the graphs.)
    }
    \label{fig:RPS5_MBL-LC_low_mut}
\end{figure}
\begin{figure}[h] %
    \begin{center}
    	\begin{subfigure}[b]{0.3\linewidth}
    	\centering
    	\includegraphics[width=\textwidth, trim={0 770bp 0 0}, clip ]{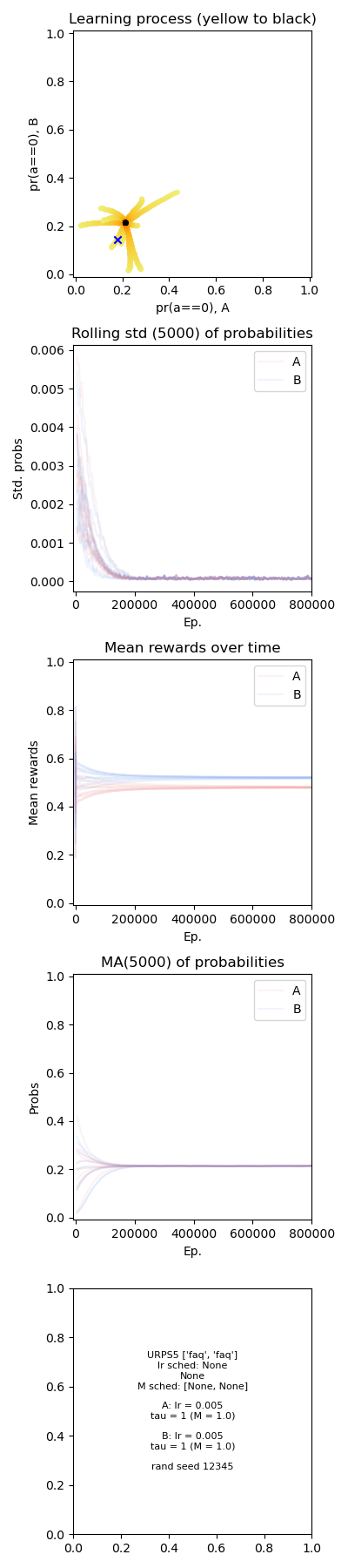}
    	\caption{$\tau = 1$, $M = 1^{-1}$}
    	\end{subfigure}
    	\hfill
    	\begin{subfigure}[b]{0.3\linewidth}
    	\centering
    	\includegraphics[width=\textwidth, trim={0 770bp 0 0}, clip ]{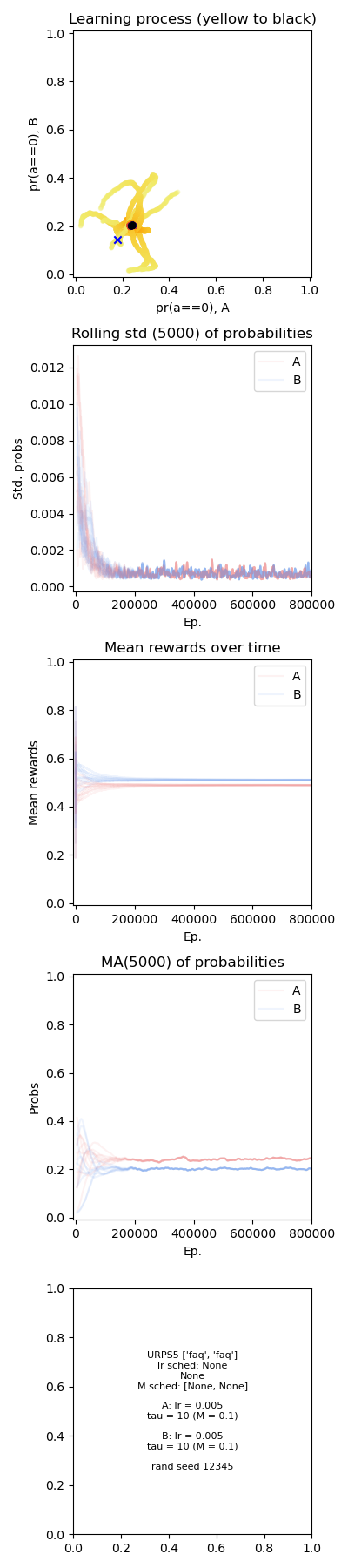}
        \caption{$\tau = 10$, $M = 10^{-1}$}
    	\end{subfigure}
    	\hfill
    	\begin{subfigure}[b]{0.3\linewidth}
    	\centering
    	\includegraphics[width=\textwidth, trim={0 770bp 0 0}, clip ]{figures/results-solon/URPS5/faq_faq/fixed_M_fixed_lr/viz/800000_eps_20201111_1244_38f8b175cdf4__combo.png}
    	\caption{$\tau = 20$, $M = 20^{-1}$}
    	\end{subfigure}
    \end{center}
    \caption[FAQ in self-play on the RPS-5 game.]{FAQ in self-play on the RPS-5 game with different values for $\tau$ ($1$, $10$, $20$) or $M$ ($1$, $10^{-1}$, $20^{-1}$) equivalently; $\theta = 5 \cdot 10^{-3}$; for 10 different initialisations.
    (See figure \ref{fig:PD_MBL-DPU_high_mut} for a detailed explanation of the graphs.)
    }
    \label{fig:RPS5_FAQ_high_mut}
\end{figure}
\begin{figure}[h] %
    \begin{center}
    	\begin{subfigure}[b]{0.3\linewidth}
    	\centering
    	\includegraphics[width=\textwidth, trim={0 770bp 0 0}, clip ]{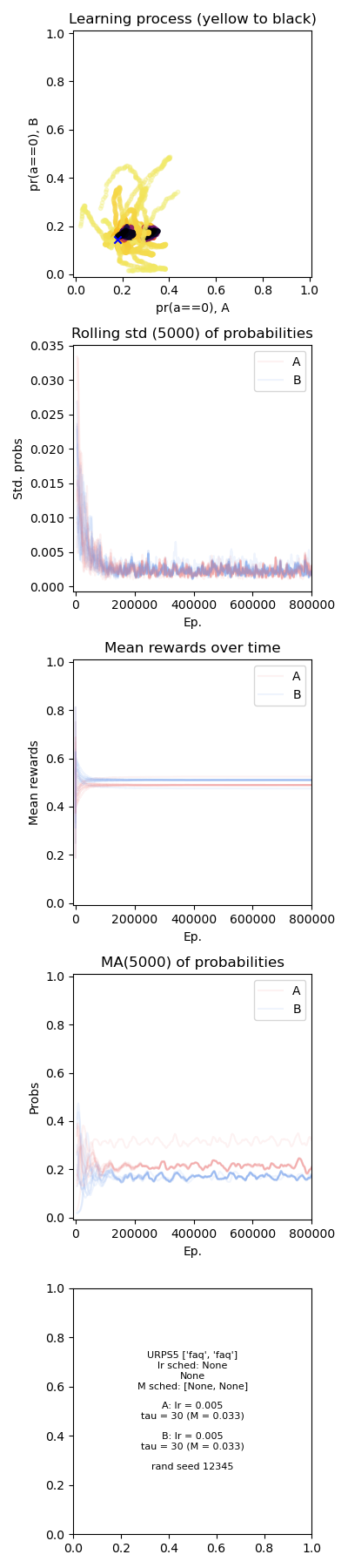}
    	\caption{$\tau = 30$, $M = 30^{-1}$}
    	\end{subfigure}
    	\hfill
    	\begin{subfigure}[b]{0.3\linewidth}
    	\centering
    	\includegraphics[width=\textwidth, trim={0 770bp 0 0}, clip ]{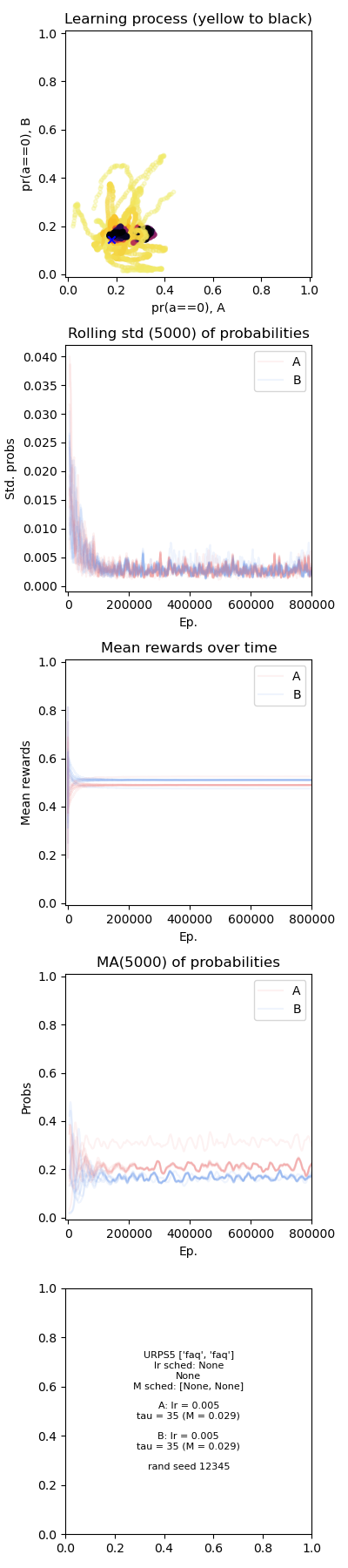}
        \caption{$\tau = 35$, $M = 35^{-1}$}
    	\end{subfigure}
    	\hfill
    	\begin{subfigure}[b]{0.3\linewidth}
    	\centering
    	\includegraphics[width=\textwidth, trim={0 770bp 0 0}, clip ]{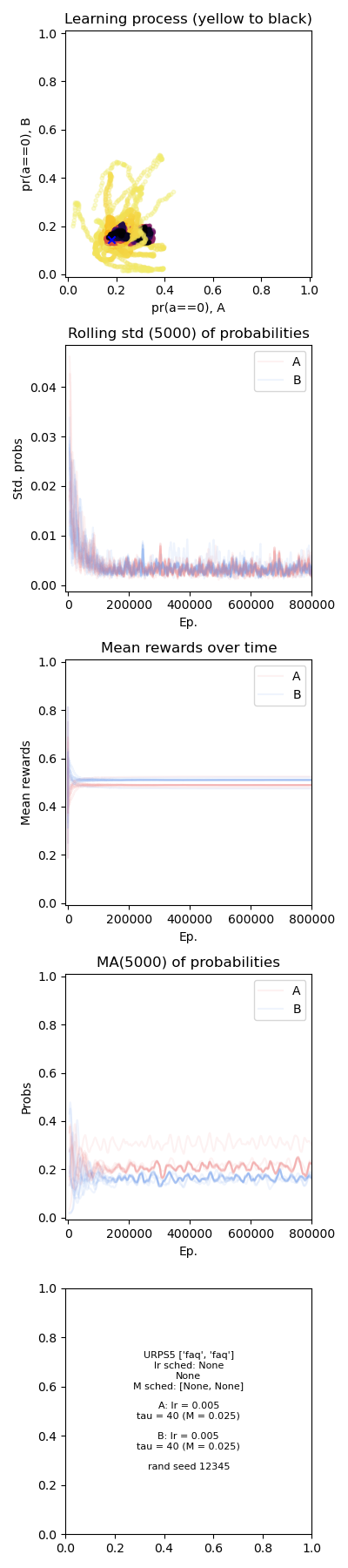}
    	\caption{$\tau = 40$, $M = 40^{-1}$}
    	\end{subfigure}
    \end{center}
    \caption[FAQ in self-play on the RPS-5 game.]{FAQ in self-play on the RPS-5 game with different values for $\tau$ ($30$, $35$, $40$) or $M$ ($30^{-1}$, $35^{-1}$, $40^{-1}$) equivalently; $\theta = 5 \cdot 10^{-3}$; for 10 different initialisations.
    (See figure \ref{fig:PD_MBL-DPU_high_mut} for a detailed explanation of the graphs.)
    }
    \label{fig:RPS5_FAQ_low_mut}
\end{figure}
\begin{figure}[h] %
    \begin{center}
    	\begin{subfigure}[b]{0.3\linewidth}
    	\centering
    	\includegraphics[width=\textwidth, trim={0 770bp 0 0}, clip ]{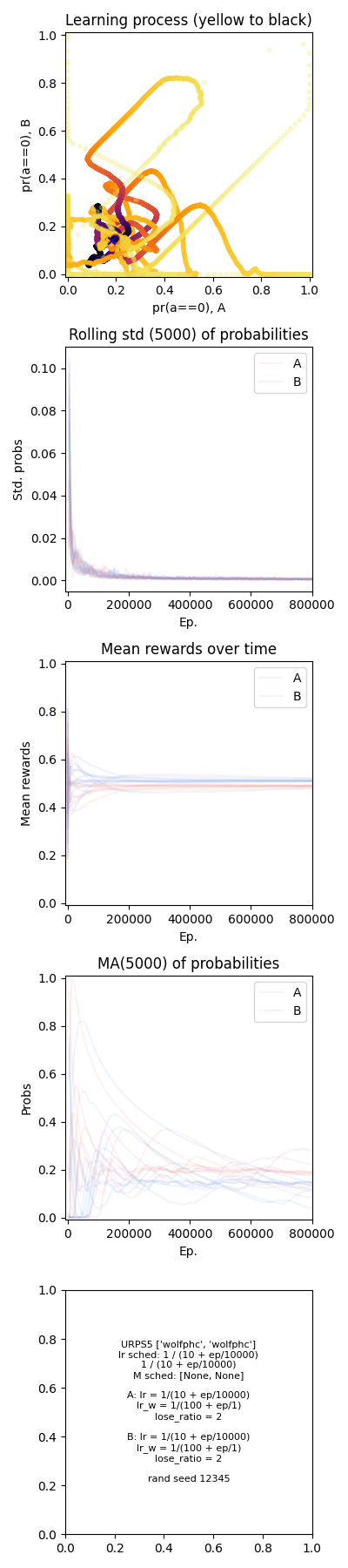}
    	\caption{Initial learning rate $10^{-1}$ for $Q$. Win learning rate $10^{-2}$.}
    	\end{subfigure}
    	\hfill
    	\begin{subfigure}[b]{0.3\linewidth}
    	\centering
    	\includegraphics[width=\textwidth, trim={0 770bp 0 0}, clip ]{figures/results-solon/URPS5/wolfphc_wolfphc/fixed_M_reducing_lr/viz/800000_eps_20201111_1258_4e99262d6dc5__combo.png}
        \caption{Initial learning rate $10^{-1}$ for $Q$. Win learning rate $1/2 \cdot 10^{-4}$.}
    	\end{subfigure}
    	\hfill
    	\begin{subfigure}[b]{0.3\linewidth}
    	\centering
    	\includegraphics[width=\textwidth, trim={0 770bp 0 0}, clip ]{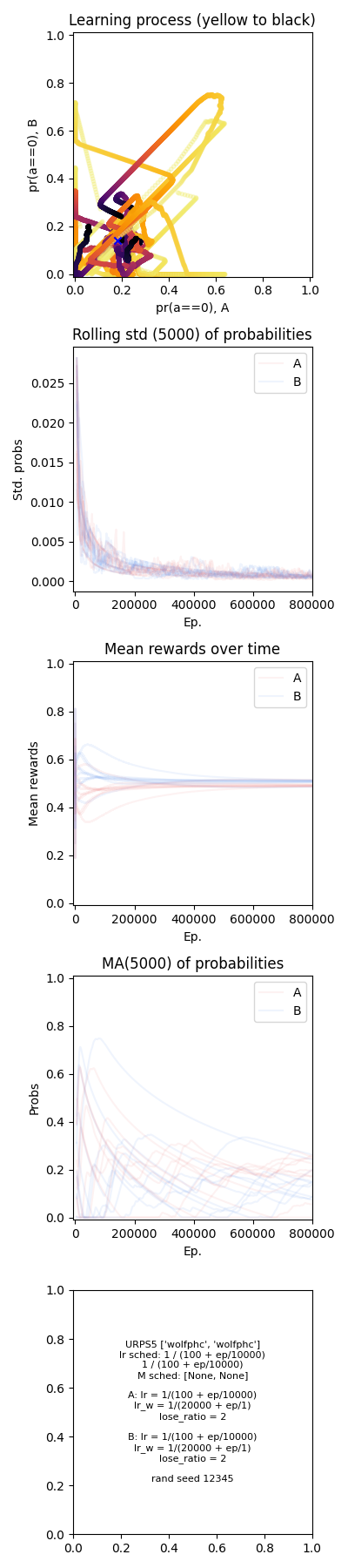}
    	\caption{Initial learning rate $10^{-2}$ for $Q$. Win learning rate $1/2 \cdot 10^{-4}$.}
    	\end{subfigure}
    \end{center}
    \caption[WoLF-PHC in self-play on the RPS-5 game.]{WoLF-PHC in self-play on the RPS-5 game with different learning schedules; for 10 different initialisations.
    (See figure \ref{fig:PD_MBL-DPU_high_mut} for a detailed explanation of the graphs.)
    }
    \label{app:fig:RPS5_WoLF-PHC}
\end{figure}
\begin{figure}[h] %
    \begin{center}
    	\begin{subfigure}[b]{0.3\linewidth}
    	\centering
    	\includegraphics[width=\textwidth, trim={0 770bp 0 0}, clip ]{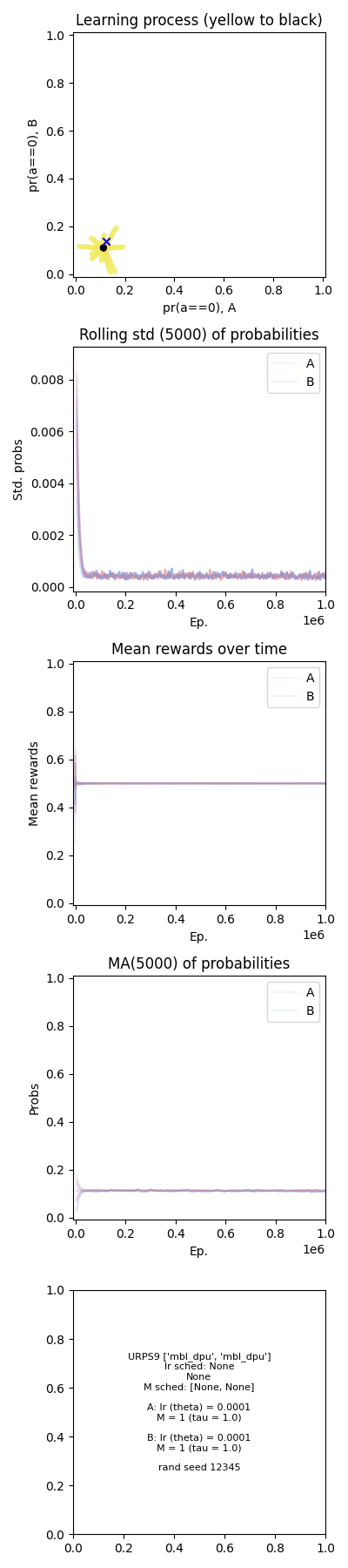}
    	\caption{$\tau = 1$, $M = 1^{-1}$}
    	\end{subfigure}
    	\hfill
    	\begin{subfigure}[b]{0.3\linewidth}
    	\centering
    	\includegraphics[width=\textwidth, trim={0 770bp 0 0}, clip ]{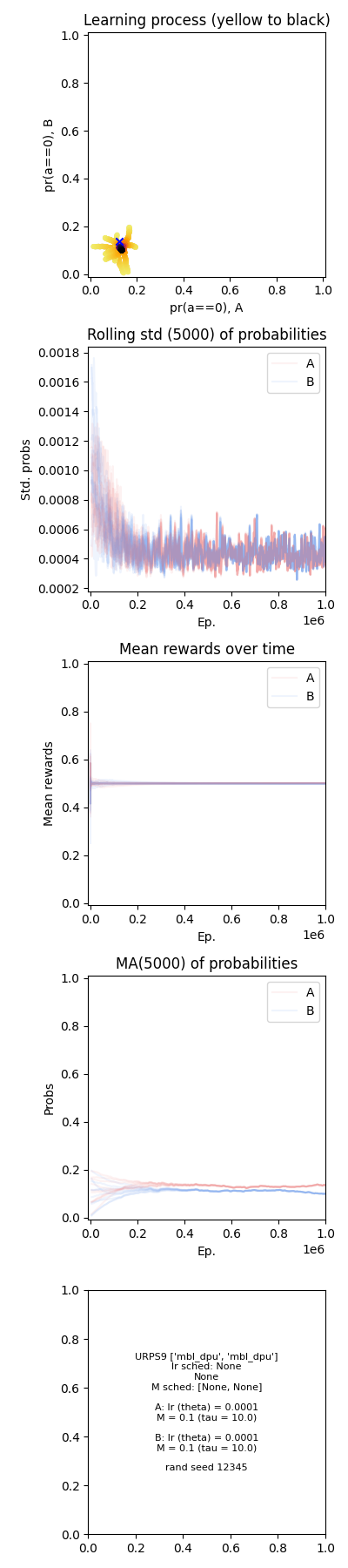}
    	\caption{$\tau = 10$, $M = 10^{-1}$}
    	\end{subfigure}
    	\hfill
    	\begin{subfigure}[b]{0.3\linewidth}
    	\centering
    	\includegraphics[width=\textwidth, trim={0 770bp 0 0}, clip ]{figures/results-solon/URPS9/mbl_dpu_mbl_dpu/fixed_M_fixed_lr/viz/1000000_eps_20201119_1559_ef9bc3e12b18__combo.png}
        \caption{$\tau = 20$, $M = 20^{-1}$}
    	\end{subfigure}
    \end{center}
    \caption[MBL-DPU in self-play on the RPS-9 game.]{MBL-DPU in self-play on the RPS-9 game with different values for $\tau$ ($1$, $10$, $20$) or $M$ ($1$, $10^{-1}$, $20^{-1}$) equivalently; $\theta = 10^{-4}$; for 10 different initialisations.
    (See figure \ref{fig:PD_MBL-DPU_high_mut} for a detailed explanation of the graphs.)
    }
    \label{fig:RPS9_MBL-DPU_high_mut}
\end{figure}
\begin{figure}[h] %
    \begin{center}
    	\begin{subfigure}[b]{0.3\linewidth}
    	\centering
    	\includegraphics[width=\textwidth, trim={0 770bp 0 0}, clip ]{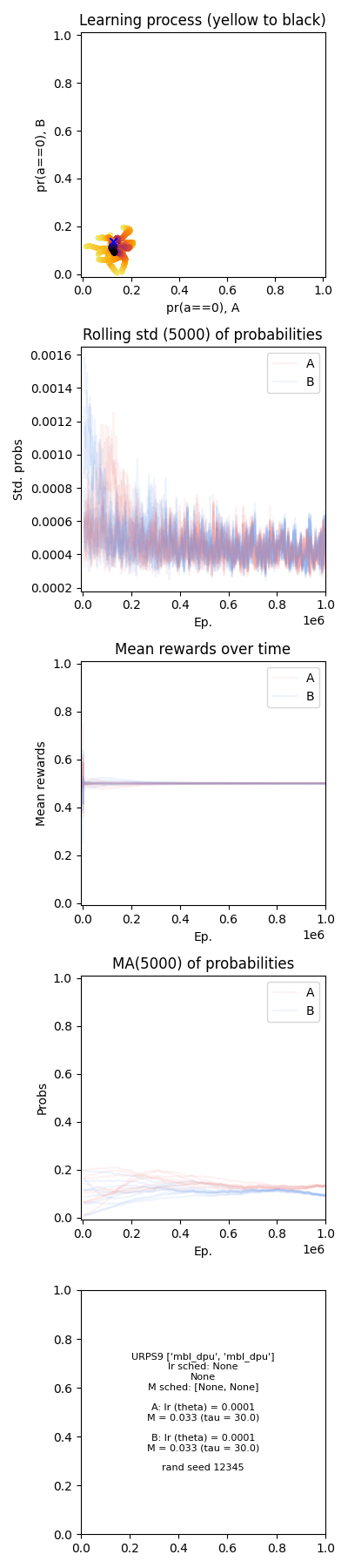}
        \caption{$\tau = 30$, $M = 30^{-1}$}
    	\end{subfigure}
    	\hfill
    	\begin{subfigure}[b]{0.3\linewidth}
    	\centering
    	\includegraphics[width=\textwidth, trim={0 770bp 0 0}, clip ]{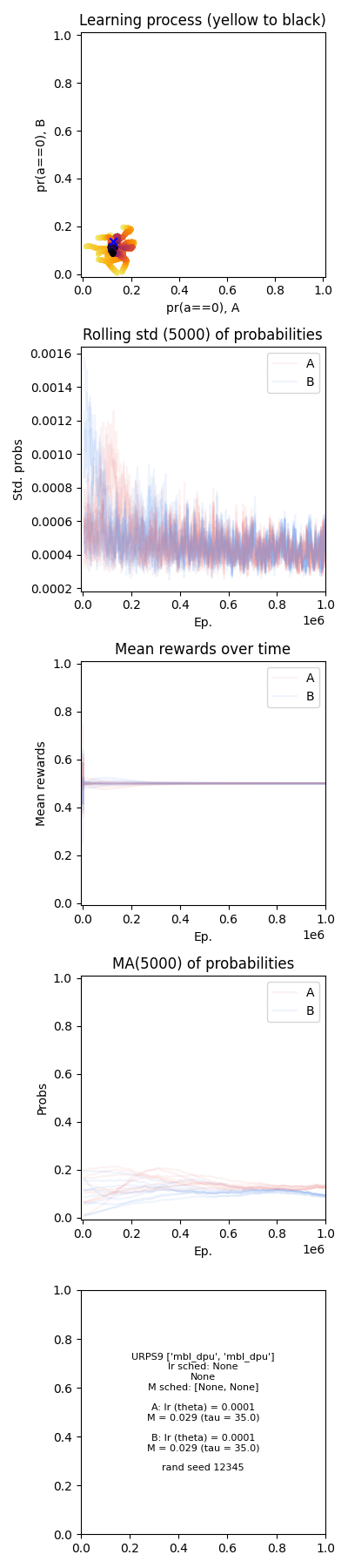}
    	\caption{$\tau = 35$, $M = 35^{-1}$}
    	\end{subfigure}
    	\hfill
    	\begin{subfigure}[b]{0.3\linewidth}
    	\centering
    	\includegraphics[width=\textwidth, trim={0 770bp 0 0}, clip ]{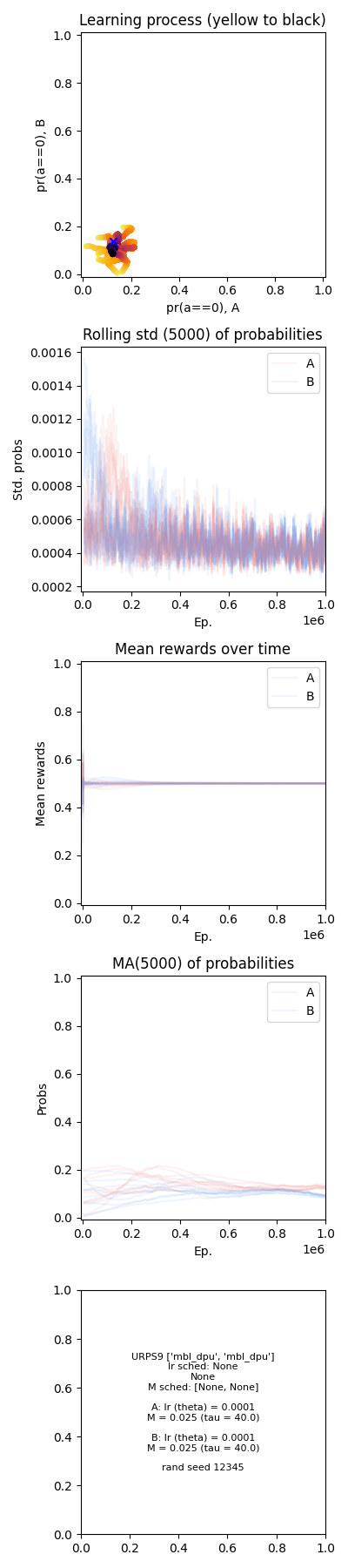}
    	\caption{$\tau = 40$, $M = 40^{-1}$}
    	\end{subfigure}
    \end{center}
    \caption[MBL-DPU in self-play on the RPS-9 game.]{MBL-DPU in self-play on the RPS-9 game with different values for $\tau$ ($30$, $35$, $40$) or $M$ ($30^{-1}$, $35^{-1}$, $40^{-1}$) equivalently; $\theta = 10^{-4}$; for 10 different initialisations.
    (See figure \ref{fig:PD_MBL-DPU_high_mut} for a detailed explanation of the graphs.)
    }
    \label{fig:RPS9_MBL-DPU_low_mut}
\end{figure}
\begin{figure}[h] %
    \begin{center}
    	\begin{subfigure}[b]{0.3\linewidth}
    	\centering
    	\includegraphics[width=\textwidth, trim={0 770bp 0 0}, clip ]{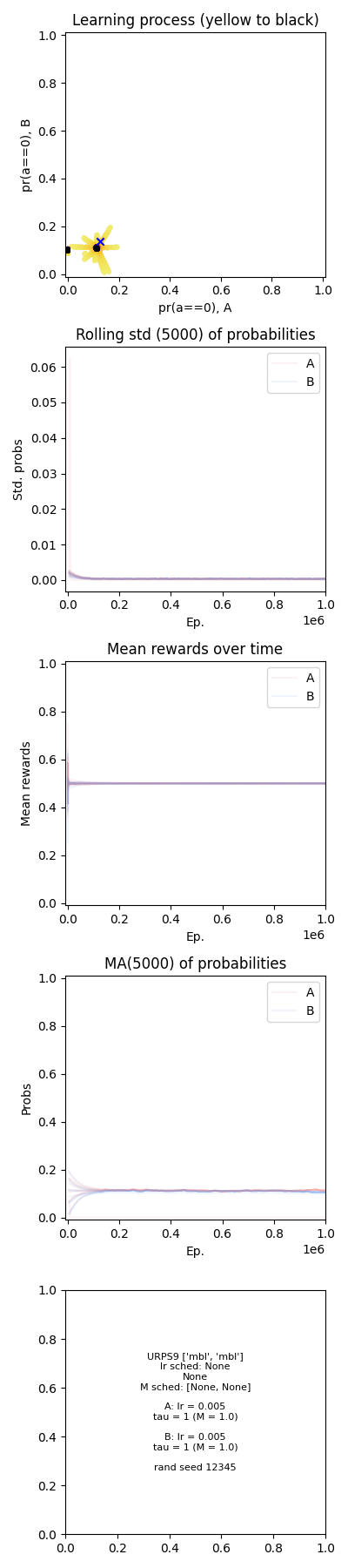}
        \caption{$\tau = 1$, $M = 1^{-1}$}
    	\end{subfigure}
    	\hfill
    	\begin{subfigure}[b]{0.3\linewidth}
    	\centering
    	\includegraphics[width=\textwidth, trim={0 770bp 0 0}, clip ]{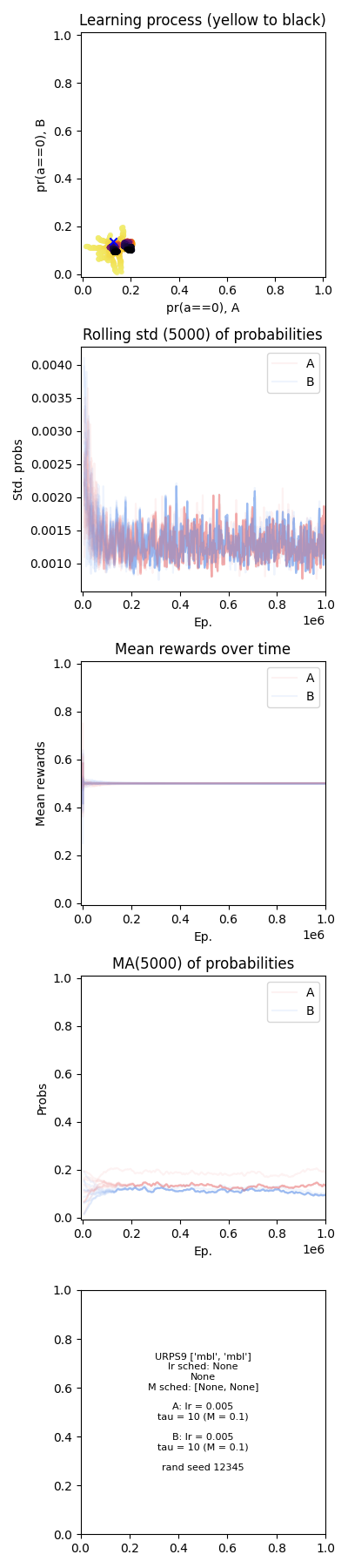}
    	\caption{$\tau = 10$, $M = 10^{-1}$}
    	\end{subfigure}
    	\hfill
    	\begin{subfigure}[b]{0.3\linewidth}
    	\centering
    	\includegraphics[width=\textwidth, trim={0 770bp 0 0}, clip ]{figures/results-solon/URPS9/mbl_mbl/fixed_M_fixed_lr/viz/1000000_eps_20201119_1323_954d43cf7216__combo.png}
    	\caption{$\tau = 20$, $M = 20^{-1}$}
    	\end{subfigure}
    \end{center}
    \caption[MBL-LC in self-play on the RPS-9 game.]{MBL-LC in self-play on the RPS-9 game with different values for $\tau$ ($1$, $10$, $20$) or $M$ ($1$, $10^{-1}$, $20^{-1}$) equivalently; $\theta = 5 \cdot 10^{-3}$; for 10 different initialisations.
    (See figure \ref{fig:PD_MBL-DPU_high_mut} for a detailed explanation of the graphs.)
    }
    \label{fig:RPS9_MBL-LC_high_mut}
\end{figure}
\begin{figure}[h] %
    \begin{center}
    	\begin{subfigure}[b]{0.3\linewidth}
    	\centering
    	\includegraphics[width=\textwidth, trim={0 770bp 0 0}, clip ]{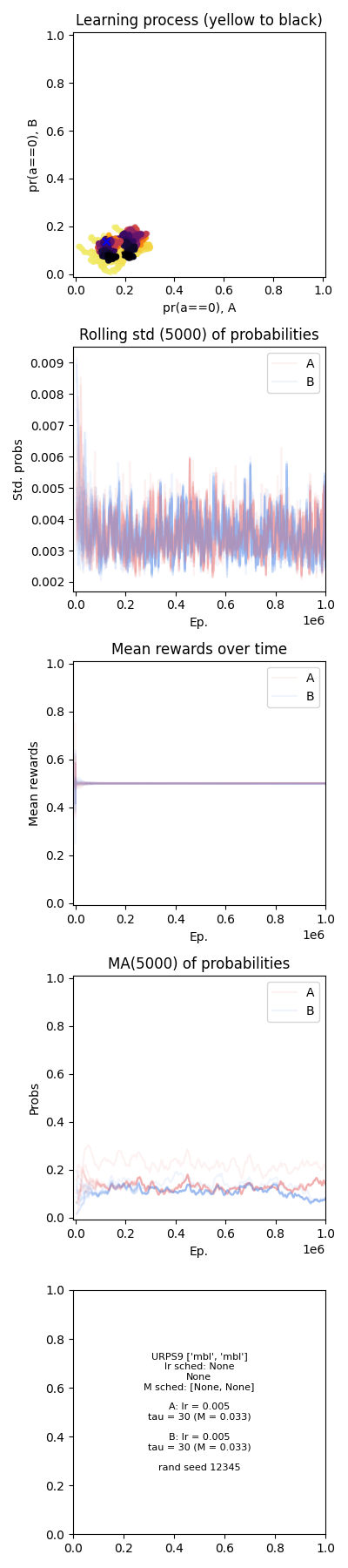}
    	\caption{$\tau = 30$, $M = 30^{-1}$}
    	\end{subfigure}
    	\hfill
    	\begin{subfigure}[b]{0.3\linewidth}
    	\centering 
    	\includegraphics[width=\textwidth, trim={0 770bp 0 0}, clip ]{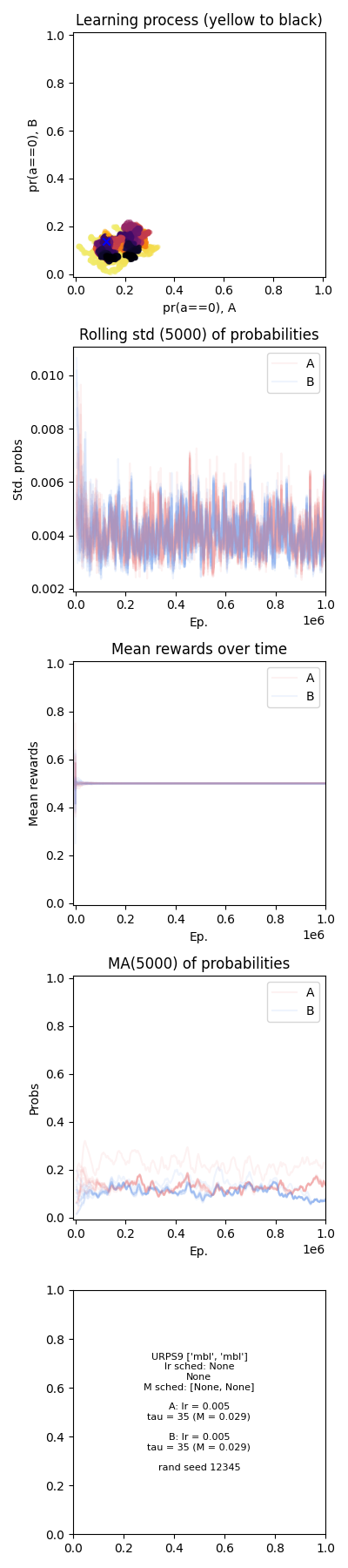}
        \caption{$\tau = 35$, $M = 35^{-1}$}
    	\end{subfigure}
    	\hfill
    	\begin{subfigure}[b]{0.3\linewidth}
    	\centering
    	\includegraphics[width=\textwidth, trim={0 770bp 0 0}, clip ]{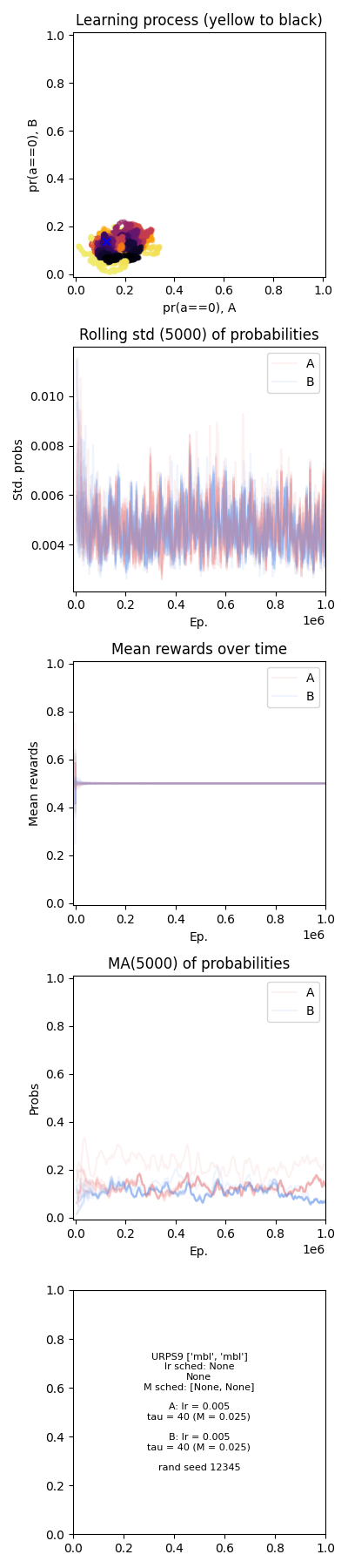}
    	\caption{$\tau = 40$, $M = 40^{-1}$}
    	\end{subfigure}
    \end{center}
    \caption[MBL-LC in self-play on the RPS-9 game.]{MBL-LC in self-play on the RPS-9 game with different values for $\tau$ ($30$, $35$, $40$) or $M$ ($30^{-1}$, $35^{-1}$, $40^{-1}$) equivalently; $\theta = 5 \cdot 10^{-3}$; for 10 different initialisations.
    (See figure \ref{fig:PD_MBL-DPU_high_mut} for a detailed explanation of the graphs.)
    }
    \label{fig:RPS9_MBL-LC_low_mut}
\end{figure}
\begin{figure}[h] %
    \begin{center}
    	\begin{subfigure}[b]{0.3\linewidth}
    	\centering
    	\includegraphics[width=\textwidth, trim={0 770bp 0 0}, clip ]{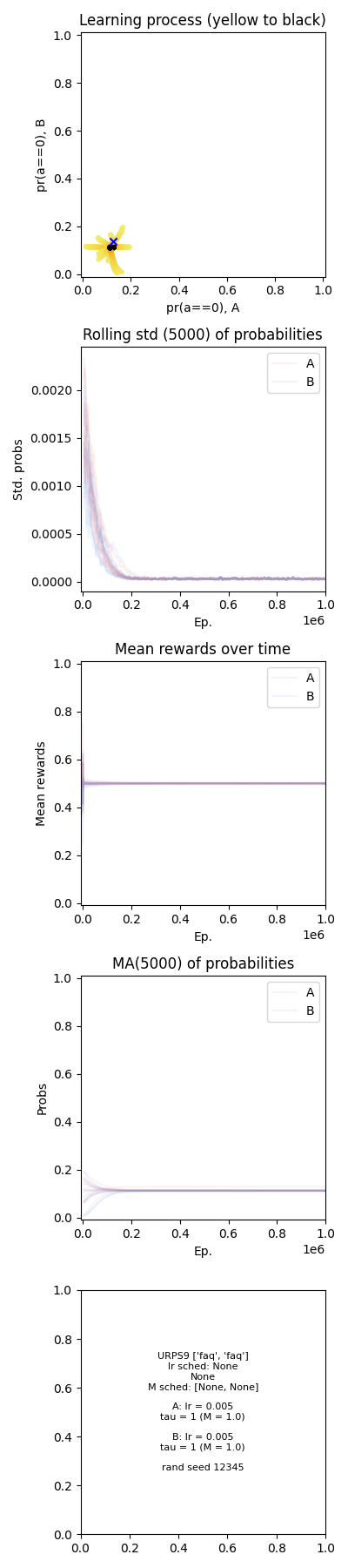}
    	\caption{$\tau = 1$, $M = 1^{-1}$}
    	\end{subfigure}
    	\hfill
    	\begin{subfigure}[b]{0.3\linewidth}
    	\centering
    	\includegraphics[width=\textwidth, trim={0 770bp 0 0}, clip ]{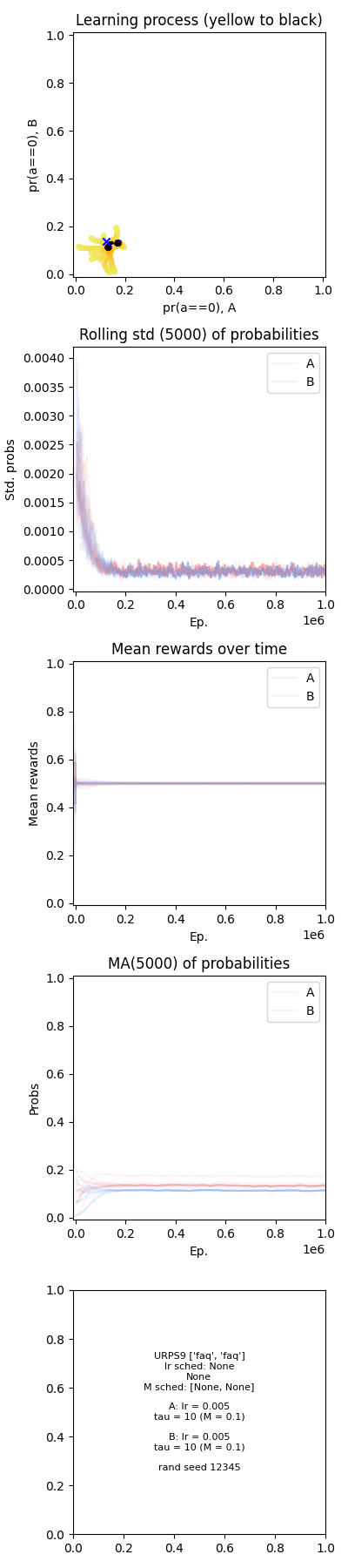}
        \caption{$\tau = 10$, $M = 10^{-1}$}
    	\end{subfigure}
    	\hfill
    	\begin{subfigure}[b]{0.3\linewidth}
    	\centering
    	\includegraphics[width=\textwidth, trim={0 770bp 0 0}, clip ]{figures/results-solon/URPS9/faq_faq/fixed_M_fixed_lr/viz/1000000_eps_20201119_1726_0ac59b5967be__combo.png}
    	\caption{$\tau = 20$, $M = 20^{-1}$}
    	\end{subfigure}
    \end{center}
    \caption[FAQ in self-play on the RPS-9 game.]{FAQ in self-play on the RPS-9 game with different values for $\tau$ ($1$, $10$, $20$) or $M$ ($1$, $10^{-1}$, $20^{-1}$) equivalently; $\theta = 5 \cdot 10^{-3}$; for 10 different initialisations.
    (See figure \ref{fig:PD_MBL-DPU_high_mut} for a detailed explanation of the graphs.)
    }
    \label{fig:RPS9_FAQ_high_mut}
\end{figure}
\begin{figure}[h] %
    \begin{center}
    	\begin{subfigure}[b]{0.3\linewidth}
    	\centering
    	\includegraphics[width=\textwidth, trim={0 770bp 0 0}, clip ]{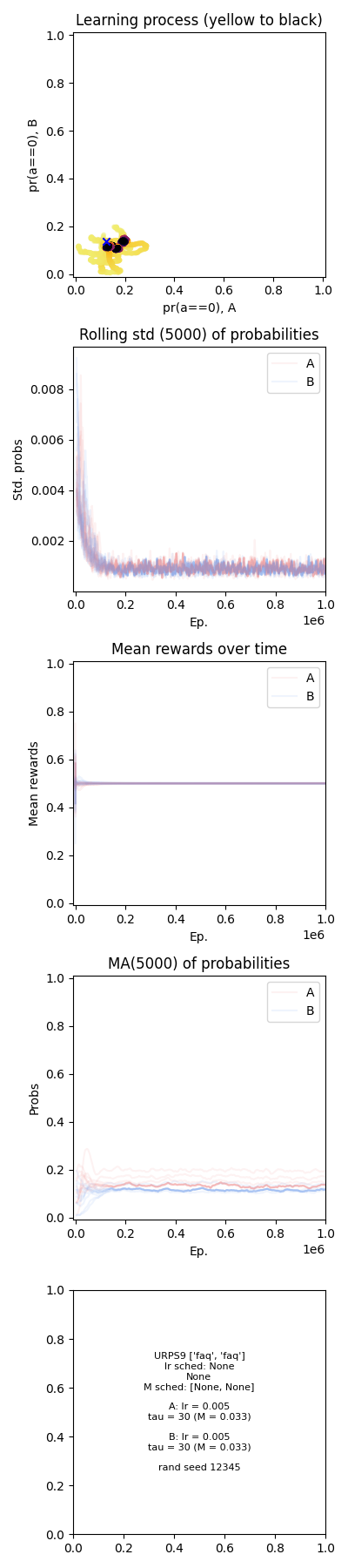}
    	\caption{$\tau = 30$, $M = 30^{-1}$}
    	\end{subfigure}
    	\hfill
    	\begin{subfigure}[b]{0.3\linewidth}
    	\centering
    	\includegraphics[width=\textwidth, trim={0 770bp 0 0}, clip ]{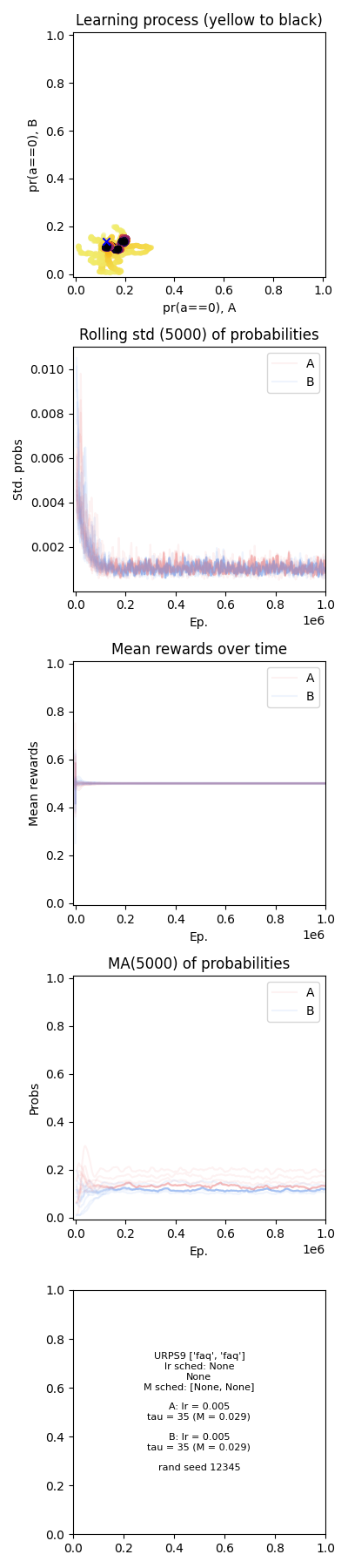}
        \caption{$\tau = 35$, $M = 35^{-1}$}
    	\end{subfigure}
    	\hfill
    	\begin{subfigure}[b]{0.3\linewidth}
    	\centering
    	\includegraphics[width=\textwidth, trim={0 770bp 0 0}, clip ]{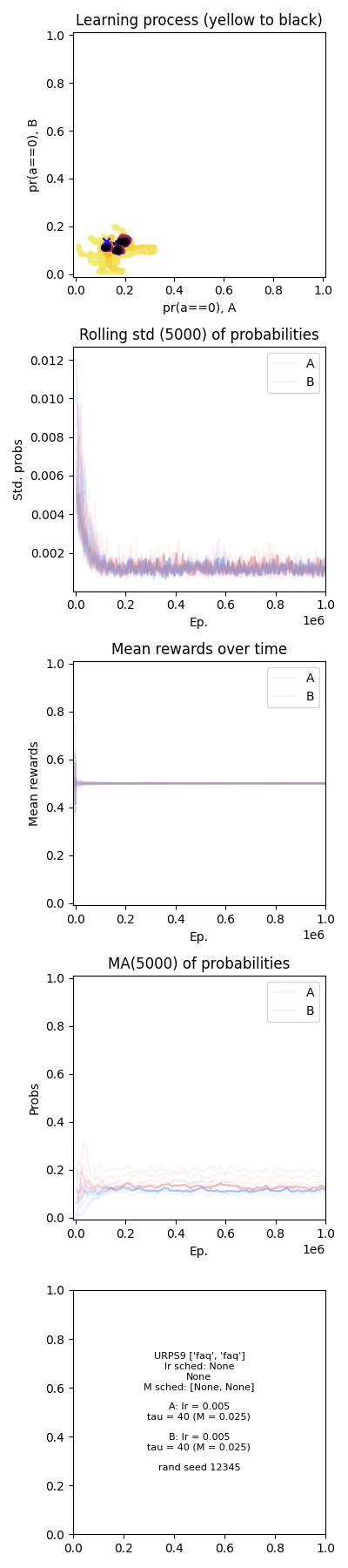}
    	\caption{$\tau = 40$, $M = 40^{-1}$}
    	\end{subfigure}
    \end{center}
    \caption[FAQ in self-play on the RPS-9 game.]{FAQ in self-play on the RPS-9 game with different values for $\tau$ ($30$, $35$, $40$) or $M$ ($30^{-1}$, $35^{-1}$, $40^{-1}$) equivalently; $\theta = 5 \cdot 10^{-3}$; for 10 different initialisations.
    (See figure \ref{fig:PD_MBL-DPU_high_mut} for a detailed explanation of the graphs.)
    }
    \label{fig:RPS9_FAQ_low_mut}
\end{figure}
\begin{figure}[h] %
    \begin{center}
    	\begin{subfigure}[b]{0.3\linewidth}
    	\centering
    	\includegraphics[width=\textwidth, trim={0 770bp 0 0}, clip ]{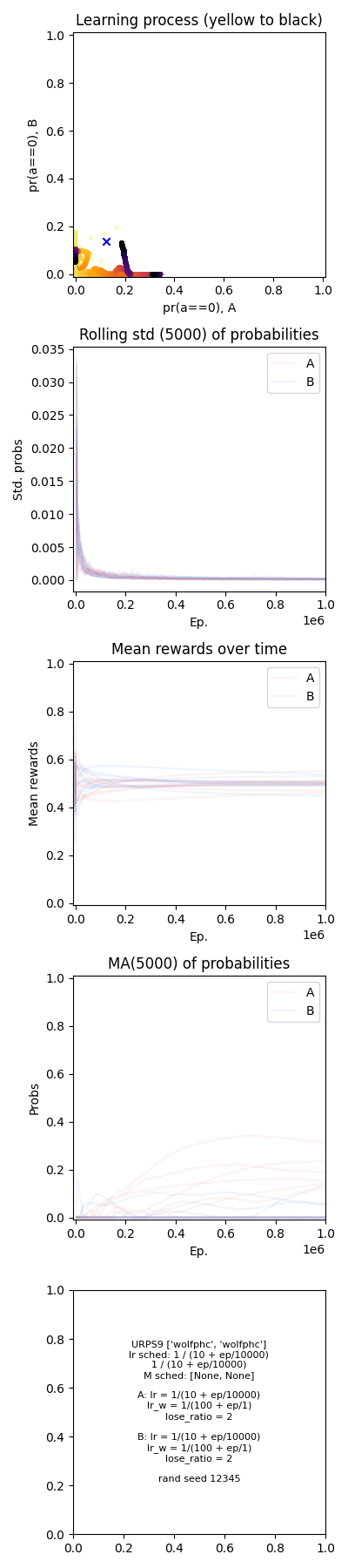}
    	\caption{Initial learning rate $10^{-1}$ for $Q$. Win learning rate $10^{-2}$.}
    	\end{subfigure}
    	\hfill
    	\begin{subfigure}[b]{0.3\linewidth}
    	\centering
    	\includegraphics[width=\textwidth, trim={0 770bp 0 0}, clip ]{figures/results-solon/URPS9/wolfphc_wolfphc/fixed_M_reducing_lr/viz/1000000_eps_20201118_1846_102fdd638cc6__combo.png}
        \caption{Initial learning rate $10^{-1}$ for $Q$. Win learning rate $1/2 \cdot 10^{-4}$.}
    	\end{subfigure}
    	\hfill
    	\begin{subfigure}[b]{0.3\linewidth}
    	\centering
    	\includegraphics[width=\textwidth, trim={0 770bp 0 0}, clip ]{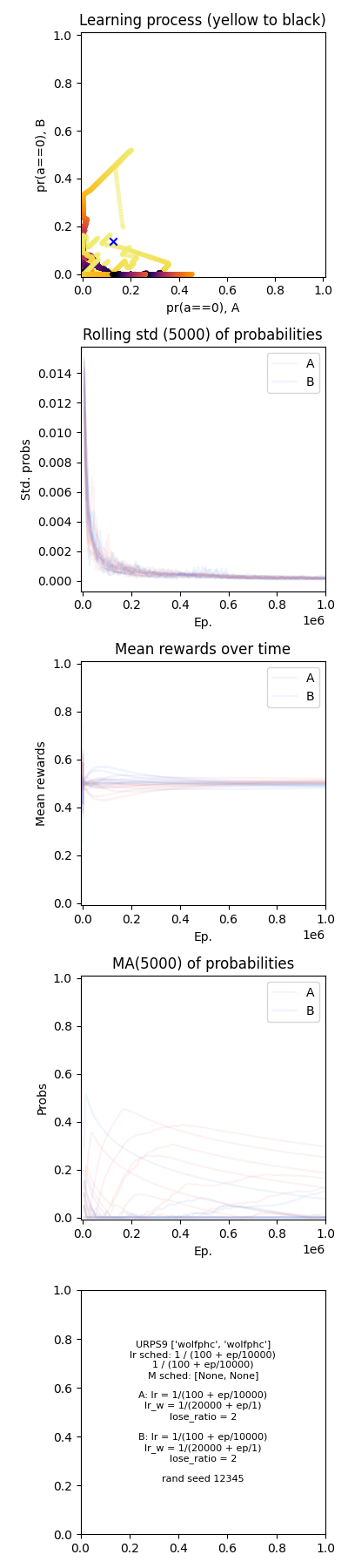}
    	\caption{Initial learning rate $10^{-2}$ for $Q$. Win learning rate $1/2 \cdot 10^{-4}$.}
    	\end{subfigure}
    \end{center}
    \caption[WoLF-PHC in self-play on the RPS-9 game.]{WoLF-PHC in self-play on the RPS-9 game with different learning schedules; for 10 different initialisations.
    (See figure \ref{fig:PD_MBL-DPU_high_mut} for a detailed explanation of the graphs.)
    }
    \label{fig:app_RPS9_WoLF-PHC}
\end{figure}

\clearpage

\subsection{Three-player Matching Pennies}\label{app:exp_spec_3MP}%

Further, we consider the behaviour of the MBL variants in comparison to FAQ learning and WoLF-PHC in a three-player Matching Pennies (3MP) game introduced in \cite{jordan_three_1993}, with payoffs as given in table \ref{tab:3MP}. The similarity to the standard MP game becomes clear when one considers that the payoff structure reflects the following idea: The first player wants to match the second player's action. The second player wants to match the third player's action. However, the third player does not want to match the first player's action. The unique Nash equilibrium for 3MP is located at the centre of $\mathcal{D}$. Note that, as initially proposed, 3MP is not a zero-sum game.

\begin{table}[h!]
\hspace{\fill}
\begin{subtable}[t]{0.45\textwidth}
\centering
\begin{tabular}[t]{c|cc}
	&	H	&	T	\\ \hline
H	&	$(1,1,-1)$ & $(-1,-1,-1)$ \\
T	&	$(-1,1,1)$ & $(1,-1,1)$
\end{tabular}
\vspace{0.5\baselineskip}
\caption{Payoffs when the third player chooses `H'.}
\end{subtable}
\hspace{\fill}
\begin{subtable}[t]{0.45\textwidth}
\centering
\begin{tabular}[t]{c|cc}
	&	H	&	T	\\ \hline
H	&	$(1,-1,1)$ & $(-1,1,1)$ \\
T	&	$(-1,-1,-1)$ & $(1,1,-1)$
\end{tabular}
\vspace{0.5\baselineskip}
\caption{Payoffs when the third player chooses `T'.}
\end{subtable}
\hspace{\fill}
\caption[Payoffs for the three-player Matching Pennies game.]{Payoff tuples for the three-player Matching Pennies (3MP) game with the first player's action determining the row, the second player's action the column, and the third player's action the table.}\label{tab:3MP}
\end{table}

In 3MP, both MBL variants (figures \ref{fig:3MP_MBL-DPU}, \ref{fig:3MP_MBL-LC}) show apparently asymptotically stable periodic limit behaviours, which approach the boundary of $\mathcal{D}$ as mutation diminishes. We further see a very similar behaviour for FAQ (figure \ref{fig:3MP_FAQ}) with $\tau^{-1}$ showing an analogous effect to $M$ in MBL, quite similar to the two-player settings. Likewise, WoLF-PHC (figure \ref{app:fig:3MP_WoLF-PHC}) exhibits apparently asymptotically stable trajectories, at least in the projection onto the first actions of the first two players. Again, WoLF-PHC shows a reduction of variance over time, presumably due to diminishing learning rates. In \cite{bowling_multiagent_2002}, the authors show that WoLF-PHC converges to the Nash equilibrium when $\delta_l / \delta_w = 3$ (as opposed to $\delta_l / \delta_w = 2$). Since there is no established ODE approximation of WoLF-PHC that we are aware of, the reasons for this remain unclear. One should also note that we have made sure that the Nash equilibrium is not located at the centre of $\mathcal{D}$ in the two-player games because the perturbation term in FAQ has its equilibrium there and convergence might easily have been coincidental. For 3MP, we have not made any such adaptations and some behaviours might change when the Nash equilibrium is moved away from the centre.

\begin{figure}[h] %
    \begin{center}
   	\begin{subfigure}[b]{0.3\linewidth}
   	\centering
   	\includegraphics[width=\textwidth, trim={0 770bp 0 0}, clip ]{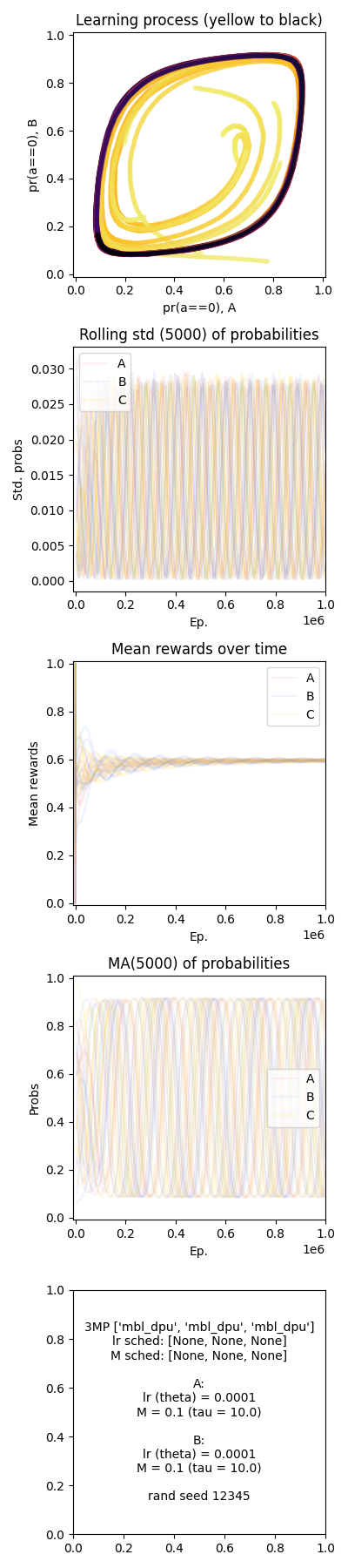}
    \caption{$\tau = 10$, $M = 10^{-1}$}
   	\end{subfigure}
   	\hfill
   	\begin{subfigure}[b]{0.3\linewidth}
   	\centering
   	\includegraphics[width=\textwidth, trim={0 770bp 0 0}, clip ]{figures/results-solon/3MP/mbl_dpu_mbl_dpu_mbl_dpu/fixed_M_fixed_lr/viz/1000000_eps_20200421_1439_33d31347cc92__combo.png}
   	\caption{$\tau = 20$, $M = 20^{-1}$}
   	\end{subfigure}
   	\hfill
   	\begin{subfigure}[b]{0.3\linewidth}
   	\centering
   	\includegraphics[width=\textwidth, trim={0 770bp 0 0}, clip ]{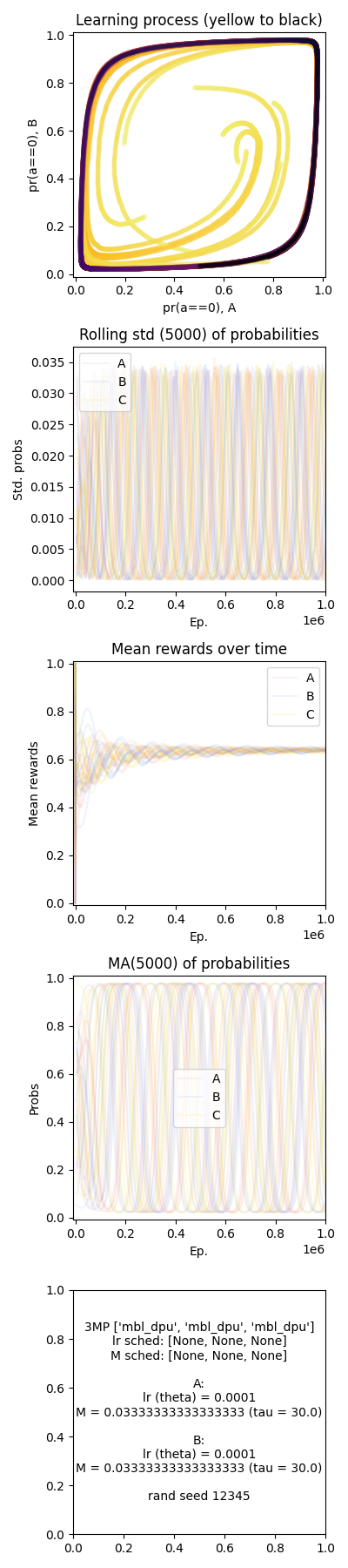}
   	\caption{$\tau = 30$, $M = 30^{-1}$}
   	\end{subfigure}
    \end{center}
    \caption[MBL-DPU in self-play on the 3MP game.]{MBL-DPU in self-play on the 3MP game with different values for $\tau$ ($10$, $20$, $30$) or $M$ ($10^{-1}$, $20^{-1}$, $30^{-1}$) equivalently; $\theta = 10^{-4}$; for 10 different initialisations.
    (See figure \ref{fig:PD_MBL-DPU_high_mut} for a detailed explanation of the graphs.)
    }
    \label{fig:3MP_MBL-DPU}
\end{figure}
\begin{figure}[h] %
    \begin{center}
    \begin{subfigure}[b]{0.3\linewidth}
        \centering
    	\includegraphics[width=\textwidth, trim={0 770bp 0 0}, clip ]{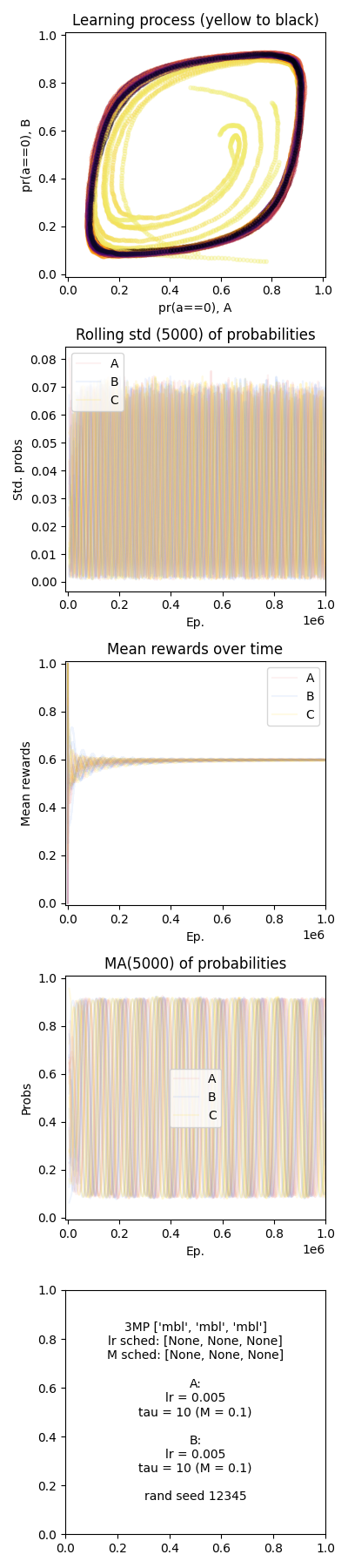}
    	\caption{$\tau = 10$, $M = 10^{-1}$}
    	\end{subfigure}
    	\hfill
    	\begin{subfigure}[b]{0.3\linewidth}
    	\centering
    	\includegraphics[width=\textwidth, trim={0 770bp 0 0}, clip ]{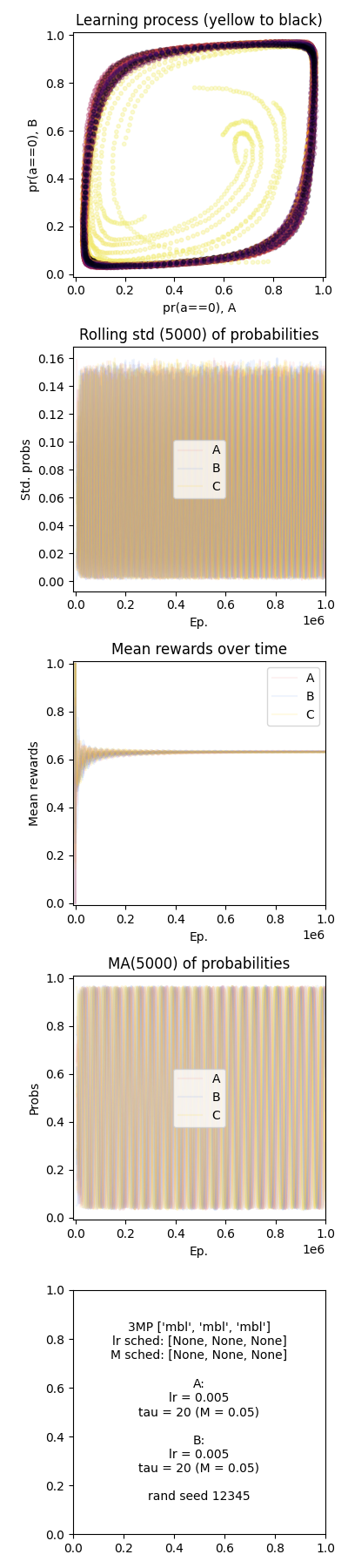}
    	\caption{$\tau = 20$, $M = 20^{-1}$}
    	\end{subfigure}
    	\hfill
    	\begin{subfigure}[b]{0.3\linewidth}
    	\centering
    	\includegraphics[width=\textwidth, trim={0 770bp 0 0}, clip ]{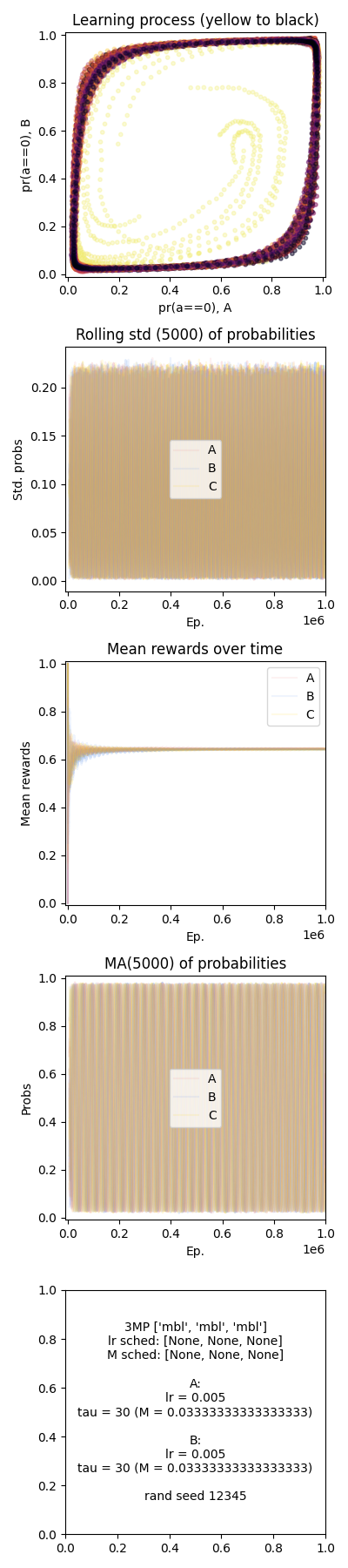}
    	\caption{$\tau = 30$, $M = 30^{-1}$}
    \end{subfigure}
    \end{center}
    \caption[MBL-LC in self-play on the 3MP game.]{MBL-LC in self-play on the 3MP game with different values for $\tau$ ($10$, $20$, $30$) or $M$ ($10^{-1}$, $20^{-1}$, $30^{-1}$) equivalently; $\theta = 10^{-4}$; for 10 different initialisations.
    (See figure \ref{fig:PD_MBL-DPU_high_mut} for a detailed explanation of the graphs.)
    }
    \label{fig:3MP_MBL-LC}
\end{figure}
\begin{figure}[h] %
    \begin{center}
    \begin{subfigure}[b]{0.3\linewidth}
   	\centering
   	\includegraphics[width=\textwidth, trim={0 770bp 0 0}, clip ]{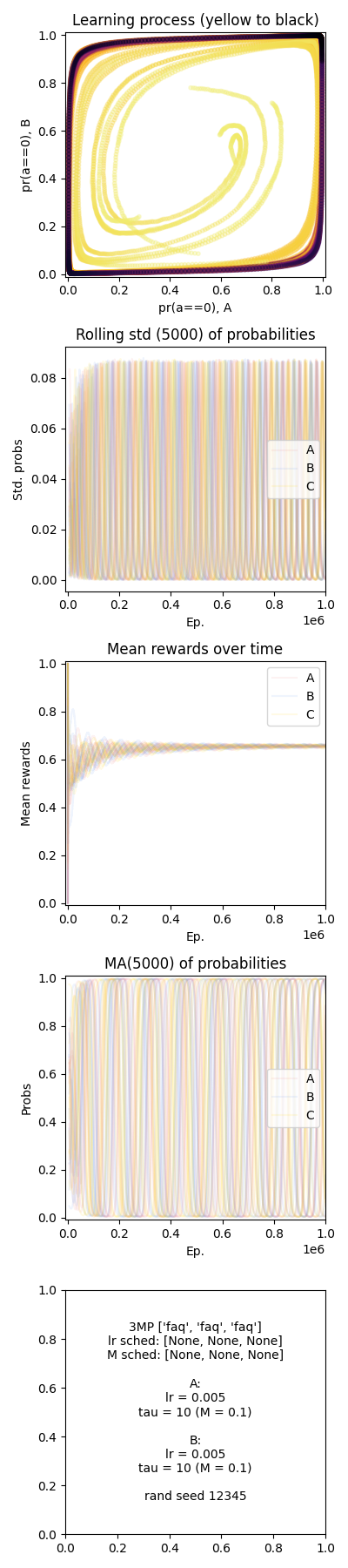}
    \caption{$\tau = 10$, $M = 10^{-1}$}
   	\end{subfigure}
   	\hfill
   	\begin{subfigure}[b]{0.3\linewidth}
   	\centering
   	\includegraphics[width=\textwidth, trim={0 770bp 0 0}, clip ]{figures/results-solon/3MP/faq_faq_faq/fixed_M_fixed_lr/viz/1000000_eps_20200421_1531_93290066331e__combo.png}
   	\caption{$\tau = 20$, $M = 20^{-1}$}
   	\end{subfigure}
   	\hfill
   	\begin{subfigure}[b]{0.3\linewidth}
   	\centering
   	\includegraphics[width=\textwidth, trim={0 770bp 0 0}, clip ]{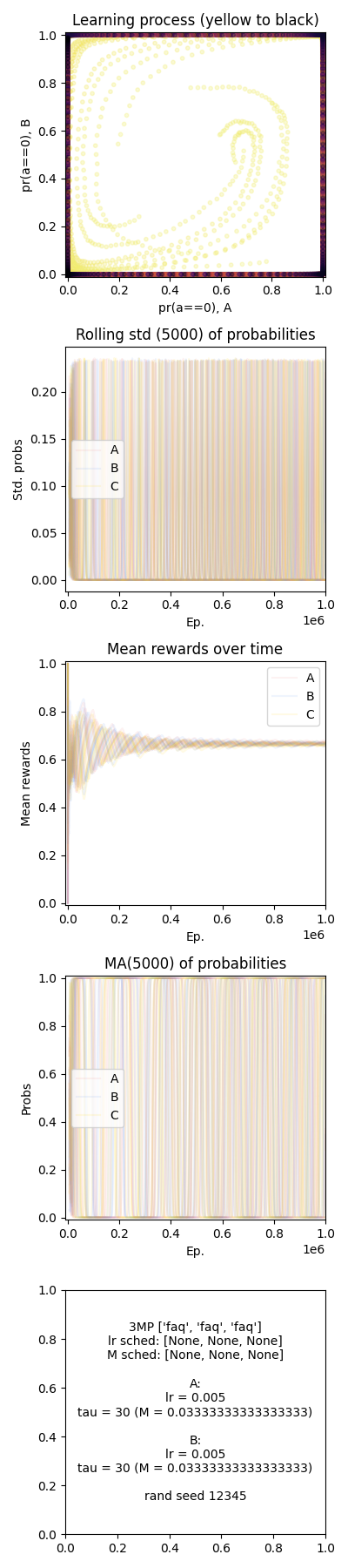}
   	\caption{$\tau = 30$, $M = 30^{-1}$}
    \end{subfigure}
    \end{center}
    \caption[FAQ in self-play on the 3MP game.]{FAQ in self-play on the 3MP game with different values for $\tau$ ($10$, $20$, $30$) or $M$ ($10^{-1}$, $20^{-1}$, $30^{-1}$) equivalently; $\theta = 10^{-4}$; for 10 different initialisations.
    (See figure \ref{fig:PD_MBL-DPU_high_mut} for a detailed explanation of the graphs.)
    }
    \label{fig:3MP_FAQ}
\end{figure}
\begin{figure}[h] %
    \begin{center}
    \begin{subfigure}[b]{0.3\linewidth}
   	\centering
   	\includegraphics[width=\textwidth, trim={0 770bp 0 0}, clip ]{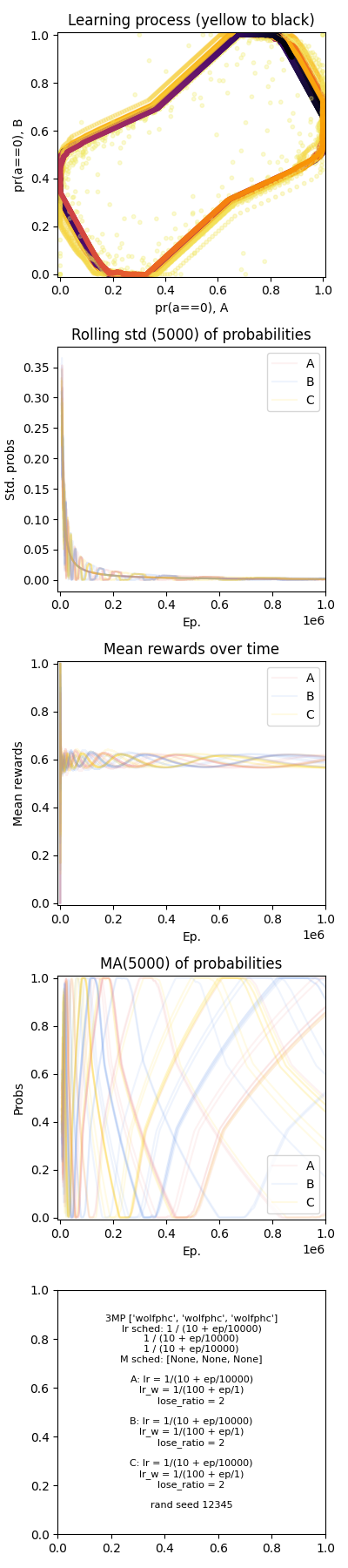}
   	\caption{Initial learning rate $10^{-1}$ for $Q$. Win learning rate $10^{-2}$.}
   	\end{subfigure}
   	\hfill
   	\begin{subfigure}[b]{0.3\linewidth}
   	\centering
   	\includegraphics[width=\textwidth, trim={0 770bp 0 0}, clip ]{figures/results-solon/3MP/wolfphc_wolfphc_wolfphc/fixed_M_reducing_lr/viz/1000000_eps_20201125_1805_fcdea3a245a6__combo.png}
    \caption{Initial learning rate $10^{-1}$ for $Q$. Win learning rate $1/2 \cdot 10^{-4}$.}
   	\end{subfigure}
   	\hfill
   	\begin{subfigure}[b]{0.3\linewidth}
   	\centering
   	\includegraphics[width=\textwidth, trim={0 770bp 0 0}, clip ]{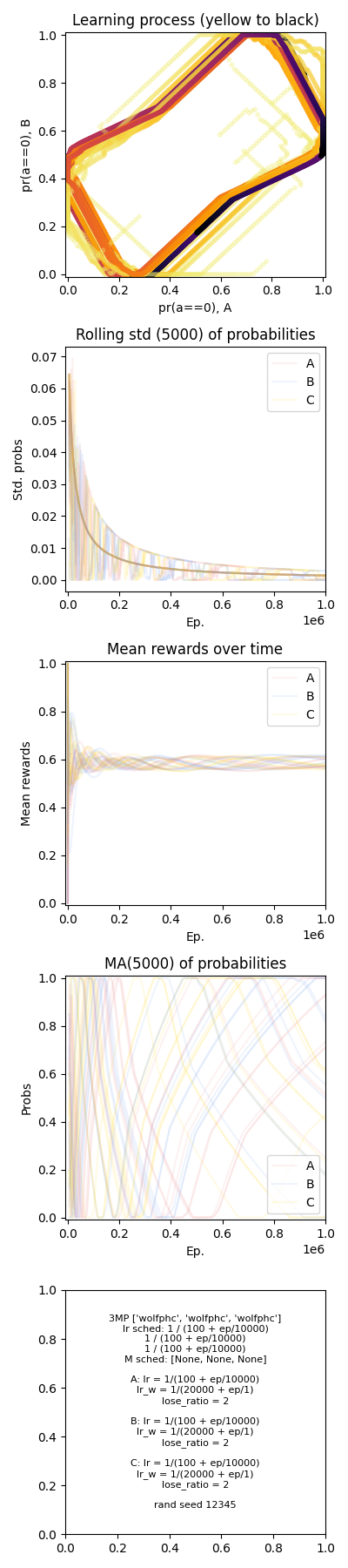}
   	\caption{Initial learning rate $10^{-2}$ for $Q$. Win learning rate $1/2 \cdot 10^{-4}$.}
    \end{subfigure}
    \end{center}
    \caption[WoLF-PHC in self-play on the 3MP game.]{WoLF-PHC in self-play on the 3MP game with different learning schedules; for 10 different initialisations.
    (See figure \ref{fig:PD_MBL-DPU_high_mut} for a detailed explanation of the graphs.)
    }
    \label{app:fig:3MP_WoLF-PHC}
\end{figure}

\end{document}